\newif\ifarxiv\arxivtrue
\newif\ifmainpaper\mainpapertrue
\newif\ifack
\newif\iftwocolumn

\ifarxiv
    \documentclass[a4paper,10pt,twoside]{article}   
    \usepackage{hyperref}
    \hypersetup{
        colorlinks,
        linkcolor={red!50!black},
        citecolor={green!50!black},
        urlcolor={blue!80!black}
    }
    \usepackage{algorithm}
    \usepackage{algorithmic}
    \usepackage[margin=1in]{geometry}
    \usepackage{parskip}
    \usepackage{xcolor}

    \usepackage[backend=bibtex,style=authoryear,backref=true]{biblatex}
    \DefineBibliographyStrings{english}{%
    backrefpage = {page},
    backrefpages = {pages},
    }
    \newcommand\citep\parencite
    \newcommand\citet\textcite

    \twocolumnfalse
    \acktrue
\else
    \documentclass{article}
    \usepackage{hyperref}
    \usepackage{icml2023}
    \ackfalse
\fi

\usepackage{amsmath}
\usepackage{amssymb}
\usepackage{amsthm}
\usepackage{booktabs}
\usepackage[capitalize,noabbrev]{cleveref}
\usepackage{dsfont}
\usepackage[inline, shortlabels]{enumitem}
\usepackage{graphicx}
\usepackage[utf8]{inputenc}
\usepackage{mathtools}
\usepackage{microtype}
\usepackage{subcaption}

\newcommand{\gbar}{\bar{g}}

\newcommand{\sigmamin}{\sigma_{\mathrm{min}}}

\DeclareMathOperator{\clip}{\mathtt{clip}}

\newcommand{\uclip}{U-Clip}
\newcommand{\bgamma}{\bar{\gamma}}
\newcommand{\tgamma}{\tilde{\gamma}}
\newcommand{\mhat}{\hat{m}}
\newcommand{\shat}{\hat{s}}
\newcommand{\opt}{\texttt{opt}}

\newcommand{\enum}[2]{{#1}\mathrm{e}{#2}}

\newcommand{\R}{\mathbb{R}}
\newcommand{\N}{\mathbb{N}}

\DeclareMathOperator{\unif}{\text{Unif}}
\DeclareMathOperator{\E}{\mathbb{E}}

\let\P\relax
\DeclareMathOperator{\P}{\mathbb{P}}

\newcommand{\1}{\mathds{1}}

\newcommand{\midvert}{\;\middle\vert{}\;}

\newcommand{\grad}{\nabla}
\DeclareMathOperator*{\argmin}{argmin}
\newcommand{\sign}{\operatorname{sign}}
\newcommand{\abs}[1]{\vert{} #1 \vert{}}

\newcommand{\norm}[1]{\lVert{}#1\rVert{}}

\newcommand{\dd}{\mathop{}\!\mathrm{d}}

\newcommand{\F}{\mathcal{F}}
\newcommand{\X}{\mathcal{X}}

\theoremstyle{plain}
\newtheorem{theorem}{Theorem}[section]
\newtheorem{proposition}[theorem]{Proposition}
\newtheorem{lemma}[theorem]{Lemma}
\newtheorem{corollary}[theorem]{Corollary}
\newtheorem*{lemma*}{Lemma}
\newtheorem*{proposition*}{Proposition}

\theoremstyle{remark}
\newtheorem{remark}[theorem]{Remark}
\newtheorem{example}[theorem]{Example}

\newcommand{\twofigwidth}{0.49\textwidth}
\newcommand{\threefigwidth}{0.3\textwidth}

\newcommand{\thetitle}{\uclip{}: On-Average Unbiased Stochastic Gradient Clipping}
\newcommand{\thekeywords}{Machine Learning, ICML, Online Learning, Optimization, Theory,
    Gradient Descent, SGD, Clipping, Bias}

\ifarxiv
    \title{\thetitle}
    \author{%
    Bryn Elesedy\thanks{Correspondence to \href{mailto:bryn@robots.ox.ac.uk}{bryn@robots.ox.ac.uk}.
    Work carried out while BE was an intern at DeepMind.}\\
    University of Oxford
    \and
    Marcus Hutter\\
    DeepMind
    }
    \date{\today}
\else
    \icmltitlerunning{\thetitle}
\fi

\begin{document}

\ifarxiv
    \maketitle
    \twocolumnfalse
\else
    \twocolumn[
        \icmltitle{\thetitle}

    \icmlsetsymbol{equal}{*}

    \begin{icmlauthorlist}
    \icmlauthor{Bryn Elesedy}{oxford}
    \icmlauthor{Marcus Hutter}{deepmind}
    \end{icmlauthorlist}

    \icmlaffiliation{oxford}{Department of Computer Science, University of Oxford, Oxford, UK}
    \icmlaffiliation{deepmind}{DeepMind, London, UK}

    \icmlcorrespondingauthor{Bryn Elesedy}{bryn@robots.ox.ax.uk}

    \icmlkeywords{\thekeywords}

    \vskip 0.3in
    ]

    \printAffiliationsAndNotice{}
    \twocolumntrue
\fi

\begin{abstract}
    \noindent \uclip{} is a simple amendment to gradient clipping
    that can be applied to any iterative gradient optimization
    algorithm. Like regular clipping, \uclip{} involves using
    gradients that are clipped to a prescribed size
    (e.g.~with component wise or norm based clipping) 
    but instead of discarding the clipped portion of the gradient,
    \uclip{} maintains a buffer of these values that is added to the gradients
    on the next iteration (before clipping).
    We show that the cumulative bias of the~\uclip{} updates is bounded by a constant.
    This implies that the clipped updates are unbiased \emph{on average}.
    Convergence follows via a lemma that guarantees
    convergence with updates $u_i$ as long as
    $\sum_{i=1}^t (u_i - g_i) = o(t)$ where $g_i$ are the gradients.
    Extensive experimental exploration is performed on CIFAR10 with further validation
    given on ImageNet.
\end{abstract}
    \ifarxiv
    \def\contentsname{\centering\normalsize Contents}\setcounter{tocdepth}{2}
    {\parskip=-2.9ex\tableofcontents}
    \fi

%%%%%%%%%%%%%%%%%%%%%%%%%%%%%%%%%%%%%%%%%%%%%%%%%%%%%%%%%%%%%%%
\section{Introduction}\label{sec:introduction}
%%%%%%%%%%%%%%%%%%%%%%%%%%%%%%%%%%%%%%%%%%%%%%%%%%%%%%%%%%%%%%%
Stochastic gradient descent (SGD)~\citep{robbins1951stochastic}
aims to minimize an expected objective $f(x) = \E[\ell(x, W)]$,
where $W$ is some external source of randomness, by using
an unbiased estimator of $\E[\grad_x \ell(x, W)]$. 
Often in machine learning, $f(x)$ is full batch loss and
the randomness comes from sub-sampling (mini-batching).
Gradient clipping is a standard technique to control the size of gradient updates.
Its applications range from reducing the impact of mini-batch noise to
controlling exploding gradients RNNs~\citep{pascanu2013difficulty} and 
it has even been observed to reduce the generalisation gap~\citep{hardt2016train}.
The gradient is constrained to be
no larger than a pre-specified size, typically coordinate wise or in Euclidean norm.
For $x\in\R$, coordinate clipping to scale $\gamma > 0$ is given by
$
  \clip(x, \gamma) = \max\{\min\{x, \gamma\}, -\gamma\}
  $
and is applied element-wise to vectors in $\R^d$. 
For $x \in \R^d$, norm clipping is given by 
$
  \clip(x, \gamma)  = \min\{1, \gamma / \norm{x}\} x
  $
where $\norm{\cdot}$ denotes the Euclidean norm.
Norm clipping has the benefit of maintaining the direction of the updates, but
will also rescale all coordinates if just one has a very large spurious value.
Coordinate clipping maintains independence between the scale of the coordinates.
Moreover, in respect of reducing the objective,
with convex smooth functions (i.e.~those with Lipschitz gradients) it is 
sufficient for the update to be in a direction with
positive inner product with the gradient (e.g.~by~\citet[Lemma 3.4]{bubeck2015convex}).
Our theory concerns coordinate clipping, but our experiments explore norm
clipping too.

%-------------------------------%
\paragraph{Main Idea.}
%-------------------------------%
Each of the $\clip$ functions above are non-linear which results in 
biased updates.
To reduce this bias, we introduce~\uclip{}.
\uclip{} can be applied to any gradient based optimiser \opt{}
without modification.
\opt{} can be any operation, procedure or sub-routine that accepts
gradients as inputs.
Let $g_t$ be a (possibly stochastic) sub-gradient.
Instead of passing $g_t$ to \opt{}, \uclip{} feeds $u_t$ to \opt{} where
\begin{align}\nonumber
  u_t &= \clip(g_t + \Delta_t, \gamma_t) \\ 
  \Delta_{t+1} &= g_t + \Delta_t - u_t.\label{eq:deltarec}
\end{align}
Above we introduced the \emph{carry} $\Delta_t$, initialized $\Delta_1=0$,
which is a buffer of clipped values. 
We treat \opt{} as a black box and \uclip{} modifies \emph{only} the gradients given to \opt{},
not of its hyperparameters or any other aspect of its configuration.
Note that we clip before applying \opt{}, rather than after. 
This is a design choice motivated by consistency in any internal statistics of \opt{},
such as with Adam~\citep{kingma2014adam}.%
\footnote{Clipping before regulates the size of the inputs to
\opt{}. Clipping after \opt{} would allow spuriously large gradients to distort the 
internal statistics of, for instance, Adam.}
\uclip{} combined with SGD is shown in~\cref{alg:uclip-sgd}.

Via the carry we use all of the information in the gradients, which
reduces the cumulative bias relative to standard clipping.
However, just like standard clipping, we guard against overly large
gradients.
We will show that, under certain conditions, the above procedure has
$\E[\norm{\Delta_t}] = O(1)$. 
This result implies that the updates are on-average unbiased, in the following sense.
Note that $\Delta_{t} = \sum_{i=1}^{t-1}(g_i - u_i)$
by unrolling recursion~\eqref{eq:deltarec}, hence
\[
    \beta_t \coloneqq \frac1t\E[\Delta_t] = \frac1t\sum_{i=1}^{t-1} (\E[g_i] - \E[u_i])
\]
is the average bias. Then, even weaker than above, $\E[\norm{\Delta_t}] = o(t)$ gives
that \uclip{} is unbiased \emph{on average}
\[
    \lim_{t \to \infty} \norm{\beta_t } 
    = \lim_{t\to\infty}\frac1t \norm{\E[\Delta_t]}
  \le \lim_{t\to\infty} \frac1t\E[\norm{\Delta_t}] 
  = 0.
\]

\begin{algorithm}
   \caption{SGD with \uclip{}}\label{alg:uclip}
    \label{alg:uclip-sgd}
    \begin{algorithmic}
        \STATE {\bfseries Input:}
            $f(\theta, \cdot)$ (loss function), 
            $S$ (dataset),
            $\theta_1 \in \R^p$ (initial parameters),
            $\eta \in \R_{>0}$ (learning rate),
            $\gamma \in \R_{>0}$ (clipping region),
            $T \in \N$ (number of timesteps).
        \STATE {\bfseries Return:} $\theta_{T+1}$ (optimized parameters).
        
        \STATE $\Delta_1 \gets 0$
        \FOR{$t=1,\dots,T$}
            \STATE Sample $Z_t \sim \unif S$
            \STATE $g_t \gets \grad_\theta f(\theta_t, Z_t)$
            \STATE $u_t \gets \clip(g_t+ \Delta_t, \gamma)$ 
            \STATE $\Delta_{t+1} \gets g_t + \Delta_t - u_t$ 
            \STATE $\theta_{t+1} \gets \theta_t - \eta u_t$  
        \ENDFOR
    \end{algorithmic}
\end{algorithm}

%-------------------------------%
\paragraph{Summary.}
%-------------------------------%
The purpose of this paper is to present \uclip{} as a simple idea that, on average,
eliminates bias in stochastic gradient clipping.
In particular, \uclip{} introduces no additional hyperparameters over standard gradient clipping.
The simplicity of \uclip{} leads to clean theoretical analyses and convergence
bounds, giving a $O(T^{-1/2})$ convergence of the expected average regret
when combined with SGD on convex objectives.
This work is intended to introduce \uclip{} as a novel concept with guarantees, rather than
to obtain state of the art results.
Nevertheless, we validate our theoretical bounds empirically on toy problems and, in addition,
find that \uclip{} (combined with SGD, momentum~\citep{sutskever2013importance} or Adam)
is better than or competitive with standard methods on CIFAR10 and ImageNet.
We believe \uclip{} is a promising candidate for further exploration into reducing bias
in stochastic gradient clipping and give
suggestions for future work in~\cref{sec:future-work}.

%-------------------------------%
\subsection{Motivation: Clipping and Harmful Bias.}
%-------------------------------%
Gradient clipping is used widely but causes a bias in the updates.
Under certain assumptions on the objective and parameters 
(clipping regions and learning rate) convergence results can be
obtained~\citep{zhang2021understanding,mai2021stability,
gorbunov2020stochastic,zhang2020improved}, but in general this bias
is an impediment from a theoretical perspective.
In this section we give examples of how the bias can be harmful.

\begin{example}[{\citet[Example 2]{chen2020understanding}}]
  Let $f_1(x) = \frac12(x - 3)^2$ and
  $f_2(x) = \frac12 (x+3)^2$. Consider a stochastic optimisation problem
  in which we receive the gradient of each function with probability $\frac12$
  on each iteration. The expected objective is
  $f(x) = \frac12 f_1(x) + \frac12 f_2(x) = \frac14 (x-3)^2 +\frac14(x+3)^2$
  which is minimized at $x=0$. Consider clipped SGD 
  and clipping updates to the region $[-1, 1]$.
  The clipped updates are uniformly distributed on $\{-1, +1\}$ when
  $x_i \in [-2, 2]$, so all of these points are `stationary' in the eyes of the
  algorithm.
\end{example}

\begin{example}[Clipping and aliasing]
  Consider clipping in a stochastic optimization problem. The goal is to
  optimize $f(x) = \E[F(x)]$ where
  \[
    F(x) =
    \begin{cases}
      & \abs{4x - 1} \quad \text{with probability } \frac14\\
      & \abs{x + 1} \quad \text{with probability } \frac34.
    \end{cases}
  \]
  Hence, $f(x) = \frac14 \abs{4x  -1} + \frac34 \abs{x + 1}$ and
  $\argmin_{x \in \R}f(x) = \frac14$ with $f(1/4) = \frac{15}{16}$.
  The algorithm receives stochastic subgradients
  \[
    D(x) =
    \begin{cases}
      & 4\sign(4x - 1) \quad \text{with probability } \frac14\\
      & \sign(x + 1) \quad \text{with probability } \frac34.
    \end{cases}
  \]
  Suppose now that we clip gradients to magnitude $2$.
  This results in clipped gradients distributed as
  \[
    \tilde{D}(x) =
    \begin{cases}
      & 2\sign(4x - 1) \quad \text{with probability } \frac14\\
      & \sign(x + 1) \quad \text{with probability } \frac34
    \end{cases}
  \]
  which are indistinguishable from the stochastic subgradients of
  \[
    \tilde{F}(x) =
    \begin{cases}
      & \frac12 \abs{4x - 1} \quad \text{with probability } \frac14\\
      & \abs{x + 1} \quad \text{with probability } \frac34.
    \end{cases}
  \]
  Set $\tilde{f}(x) = \E[\tilde{F}(x)] = \frac18 \abs{4x - 1} + \frac34\abs{x+1}$.
  Any algorithm that takes subgradients
  and is guaranteed to converge on arbitrary stochastic convex optimization problems
  (e.g.~SGD with some learning rate) must converge to 
  $\argmin_{x\in \R} \tilde{f}(x)$ if it instead receives the clipped
  gradients,
  but $\argmin_{x\in\R}\tilde{f}(x) = -1$ and $f(-1) = \frac54 > \frac{15}{16}$.
  An empirical example of this aliasing phenomenon is given in~\cref{fig:clip-divergence}.
\end{example}

\begin{figure*}[tb]
    \centering
    \begin{subfigure}[t]{\twofigwidth}
        \centering
    \includegraphics[width=\textwidth]{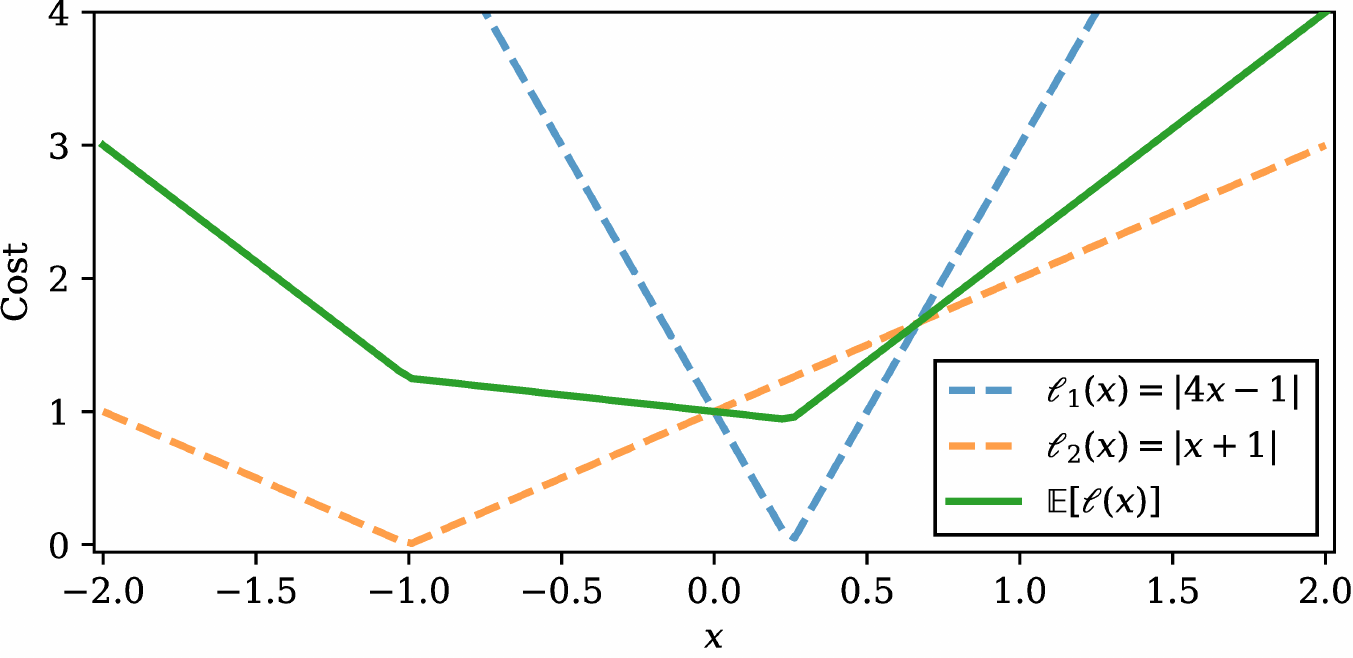}
    \end{subfigure}%
    ~ 
    \begin{subfigure}[t]{\twofigwidth}
        \centering
    \includegraphics[width=\textwidth]{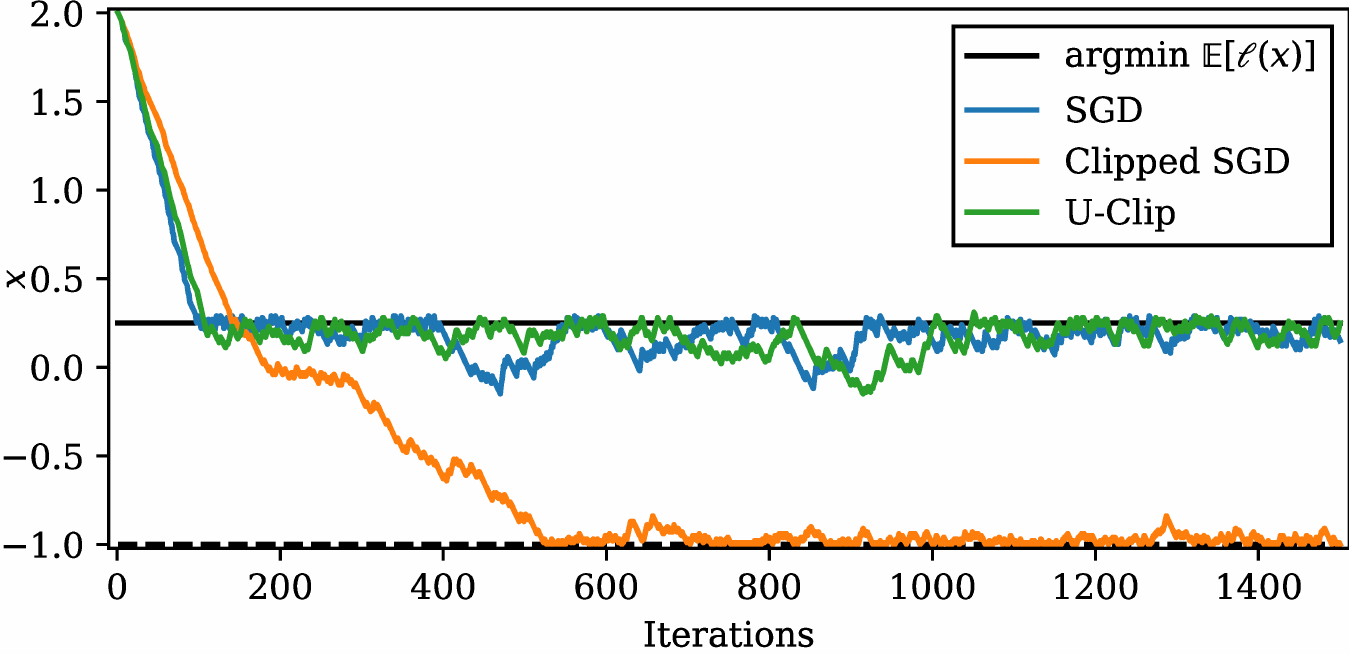}
    \end{subfigure}
    \caption{An example of clipped SGD non-convergence due to aliasing.
    Left: the objective, the algorithm receives a subgradient of $\ell_1$
    with probability $1/4$ and from $\ell_2$ with probability $3/4$.
    We write $\E[\ell(x)] \coloneqq \frac14 \ell_1(x) + \frac34 \ell_2(x)$ for the
    expected objective.
    Right: the optimization trajectories, each initialized at $x=2$ with learning
    rate $0.01$ and ran for 1500 iterations. Note that SGD and SGD with \uclip{} converge
    to the optimum, but clipped SGD arrives instead at $x=-1$ due to aliasing.
    The clip region is $\gamma = 2$ for clipped SGD and \uclip{}.}
    \label{fig:clip-divergence}
\end{figure*}

%-------------------------------%
\subsection{Related Work}
%-------------------------------%
Gradient clipping has had widespread applications for some time~\citep{pascanu2013difficulty},
is important in training large language models~\citep{rae2021scaling,chowdhery2022palm}
and plays a key role in differentially private optimisation~\citep{abadi2016deep}.
It is often motivated by control over the update size in the presence of large gradients.
Alternatives are normalized gradient descent~\citep{hazan2015beyond} (closely related to norm clipping)
and adaptive step size methods~\citep{duchi2011adaptive,kingma2014adam}.
Adaptive choice of the clip region has been studied for
norm clipping~\citep{andrew2021differentially} and coordinate-wise
clipping~\citep{pichapati2019adaclip}.
Recently, gradient clipping was a component in achieving strong generalisation
performance with full batch gradient descent~\citep{geiping2021stochastic}.

The convergence of gradient clipping in the deterministic setting with a convex 
and smooth objective is well known (e.g.~one can use~\citet[Lemma 3.4]{bubeck2015convex}).
Convergence under a relaxation of this smoothness condition is given for norm
clipping by~\citet{zhang2019gradient}, who also study the stochastic case, finding
\emph{faster} convergence for clipped SGD than vanilla SGD in their setting.
Their work is extended by~\citet{zhang2020improved}, taking momentum into account.
\citet{chen2020understanding} give a convergence proof for norm clipping in SGD
for convex smooth functions, quantify the bias due to
norm clipping and propose using perturbed gradients to mitigate it.
A high probability convergence bound for norm clipped SGD on convex smooth objectives 
with heavy-tailed gradient noise is given by~\citet{gorbunov2020stochastic}.
\citet{mai2021stability} prove convergence guarantees for SGD with norm clipping
on convex functions arbitrary growth in the subgradients
(provided the objective grows faster than a quadratic), but some of their results
require the clip region $\gamma$ to increase over time, potentially making the 
clipping redundant.

%%%%%%%%%%%%%%%%%%%%%%%%%%%%%%%%%%%%%%%%%%%%%%%%%%%%%%%%%%%%%%%
\section{Theory}\label{sec:theory}
%%%%%%%%%%%%%%%%%%%%%%%%%%%%%%%%%%%%%%%%%%%%%%%%%%%%%%%%%%%%%%%

%-------------------------------%
\paragraph{General Assumptions and Notation.}
%-------------------------------%
Our results share some basic assumptions and notation,
which we state once here to avoid repetition. 
We assume $\{f_t\}$ is a sequence of convex functions $\X \to \R$ where
$\X \subset \R^d$ is convex and
$g_t$ is a stochastic subgradient of $f_t$ and $x_t$ with
$\gbar_t \coloneqq \E[g_t\vert{} x_t] \in \partial f(x_t)$ and
$\xi_t \coloneqq g_t - \gbar_t$. This implies that $\E[\xi_t] = 0$.
We assume that each \emph{coordinate} of $g_t$ is bounded in magnitude
by $G$ and that the $\{\xi_i\}$ are conditionally independent for
$i\ne j$ given $x_i$ and $x_j$.
We assume that that $\xi_t \vert{} x_t$ is 
sub-Gaussian with variance proxy $\sigma^2_t \ge 0$ (of course,
sub-Gaussianity is provided by boundedness).
It is possible that $\sigma_t$ depends on $x_t$.
Sometimes the clip region $\gamma$ will be a scalar, in which case the same
clip region is used for each coordinate, other times it will be a vector 
specifying coordinate-wise clip regions.
We will often assume $\gamma_t \ge \abs{\gbar_t}$ so that we clip noise in the gradients
but not the signal, our theory is specifically aimed at this setting and we do not
study the regime in which the signal is clipped.
The bound on the gradients means we are free to apply the bound
$\gamma_t \le G$ for each coordinate, since $\gamma_t \ge G$
is a no-op.
We will frequently abuse notation by writing
$a > h(b)$ for $a, b \in \R^d$ and $h: \R \to \R$ which is to be interpreted
component-wise, i.e.~$a_j > h(b_j)$ for $j=1,\dots, d$, and $a > c$ where $c$
is a scalar which is to be interpreted component-wise.  We will study the
average regret, defined by 
$R_T = \frac1T\sum_{t=1}^T (f_t(x_t) - f_t(x_*))$
where $x_* = x_*(T) = \argmin_{x\in\X}\sum_{t=1}^T f_t(x)$.

%-------------------------------%
\subsection{Convergence Results}\label{sec:convergence}
%-------------------------------%
In this section we give an overview of our main results,
which are regret bounds for~\uclip{} with various clipping
schedules $\{\gamma_t\}$. 

\Cref{thm:uclip-convergence-simple} is the special case of~\cref{thm:uclip-convergence}
with the clip region $\gamma_t = \gamma$ an (absolute) constant.
In these results $\gamma_t$ is a scalar representing the same clip region for each coordinate.
Per-coordinate clip regions are possible by applying~\cref{prop:carry-bounded}
element-wise in the proof of~\cref{thm:uclip-convergence}.
\begin{theorem}\label{thm:uclip-convergence-simple}
  Consider \uclip{} with learning rate $\eta = \frac{1}{\sqrt{T}}$
  and clipping regions $G \ge \gamma > \abs{\gbar_t}$ component-wise.
  The expected average regret is bounded by
  \iftwocolumn
  \begin{align*}
      \E[R_T] &\le \frac{8 G^4 d^{3/2}\gamma}{\alpha^3\sqrt{T}}  \\
      &\phantom{\le}+ \frac{4G\sqrt{d}\gamma + (4\sqrt{d} + d)\gamma^2 + \norm{x_1 - x_*}^2 }{2\sqrt{T}}\\
      &\phantom{\le}+  \left(\frac{4dG^4}{\alpha^3} + G + \gamma\right)\frac{\norm{x_1 - x_*}}{T}\\
     &= O \left(\frac{G^4d^{3/2}}{\alpha^3 \sqrt{T}}\right)
  \end{align*}
  \else
    \begin{align*}
      \E[R_T] &\le \frac{8 G^4 d^{3/2}\gamma}{\alpha^3\sqrt{T}}  
      + \frac{4G\sqrt{d}\gamma + (4\sqrt{d} + d)\gamma^2 + \norm{x_1 - x_*}^2 }{2\sqrt{T}}
      +  \left(\frac{4dG^4}{\alpha^3} + G + \gamma\right)\frac{\norm{x_1 - x_*}}{T}\\
     &= O \left(\frac{G^4d^{3/2}}{\alpha^3 \sqrt{T}}\right)
    \end{align*}
    \fi
  where $x_* = x_*(T) = \argmin_{x\in\X}\sum_{t=1}^T f_t(x)$ and
  $\alpha = \inf(\gamma_t - \abs{\gbar_t})$.
\end{theorem}

\Cref{thm:uclip-convergence-simple} stems from~\cref{lemma:gradient-transformation}
which requires a bound on the carry.
If $\alpha$ becomes 0, then the carry is a random walk $\sum_t \xi_t$ which may diverge in expectation.
A factor of $d$ appears because the clipping is element-wise, but note that the 
bound on the \emph{norm} of the gradients is $\norm{g_t} \le G\sqrt{d}$.
It's possible that this factor would not appear in the corresponding result
for norm clipping.

In respect of the more general~\cref{thm:uclip-convergence} with time-varying $\gamma$,
it may be possible to craft the sequence $\{\gamma_t\}$ such that 
$\gamma_t \to 0$ as the gradients $g_t \to 0$. In which case \uclip{} would be
\emph{asymptotically unbiased}.\footnote{This is not possible in all circumstances, 
e.g.~$f(x) = \abs{x}$.}

We loosely define \emph{adaptive clipping} as any clip region schedule
that adapts to the noise in the gradients.
We now give a convergence result for an instance of adaptive clipping that allows us to improve the 
dependence on $G$, in particular by taking $\gamma_t = \abs{\gbar_t} + \beta \sigma_t^2$
for $\beta > 0$.%
\footnote{Unfortunately, this clip region is not scale equivariant, so for our 
practical experiments we prefer regions of the form 
$\gamma_t = \abs{\gbar_t} + \beta \sigma_t$
(covered by~\cref{thm:uclip-convergence}, if one has a constant lower bound on
$\sigma_t$).
A direct argument with this (scale equivariant) clip region
is likely possibly, albeit by different techniques.
}

\begin{theorem}[Adaptive \uclip{}]\label{thm:adaptive-clipping}
  Consider \uclip{} with clip regions 
  $\gamma_t = \abs{\gbar_t} + \beta \sigma_t^2$ 
  (component-wise) for $\beta > 0$ and learning rate $\eta = \frac{1}{\sqrt{T}}$.
  Suppose there is a constant $\sigmamin \le \sigma_t$ $\forall t$.
  Then the expected regret is bounded by
  \iftwocolumn
  \begin{align*}
      \E[R_T] &\le
      \frac{16d^{3/2} G}{\beta^3 \sigmamin^2 \sqrt{T}} \\
      &\phantom{\le}+ \frac{4(2+ \beta)\sqrt{d}G^2 + dG^2 +\norm{x_1 - x_*}^2 }{2\sqrt{T}}\\
     &\phantom{\le}+ \left( 8 d\beta^{-3} \sigmamin^{-2} + (2 + \beta)G\right)
    \frac{\norm{x_1 - x_*}}{T}\\
      &=O \left( \frac{d^{3/2} G}{\beta^3 \sigmamin^2 \sqrt{T}} + \frac{dG^2}{\sqrt{T}}\right)
  \end{align*}
  \else
  \begin{align*}
      \E[R_T] &\le
      \frac{16d^{3/2} G}{\beta^3 \sigmamin^2 \sqrt{T}} 
      + \frac{4(2+ \beta)\sqrt{d}G^2 + dG^2 +\norm{x_1 - x_*}^2 }{2\sqrt{T}}
     + \left( 8 d\beta^{-3} \sigmamin^{-2} + (2 + \beta)G\right) \frac{\norm{x_1 - x_*}}{T}\\
      &=O \left( \frac{d^{3/2} G}{\beta^3 \sigmamin^2 \sqrt{T}} + \frac{dG^2}{\sqrt{T}}\right)
  \end{align*}
    \fi
\end{theorem}
\begin{proof}
  Identical to the proof of~\cref{thm:uclip-convergence} but using 
  $\gamma_t \le G$ WLOG and instead 
  applying~\cref{prop:adaptive-clipping} to get
  $\Delta_+ = 8 d\beta^{-3} \sigmamin^{-2} + (2 + \beta)G$.
\end{proof}

%-------------------------------%
\subsection{Gradient Transformations}
%-------------------------------%
At the base of our analysis is the following lemma, which may
be of independent interest. It says that in the stochastic-online setting
with an additive update rule, if the cumulative difference between the
updates and the subgradients remains bounded, then the regret of the update rule
is also bounded. Intuitively, if an update rule is globally sufficiently 
similar to gradient descent, then it enjoys the same convergence rate.
Note that the gradient descent rate can be achieved with $\eta = O(1/\sqrt{T})$
and $D_t = D$ (defined in the lemma) constant.

\newcommand{\gtlemma}{%
  Let $\{f_t\}$ be a sequence of convex functions $\X \to \R$
  where $ \X \subset \R^d $.
  Consider the update equations 
  \[
    x_{t+1} = x_t - \eta u_t 
    \quad \text{ and }  \quad
    \Delta_{t+1} = \Delta_t + g_t - u_t
  \]
  where $g_t$ is a stochastic subgradient of $f_t$ at $x_t$ satisfying
  $\E[g_t \vert{} x_t] \in \partial f_t(x_t)$, $\Delta_1=0$ and
  $\eta$ is the learning rate.
  We call the first sequence the iterates and the second the carry.
  Set $x_* = x_*(T) = \argmin_{x}\frac1T \sum_{t=1}^Tf_t(x)$.
  Suppose that the updates are bounded
  almost surely by $\norm{u_t} \le \Gamma_t$, and let $D_t$ be a non-decreasing
  sequence.
  \begin{enumerate}[1)]
    \item If the carry is bounded $\E[\norm{\Delta_t}] \le D_t$ in expectation
      for $t = 1, \dots, T+1$ then the expected (average) regret 
      is bounded by 
      \ifmainpaper
      \label{point:gtlemma-1}
      \fi

      \iftwocolumn 
      \begin{align*}
          \E[R_T] \le 
          &\phantom{\le}\frac{2\eta D_{T+1}}{T} \sum_{t=1}^T \Gamma_t
          + \frac{\eta}{2T}\sum_{t=1}^T \Gamma_t^2\\
          &\phantom{\le}+ \frac{1}{2 \eta T} \norm{x_1 - x_*}^2
          + \frac{D_{T+1}}{T} \norm{x_1 - x_*}.
      \end{align*}
      \else 
          \[
          \E[R_T] \le 
          \frac{2\eta D_{T+1}}{T} \sum_{t=1}^T \Gamma_t
          + \frac{\eta}{2T}\sum_{t=1}^T \Gamma_t^2
          + \frac{1}{2 \eta T} \norm{x_1 - x_*}^2
          + \frac{D_{T+1}}{T} \norm{x_1 - x_*}.
          \]
      \fi
    \item Recall the definition $\xi_t = g_t - \E[g_t\vert{} x_t]$.
      Suppose that the variables $Z_t \coloneqq \xi_t \vert{} x_t$
      (independent by assumption) are
      sub-Gaussian with variance proxy
      $\sigma^2$. Assume also that the iterates are bounded
      $\sup_t \norm{x_t - x_*} \le R$.
     Then if $\norm{\Delta_t} \le D_t$ with probability at least $1-\delta$, 
     then with probability at least $1-2\delta$ we have 
        \ifmainpaper
        \label{point:gtlemma-2}
        \fi
     \iftwocolumn 
     \begin{align*}
         R_T &\le \frac{2\eta D_{T+1}}{T} \sum_{t=1}^T \Gamma_t
        + \frac{\eta}{2T}\sum_{t=1}^T \Gamma_t^2 \\
         &\phantom{\le}+ \frac{1}{2 \eta T} \norm{x_1 - x_*}^2 + \frac{D_{T+1}}{T} \norm{x_1 - x_*}\\
         &\phantom{\le}+ \frac{\sqrt{2}R\sigma}{\sqrt{T}}\log(2/\delta).
     \end{align*}
     \else % appendix
        \[
         R_T \le \frac{2\eta D_{T+1}}{T} \sum_{t=1}^T \Gamma_t
        + \frac{\eta}{2T}\sum_{t=1}^T \Gamma_t^2 
        + \frac{1}{2 \eta T} \norm{x_1 - x_*}^2 %
        + \frac{D_{T+1}}{T} \norm{x_1 - x_*}\frac{\sqrt{2}R\sigma}{\sqrt{T}}\log(2/\delta).
        \]
     \fi
  \end{enumerate}
}
\begin{lemma}[Gradient Transformation]\label{lemma:gradient-transformation}
    \gtlemma{}
\end{lemma}

The proof is in~\cref{sec:proof-gradient-transformation}.
Some remarks are in order.
First, although sufficient, the condition $\norm{\Delta_t}$ bounded (either in expectation
or in high probability) is not necessary for convergence.
For instance there is the trivial situation
$g_t = 0$ and $u_t = -1$ $\forall t$ with zero regret but
$\Delta_t \sim t$. Second, in the case of \uclip{} we assume bounded gradients,
so we may always take the clipping region $\gamma_t \le \Gamma_t \le G$,
hence the $\Gamma_t$ terms in~\cref{lemma:gradient-transformation}
can be treated, albeit crudely, by $\frac1T \sum_{t=1}^T \Gamma_t \le \sqrt{d}G$ and
$\frac1T \sum_{t=1}^T \Gamma_t^2 \le dG^2$.
Third, from~\cref{lemma:gradient-transformation},
the condition $\E[\norm{\Delta_t}] = o(t)$ is sufficient for convergence.
Note that this is the same condition required for \uclip{} to be unbiased 
on average (c.f.~\cref{sec:introduction}).
Finally, an alternative version of~\cref{lemma:gradient-transformation} with
a $p$-norm bound on the carry $\norm{\Delta_t}_p \le D_t$
is easily available by using H\"{o}lder's inequality
rather than Cauchy-Schwarz towards the end of the proof.

%-------------------------------%
\subsection{Carry Bounds}
%-------------------------------%
In this section we give bounds on the carry for component-wise clipping
for clip regions satisfying $\gamma_t > \abs{\E[g_t\vert{} x_t]}$
element-wise.
In particular, the $\gamma_t$ may be chosen as a function of past
gradients as long as the sequence is bounded 
and the above assumption is satisfied.
We leave a carry bound for norm clipping to future work.
Using~\cref{lemma:gradient-transformation} we can translate bounds on
the carry of update rules into convergence results. 
For instance, if $\eta = 1/\sqrt{T}$ and 
$\E[\norm{\Delta_t}] = O(1)$ then the expected regret is
$O(1/\sqrt{T})$.
Note that when doing this, we \emph{do not} need to apply a union
bound over $t$.
We have previously assumed that the functions
$\{f_t\}$ are convex. By inspecting the proofs,
the reader will find that this assumption is not
necessary for the carry bounds in this section.

The bounds are given for scalar variables, but the results
can be applied to control $\norm{\Delta_t}$ either by a union bound over
the co-ordinates (not time) or
by application of the expectation results in~\cref{prop:carry-bounded,prop:adaptive-clipping}.
Doing this introduces an additional factor of the dimension $d$.
Results for component specific clip regions $\gamma \in \R^d$ can be obtained this way.
The idea of the proofs in this section is that it is  sufficient to consider the 
behavior of $\Delta_t$ when it is large and not worry about fine-grained
behavior around 0. The proof of~\cref{prop:carry-bounded} is in~\cref{sec:proof-carry-bounded}.
A proof sketch is given below. Notice that when $\alpha =0$ the carry becomes a random walk,
consistent with the $\sqrt{t}$ divergence in the bound.

\newcommand{\propcarrybounded}{%
  In this result, all variables are real scalars.
  Consider the following update rules with $\Delta_1 = 0$
  \[
    u_t = \clip(g_t + \Delta_t, \gamma_t)
    \quad \text{ and } \quad
    \Delta_{t+1} = g_t +\Delta_t - u_t 
  \]
  where $\clip(x, \gamma) = \min\{\max\{g_t + \Delta_t, -\gamma\}, \gamma\}$
  projects $x$ onto the interval $[-\gamma, \gamma]$.
  We assume that the sequence $\{\gamma_t\}$ 
  satisfies $\gamma_+ = \sup_t \gamma_t < \infty$
  and $\gamma_t > \abs{\gbar_t}$.
  Let $\alpha \ge 0$ be a real constant such that
  $\alpha \le \inf_{t} (\gamma_t - \abs{\gbar_t})$.
  Then for any $\epsilon > 0$
  \[
    \P(\abs{\Delta_t} \ge \epsilon + G + \gamma_+) 
    \le 
    2c_\alpha
    \exp\left(-\frac{\epsilon^2 + 2\alpha\epsilon t}{2G^2t}\right)
  \]
  where 
  $c_\alpha = 
    \left(\exp\left(\frac{\alpha^2}{2G^2}\right) -1\right)^{-1}  \le 2G^2 \alpha^{-2}$
  and consequently
  \[
    \E[\abs{\Delta_t}] \le \frac{2c_\alpha G^2}{\alpha} + G + \gamma_+.
  \]
  }
\begin{proposition}[Carry bound: component-wise clipping]\label{prop:carry-bounded}
    \propcarrybounded{}
\end{proposition}

\begin{proof}[Sketch Proof]
    If $\Delta_t \ge G + \gamma_+$ then $g_t + \Delta \ge \gamma_t$ and
    clipping is guaranteed on this iteration. This gives
    $\Delta_{t+1} = \Delta_t + \gbar_t + \xi_t - \gamma_t$.
    If this happens for $j$ iterations then
    $\Delta_{t+j} = \Delta_t + \sum_{i=0}^{j-1}(\gbar_{t+i} - \gamma_{t+i}) + \sum_{i=0}^{j-1}\xi_{t+i}$.
    Using $\abs{\sum_{i=0}^{j-1}(\gbar_{t+i} - \gamma_{t+i})} \le j\alpha$ we apply Hoeffding's
    inequality, then integrate the bound to get the expectation result.
\end{proof}

\begin{corollary}
  \Cref{prop:carry-bounded} implies, using the bound on $c_\alpha$, 
  \[
    \P(\abs{\Delta_t} \ge \epsilon + G + \gamma_+) 
    \le 
    4G^2\alpha^{-2}
    \exp\left(-\frac{\alpha\epsilon }{G^2}\right)
  \]
  which in turn gives that with probability at least $1-\delta$
  \[
    \abs{\Delta_t} \le \frac{G^2}{\alpha^2} \log\left(\frac{4G^2}{\alpha^2\delta}\right) + 2G.
  \]
\end{corollary}

\Cref{prop:adaptive-clipping} below is the carry bound on which~\cref{thm:adaptive-clipping}
relies on. 
The proof is basically the same but we give it in~\cref{sec:proof-adaptive-clipping}
for completeness. 
Although not scale equivariant, it gives a reduced dependence on $G$.

\newcommand{\propadaptiveclipping}{%
  Set $\gamma_t = \abs{\gbar_t} + \beta \sigma_t^2 \le (1+\beta)G$ 
  for some $\beta > 0$.
  Consider the following update rules with $\Delta_1 = 0$
  \[
    u_t = \clip(g_t + \Delta_t, \gamma_t)
    \quad \text{ and } \quad
    \Delta_{t+1} = g_t +\Delta_t - u_t 
  \]
  where $\clip(x, \gamma) = \min\{\max\{g_t + \Delta_t, -\gamma\}, \gamma\}$
  projects $x$ onto the interval $[-\gamma, \gamma]$.
  Suppose there is a constant $\sigmamin \le \sigma_t$ $\forall i$.
  Then with probability at least $1- \delta$, setting 
  and $m_t = \min\{t, \beta^{-2}\sigmamin^{-2}/2\}$
  we get
  \iftwocolumn 
  \begin{align*}
    \abs{\Delta_t} 
      &\le \frac{2}{\beta}\log\left(\frac{2m_t}{\delta}\right) + (2+\beta)G  \\
      &\le \frac{2}{\beta}\log\left(\frac{\beta^{-2} \sigmamin^{-2}}{ \delta}\right) + (2+\beta)G 
    .
  \end{align*}
  \else 
  \[
    \abs{\Delta_t} 
      \le \frac{2}{\beta}\log\left(\frac{2m_t}{\delta}\right) + (2+\beta)G
      \le \frac{2}{\beta}\log\left(\frac{\beta^{-2} \sigmamin^{-2}}{ \delta}\right) + (2+\beta)G 
    .
  \]
  \fi
  It follows that
  \[
    \E[\abs{\Delta_t}] \le \frac{8}{\beta^3\sigmamin^2} + (2+\beta)G.
  \]
  }
\begin{proposition}[Carry bound: adaptive clipping]\label{prop:adaptive-clipping}
    \propadaptiveclipping{}
\end{proposition}

\begin{remark}[Proportional adaptive clipping]\label{remark:proportional-adaptive-clipping-noise}
  We define proportional adaptive clipping by the 
  clip region $\gamma_t = c\abs{\gbar_t}$.
  In~\cref{prop:adaptive-clipping} this can be achieved by an assumption
  that the gradient noise variance scales with $\abs{\gbar}$.
  However, we can still say something without assumptions on the noise.
  If $\E[\abs{\Delta_t}]$ is bounded by a constant then we know
  from~\cref{lemma:gradient-transformation} that the regret of the
  average iterate is $O(t^{-1/2})$. 
  On the other hand, if $\E[\abs{\Delta_t}]$ is unbounded in $t$
  then by comparison with~\cref{prop:carry-bounded} we see that
  $\alpha = \inf_t(\gamma_t - \abs{\gbar_t}) = (c-1)\abs{\gbar_t} =0$,
  which in turn implies that the algorithm converges (qualitatively)
  on a domain of diameter $R$ by linearisation
  \[
    \inf_t(f(x_t) - f(x_*)) 
    \le \inf_t(g_t^\top(x_t - x_*))
    \le R \inf_t\norm{g_t} = 0.
  \]
  If, in addition, we know that $\E[\abs{\Delta_t}] = \Omega(b_t)$ for some sequence
  $\{b_t\}$,
  then we get that $\alpha = o(b_t^{-1/3})$ from the bound on $c_\alpha$ 
  in~\cref{prop:carry-bounded} which gives a rate.
\end{remark}

%%%%%%%%%%%%%%%%%%%%%%%%%%%%%%%%%%%%%%%%%%%%%%%%%%%%%%%%%%%%%%%
\section{Experiments}\label{sec:experimental-details}
%%%%%%%%%%%%%%%%%%%%%%%%%%%%%%%%%%%%%%%%%%%%%%%%%%%%%%%%%%%%%%%

%-------------------------------%
\subsection{Validation of Theoretical Results}\label{sec:validation}
%-------------------------------%
In this section we analyze the results from~\cref{sec:theory} empirically on a simple example.
We run SGD with \uclip{} for 10k iterations, initialized at $x=100$ with learning rate $0.1$ and 
other parameters varying. The (expected) objective is $f(x) = \abs{x}$
and the algorithm receives stochastic
sub-gradients $g_t = \grad f(x_t) + \xi_t$ where $\xi_t$ is independent, symmetric uniform noise.
In this problem, the maximal gradient is related to the noise level (the bounds on
the uniform distribution), which in turn is related to the variance (or variance proxy).
We use the formula $\sigma^2 = \frac{(l_2 - l_1)^2}{4}$ where $\xi_t \sim \unif[l_1, l_2]$.
This is the basic estimate of the sub-Gaussian variance proxy for any bounded random variable,
which is why we use it (despite it not being tight for the uniform distribution).

In~\cref{fig:validation-standard,fig:validation-adaptive} we report our validation results
for~\cref{prop:carry-bounded,prop:adaptive-clipping} respectively.
The charts compare the high probability bounds for $\abs{\Delta_t}$
at failure rate $\delta = 0.01$ to the empirical 99\textsuperscript{th} 
percentile. We vary each parameter independently. In general, up to a constant scaling,
we see good agreement between prediction and experiment.

In~\cref{fig:convergence-validation} we validate~\cref{lemma:gradient-transformation}
with expected carry bounds coming from~\cref{prop:carry-bounded,prop:adaptive-clipping} 
respectively. We report the regret bound from~\cref{lemma:gradient-transformation}
with both the predicted expected carry and the observed carry, finding that the 
regret bound is essentially tight in the latter case (so the slack in the overall result
comes from the expectation bound). In a way this is to be expected, because $f(x) = \abs{x}$
saturates the linearised analysis in~\cref{lemma:gradient-transformation}.

\begin{figure*}[tb]
    \centering
    \begin{subfigure}[t]{\twofigwidth}
        \centering
  \includegraphics[width=\textwidth]{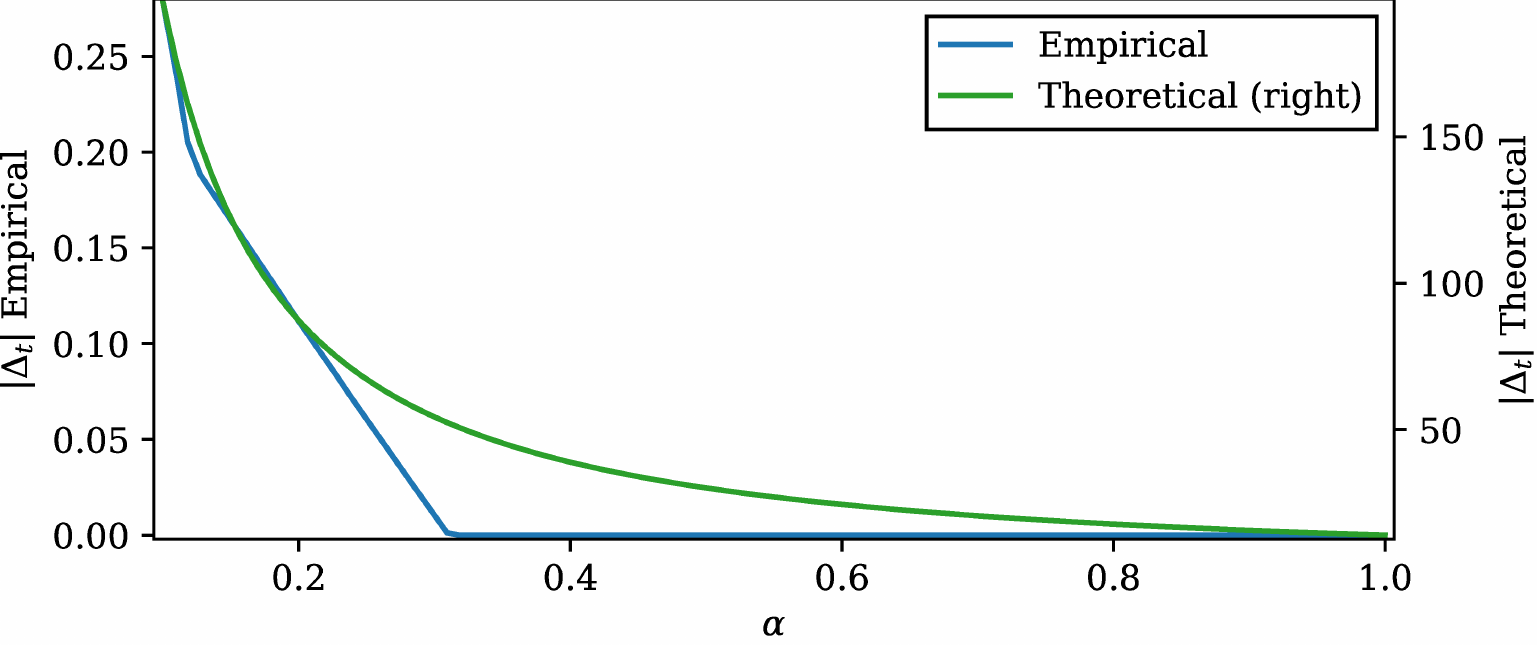}
    \end{subfigure}%
    ~ 
    \begin{subfigure}[t]{\twofigwidth}
        \centering
  \includegraphics[width=\textwidth]{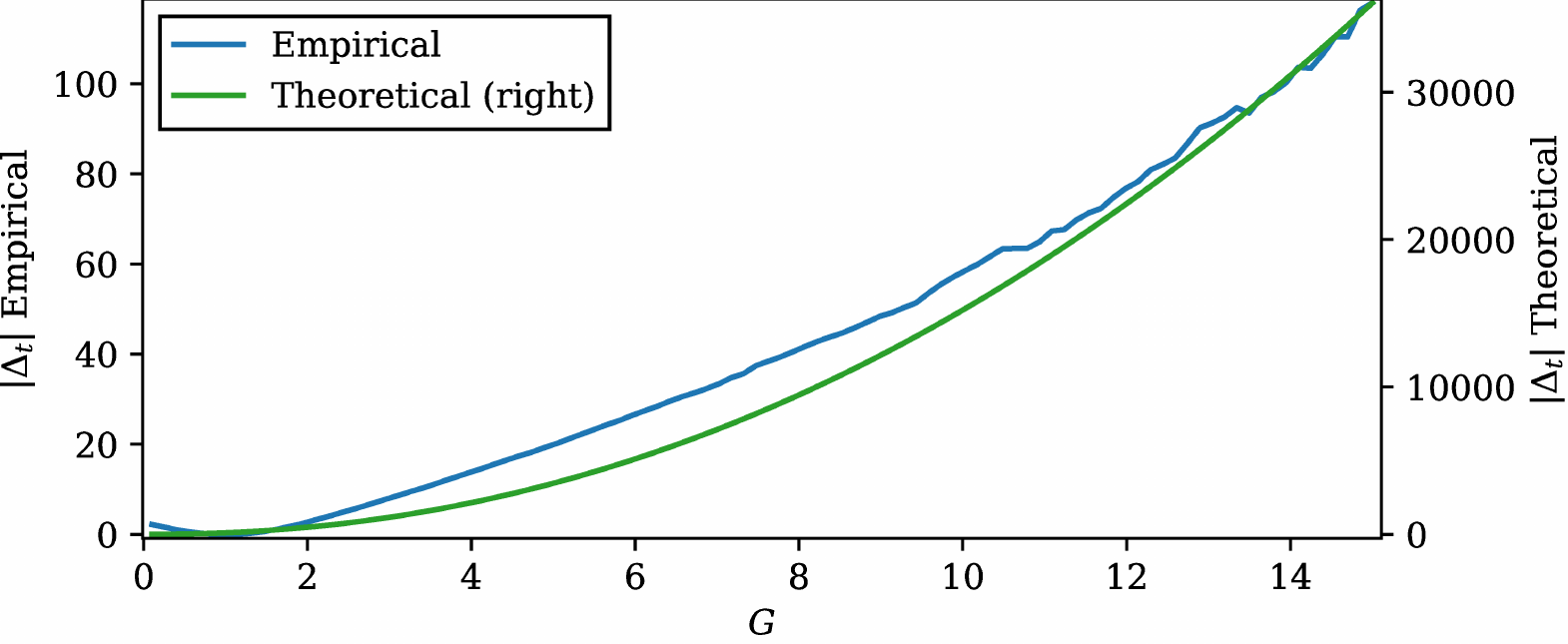}
    \end{subfigure}
    \caption{
      Validation of~\cref{prop:carry-bounded} at failure rate $\delta=0.01$,
      using clip region $\gamma = \gbar + \alpha$.
      Left: vary $\alpha \in [0.1, 1]$ with $G$ fixed by $\sigmamin^2=0.1$.
      Right: vary $G \in [0.1, 15]$ with $\alpha=0.1$.
  }
    \label{fig:validation-standard}
\end{figure*}

\begin{figure*}[tb]
    \centering
    \begin{subfigure}[t]{\twofigwidth}
        \centering
  \includegraphics[width=\textwidth]{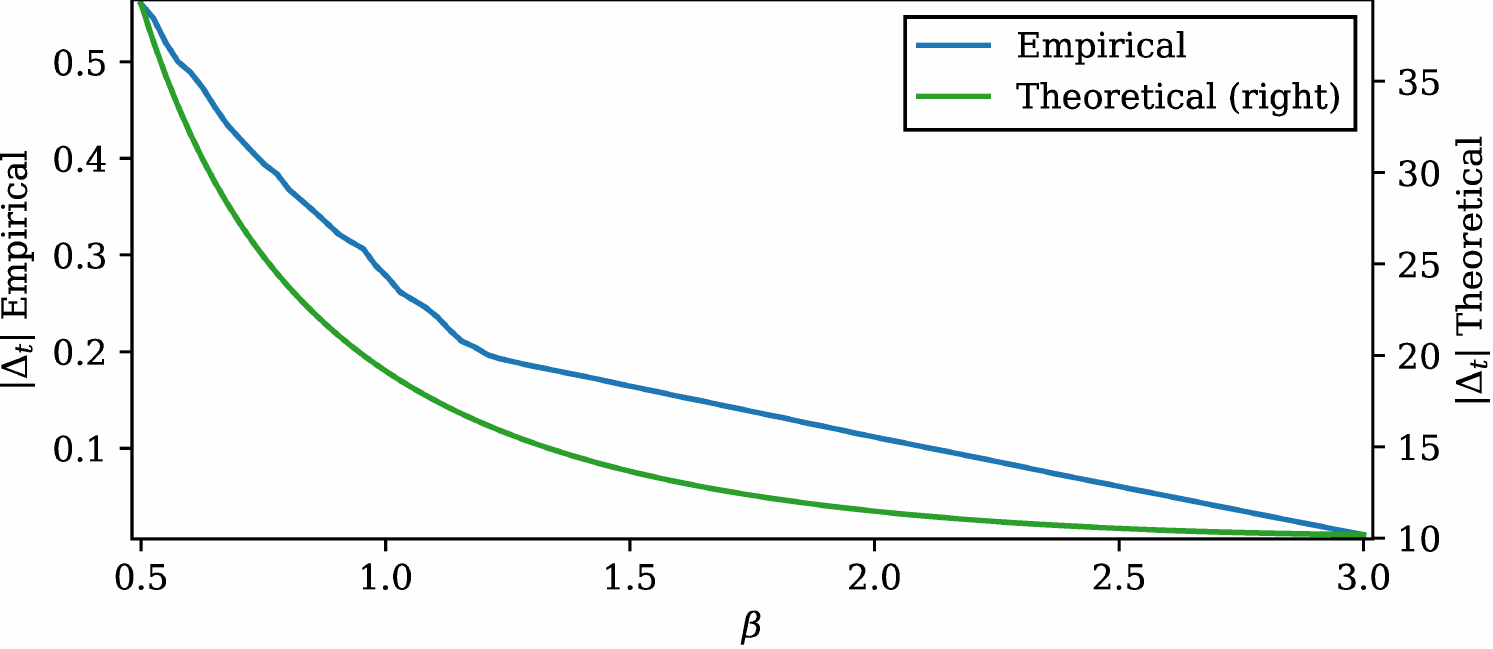}
    \end{subfigure}%
    ~ 
    \begin{subfigure}[t]{\twofigwidth}
        \centering
  \includegraphics[width=\textwidth]{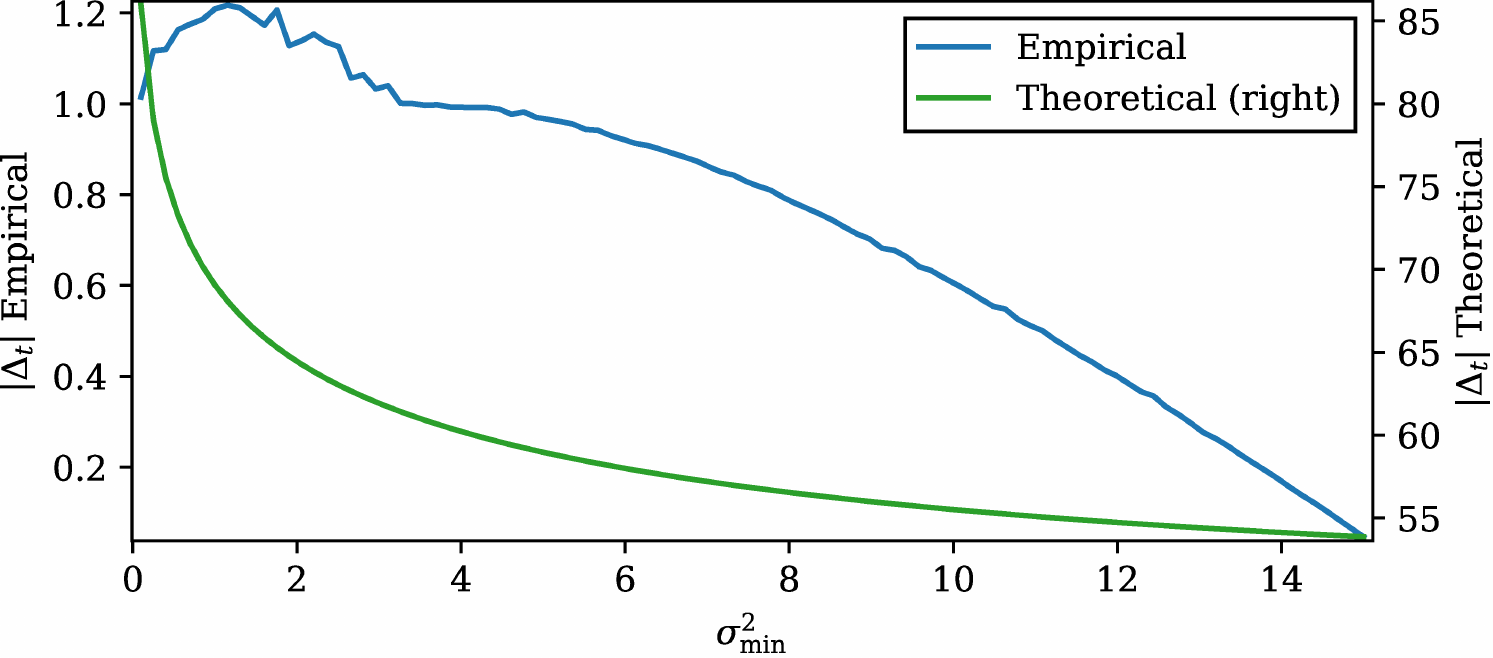}
    \end{subfigure}
    \caption{Validation of~\cref{prop:adaptive-clipping}, adaptive clipping with
    $\gamma = \gbar + \beta \sigma^2$.
    Left: vary $\beta \in [0.5, 3]$ with $\sigmamin^2 = 0.1$.
    Right: vary $\sigmamin^2 \in [0.1, 15]$ with $\beta=0.25$.
    }
    \label{fig:validation-adaptive}
\end{figure*}

\begin{figure*}[tb]
    \centering
    \begin{subfigure}[t]{\twofigwidth}
        \centering
  \includegraphics[width=\textwidth]{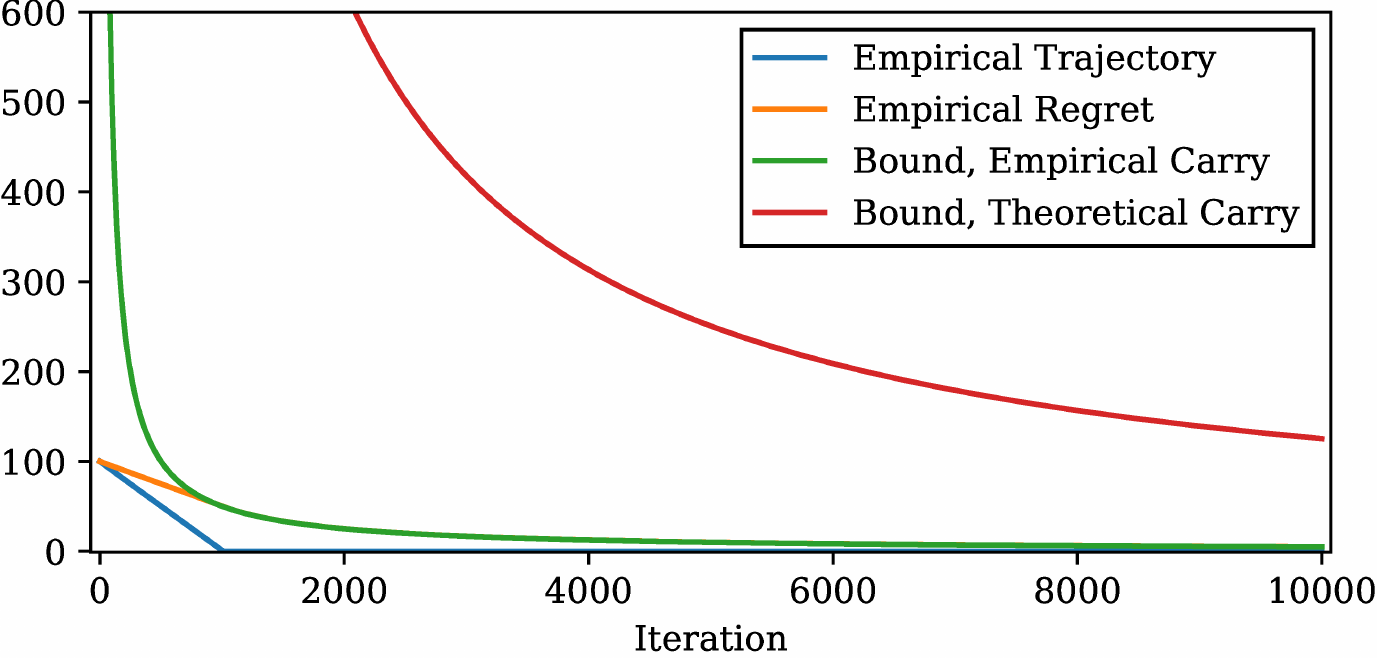}
    \end{subfigure}%
    ~ 
    \begin{subfigure}[t]{0.5\textwidth}
        \centering
  \includegraphics[width=\textwidth]{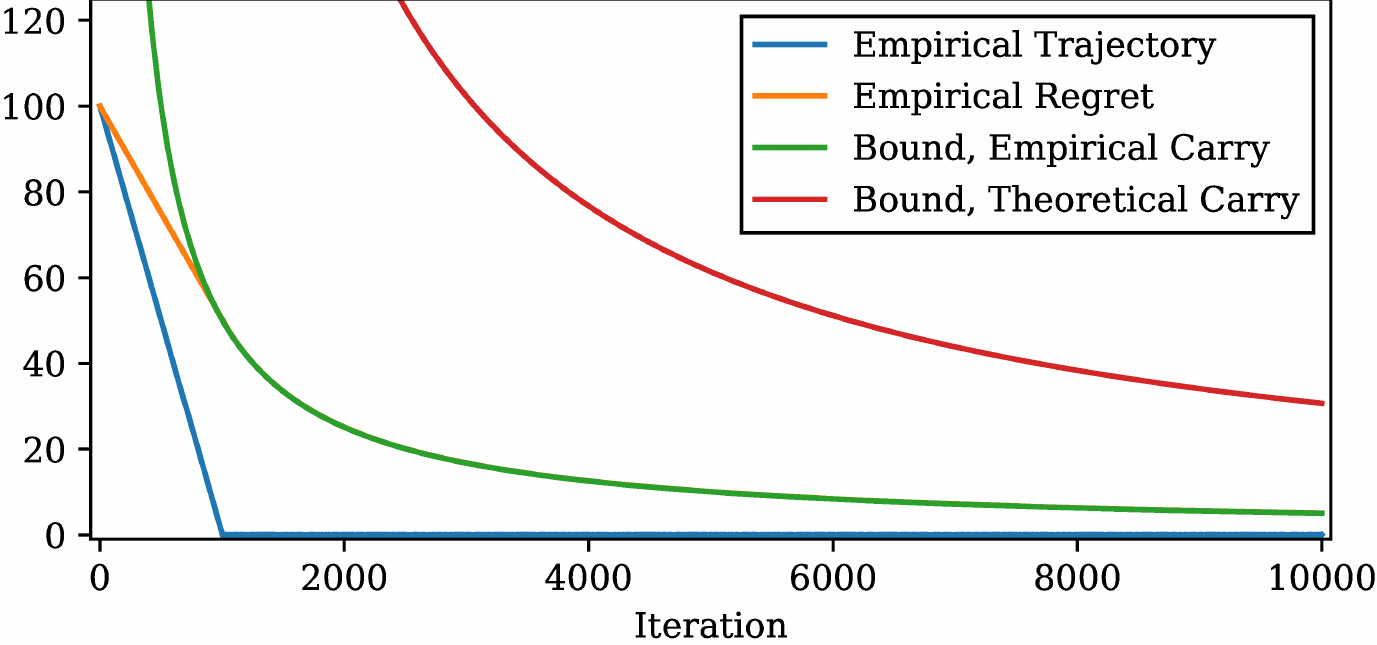}
    \end{subfigure}
    \caption{Validation of~\cref{lemma:gradient-transformation} using expected carry 
    bounds.
    Left: carry bound from~\cref{prop:carry-bounded} with clip region
    $\gamma = \gbar + \alpha$, $\alpha = 0.1$ and $\sigmamin^2 =0.1$.
    Right: carry bound from~\cref{prop:adaptive-clipping} with clip region
    $\gamma = \gbar + \beta \sigma^2$, $\beta = 0.25$ and $\sigmamin^2=0.2$.}
    \label{fig:convergence-validation}
\end{figure*}

%-------------------------------%
\subsection{Practical Performance}
%-------------------------------%
In this section we report the performance of \uclip{} when combined 
(as described in~\cref{sec:introduction})
with standard optimizers (SGD, momentum, Adam) to optimize a small conv-net on
CIFAR10~\citep{krizhevsky2009learning}.  We also provide validation for using
\uclip{} with momentum to train
a ResNet50v2 on ImageNet~\citep{deng2009imagenet}.
We compare the performance of \uclip{} to that of
the base optimizer and of the base optimizer with clipping (but no carry).
The purpose of this section is to demonstrate that \uclip{} works in practical settings.
Although \uclip{} does outperform the base optimizer in some settings,
we make no attempt to achieve state of the art results.

We explore the performance of various clipping regimes: component clipping
and norm clipping with $\gamma \in \R$ constant (same region for each parameter),
and component clipping with clip region chosen adaptively as a function of past gradients.
In the adaptive case, we use $\gamma_t = a |\mhat_t| + b \shat_t$ for
$a,b \in \R_+$
where $\mhat_t$ and $\shat_t$ are component-wise estimates of the mean and standard-deviation
of each parameter respectively.

We consider two methods of estimation.
The first is the sample variance using Welford estimation 
with the standard iterative calculation of the sample mean, 
then using $\shat_t = \sqrt{\texttt{variance}}$.
The second is exponentially weighted moving average (EWMA) of the first and second
moments with decay factor $0.95$, then using $\shat_t = \sqrt{\texttt{second moment}}$.

%-------------------------------%
\subsubsection{CIFAR10}\label{sec:cifar10}
%-------------------------------%
We train a small convolutional net on the CIFAR10 dataset to minimize the
cross-entropy loss.
We explore the performance of \uclip{} in combination
with SGD, momentum (parameter $0.9$) and Adam 
(parameters $\beta_1 = 0.90$ and $\beta_2 = 0.999$). 
We refer to these as the base optimizers.
We calculate the median, minimum and max number of epochs needed to reach 99\% training accuracy
for each configuration over 5 seeds. The learning rate is tuned to the batch size
according to the median (smaller is better) across 5 log-spaced values. 
Further experimental details are given in~\cref{sec:cifar-exp-details}

We report a selection of the results here with the rest in~\cref{sec:cifar-additional-results}.
We find that for smaller batch sizes ($\le$ 100) \uclip{} rarely slows
optimization when compared to the base optimizer and sometimes gives a significant
advantage. On the other hand, for large batch sizes the carry seems to slow training
or make it less stable. We believe the benefit for small batch sizes, 
as shown in~\cref{fig:cifar-norm-0.5-results}, is due to the memory of the carry 
smoothing out mini-batch noise. For small batch sizes this noise is prevalent
so the carry is helpful.  With larger batch sizes 
large gradients are more likely due (possibly transitory) large signals from the underlying
expected objective.
In this case the carry is harmful because a region of large gradients could make the
carry large for many iterations.

\begin{figure*}[htbp]
    \centering
    \begin{subfigure}[t]{\twofigwidth}
        \centering
  \includegraphics[width=\textwidth]{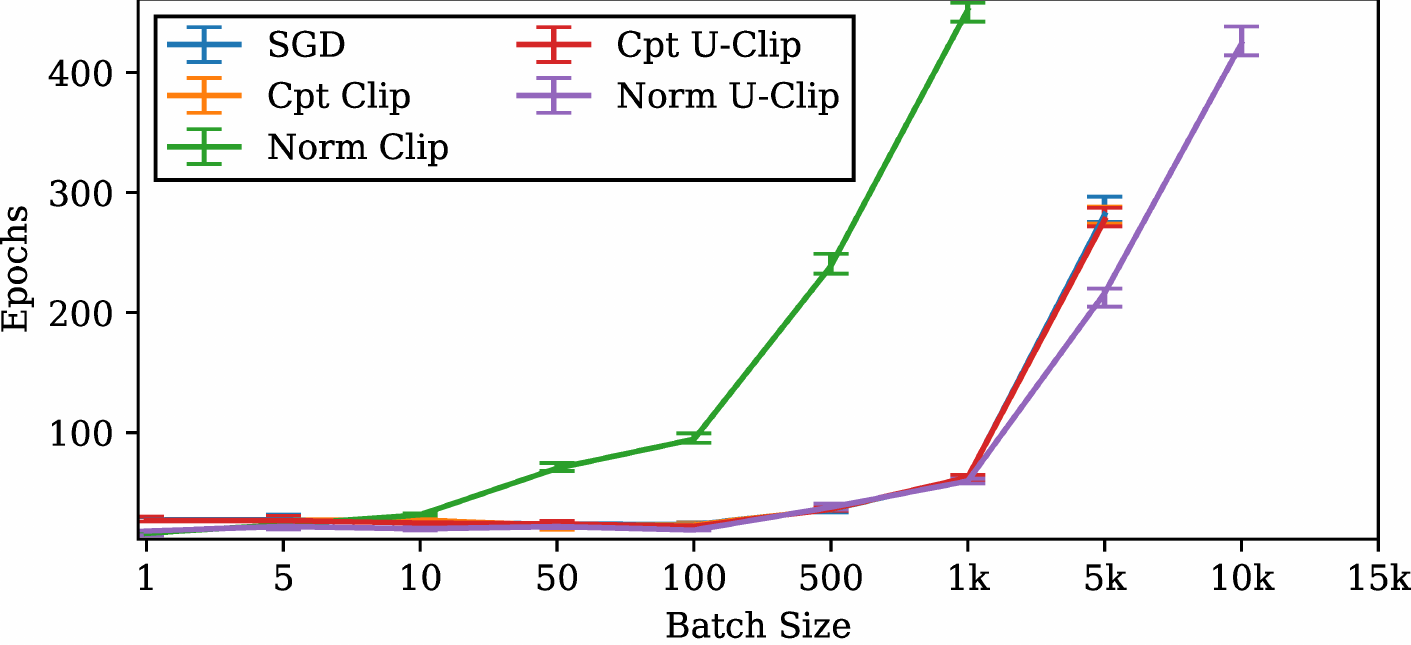}
    \end{subfigure}%
    ~ 
    \begin{subfigure}[t]{\twofigwidth}
        \centering
  \includegraphics[width=\textwidth]{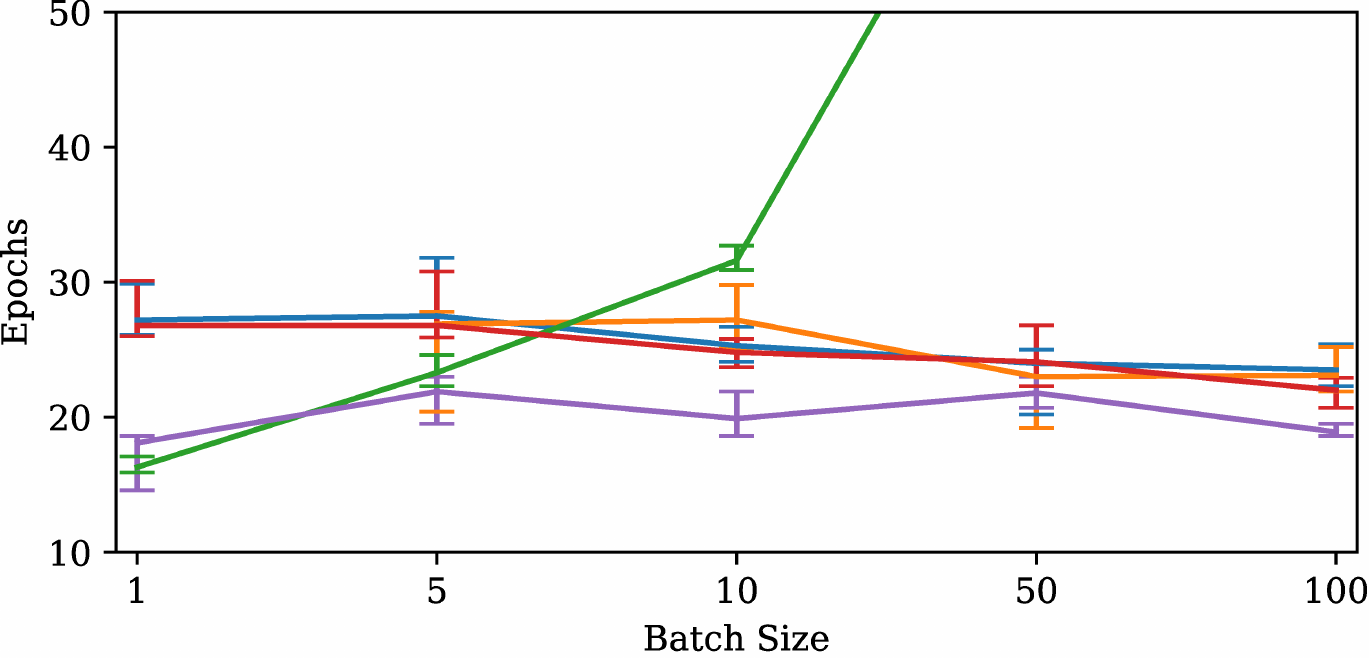}
    \end{subfigure}

    \begin{subfigure}[t]{\twofigwidth}
        \centering
  \includegraphics[width=\textwidth]{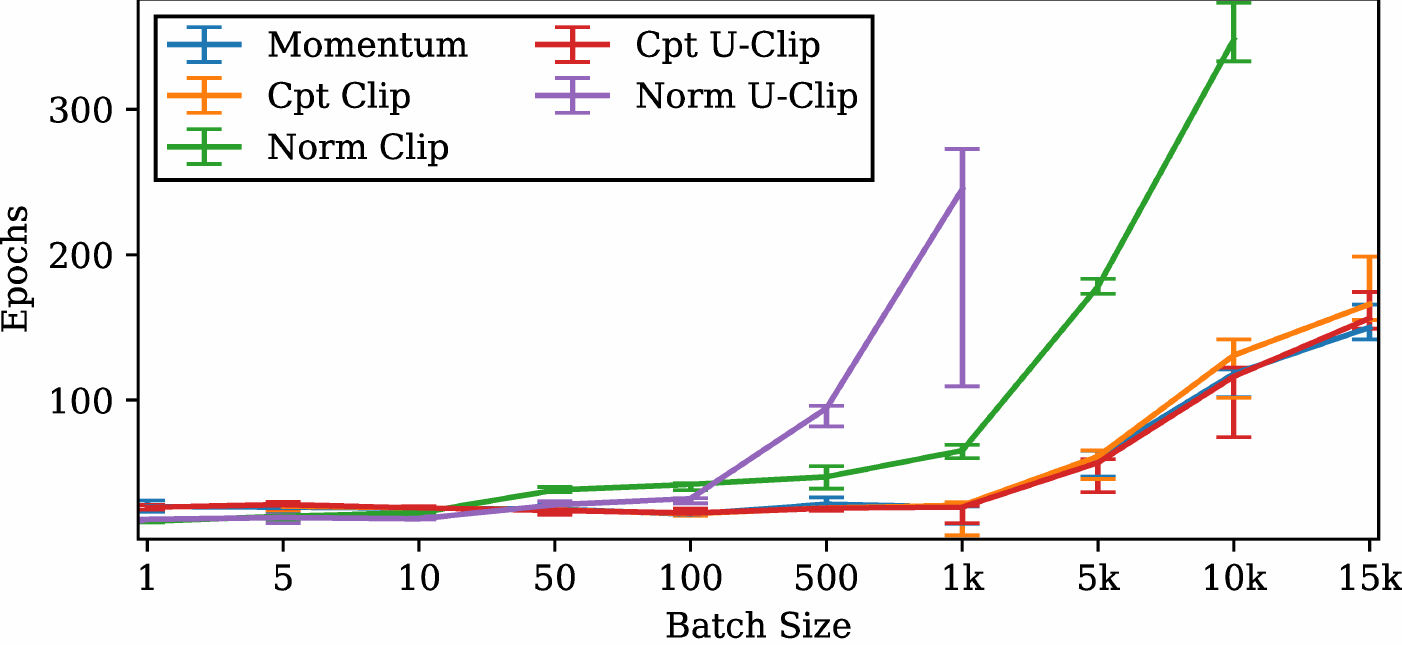}
    \end{subfigure}%
    ~ 
    \begin{subfigure}[t]{\twofigwidth}
        \centering
  \includegraphics[width=\textwidth]{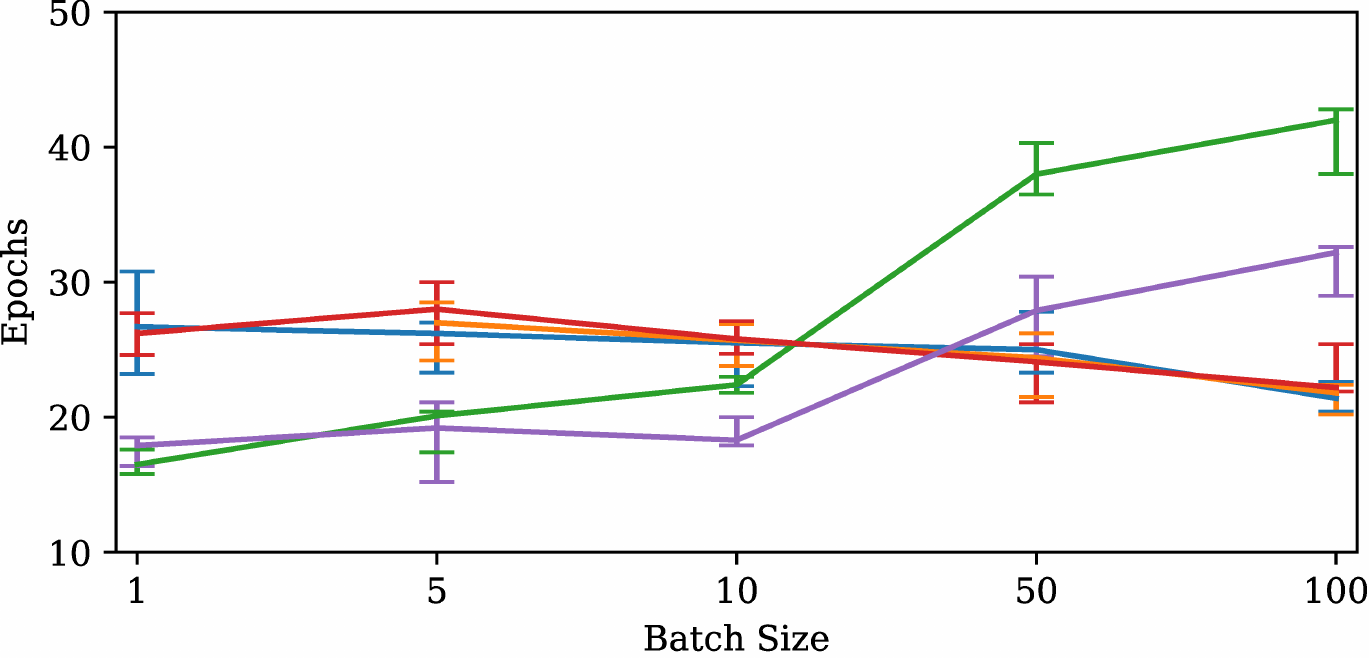}
    \end{subfigure}

    \begin{subfigure}[t]{\twofigwidth}
        \centering
  \includegraphics[width=\textwidth]{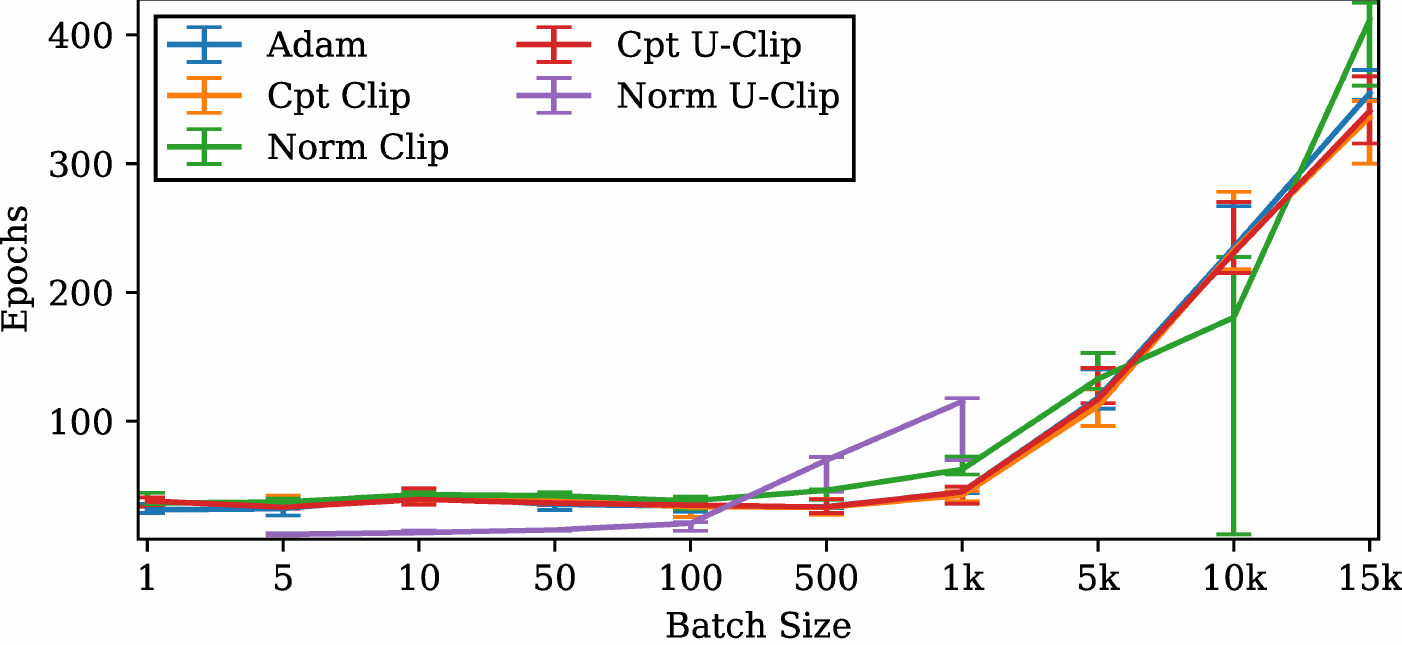}
    \end{subfigure}%
    ~ 
    \begin{subfigure}[t]{\twofigwidth}
        \centering
  \includegraphics[width=\textwidth]{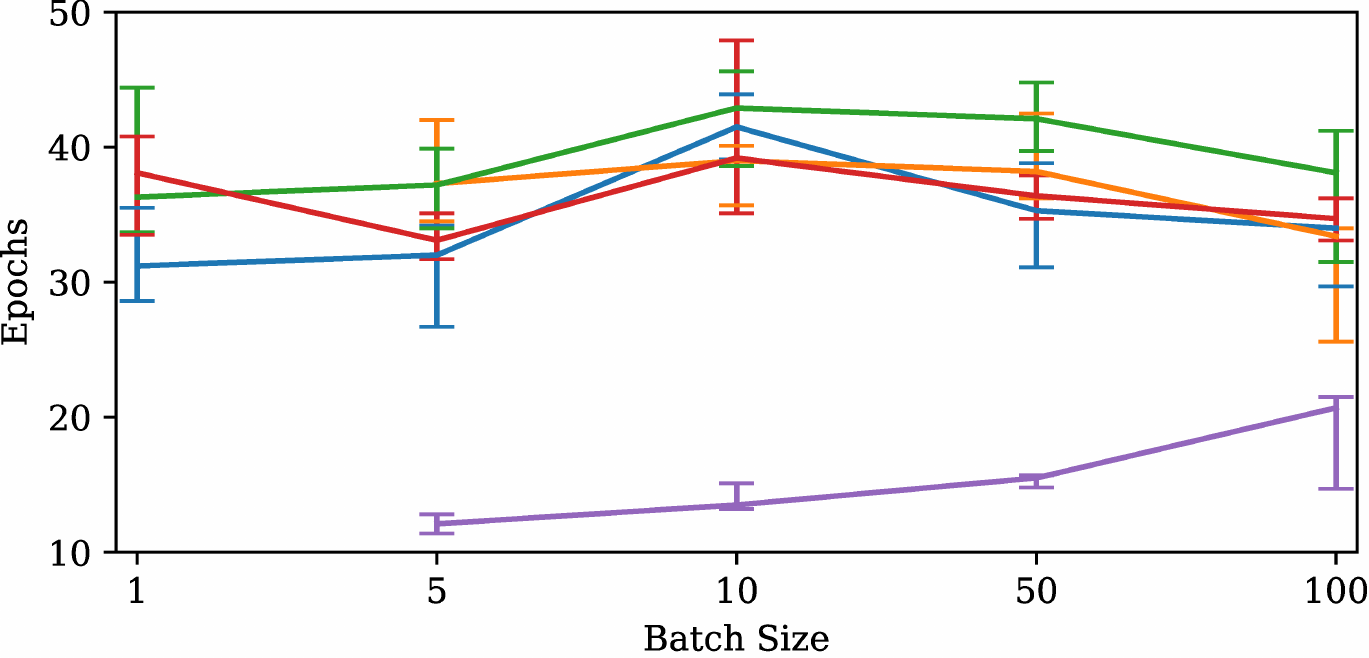}
    \end{subfigure}

    \caption{
      CIFAR10 results for norm clipping with $\gamma=0.5$ constant for
      each base optimizer. From top to bottom: SGD, momentum, Adam.
      Left: full results for each base optimizer.
      Right: zoomed in for small batches, legend the same.
      Errorbars denote min/max over the 5 seeds, line is the median.
      Where a datapoint is missing, some runs did not arrive at 99\%
      training accuracy for that batch size.
      We see a improvement in optimization performance for small batches with both norm
      clipping and norm \uclip{} for SGD and momentum, and a significant improvement for 
      norm \uclip{} with Adam. We believe that this benefit is due to a smoothing out of 
      gradient noise due to the memory of the carry.
    }
    \label{fig:cifar-norm-0.5-results}
\end{figure*}

%-------------------------------%
\subsubsection{ImageNet}\label{sec:imagenet}
%-------------------------------%
We validate the performance of \uclip{} on a larger scale experiment,
training a ResNet50v2~\citep{he2016identity} on ImageNet with Nesterov momentum.
We modify a base setup by adding clipping and \uclip{} for the following clip
regions: constant $\gamma = 0.25$, adaptive $(a, b) = (1, 2)$ and 
proportional adaptive clipping $(a, b) = (10, 0)$.
In each adaptive case we use each of Welford (sample mean/variance) and EWMA
estimation (decay 0.95) as in~\cref{sec:cifar10}.
Further experimental details are given in~\cref{sec:imagenet-exp-details}

We run each configuration for 5 seeds and report 
the median and min/max range over the seeds of the best top 1 accuracy
on a validation set. 
The results are presented in~\cref{tab:imagenet-results}
and a sample optimization trajectory is given in~\cref{fig:imagenet-trajectory}.
We find it encouraging that \uclip{} proves to be competitive with the base setup, 
but it does not provide significant benefit under these conditions,
possibly because the batch size is relatively large. 
We made no attempt to tune the base setup according to the modifications, 
so it is likely that improvements on the results shown are possible.
Time and resource constraints mean we must leave further exploration to
future work.

\begin{table}
  \centering
\begin{tabular}{llllllll}
\toprule
    Median &  Range &  Clip Region & Carry &    $(a, b)$ or $\gamma$ \\
\midrule
 77.2\% & $\pm 0.2$ &    Welford &      True &       (   1 ,    2)  \\
 77.2\% &  $\pm 0.1$ &       EWMA &      True &       (   1 ,    2 ) \\
 77.2\% &  $\pm 0.2$ &  Component &      True &    0.25 \\
 77.1\% &  $\pm 0.2$ &  Component &     False &      0.25 \\
 77.1\% &  $\pm 0.3$ &       EWMA &     False &       (   1 ,    2 ) \\
 77.1\% &  $\pm 0.1$ &No Clipping &         - &       - \\
 77.0\% &  $\pm 0.1$ &       EWMA &     False &       $(  10 ,    0 )$ \\
 76.9\% &  $\pm 0.3$ &    Welford &     False &       $(   1 ,    2 )$ \\
 76.8\% &  $\pm 0.2$ &       EWMA &      True &       $(  10 ,    0 )$ \\
 76.2\% &  $\pm 0.1$ &       Norm &     False &       0.25 \\
 75.6\% &  $\pm 0.3$ &       Norm &      True &       0.25 \\
 72.7\% &  $\pm 0.4$ &    Welford &     False &       $(10 , 0 )$ \\
 61.2\% &  $\pm 1.2$ &    Welford &      True &       $(10 , 0 )$ \\
\bottomrule
\end{tabular}
  \caption{ ResNet50v2 trained on ImageNet for the various clipping and \uclip{} protocols 
    outlined in~\cref{sec:imagenet}.
    The first column is the median top 1 validation accuracy over 5 seeds
    and the second columns is the min/max range. 
    If the Carry column displays True then the row is \uclip{}, if False
    it is clipping (no carry) and otherwise it is the base setup.
    The final column displays the parameters of the clip region,
    $(a, b)$ is for adaptive clipping while the single number $\gamma$
    is a constant clip region.
    We find these results promising, but not conclusive.
    }
  \label{tab:imagenet-results}
\end{table}

\begin{figure}[H]
  \centering
  \iftwocolumn
    \newcommand\thisfigwidth{0.45\textwidth}
    \else
    \newcommand\thisfigwidth{0.75\textwidth}
    \fi
  \includegraphics[width=\thisfigwidth]{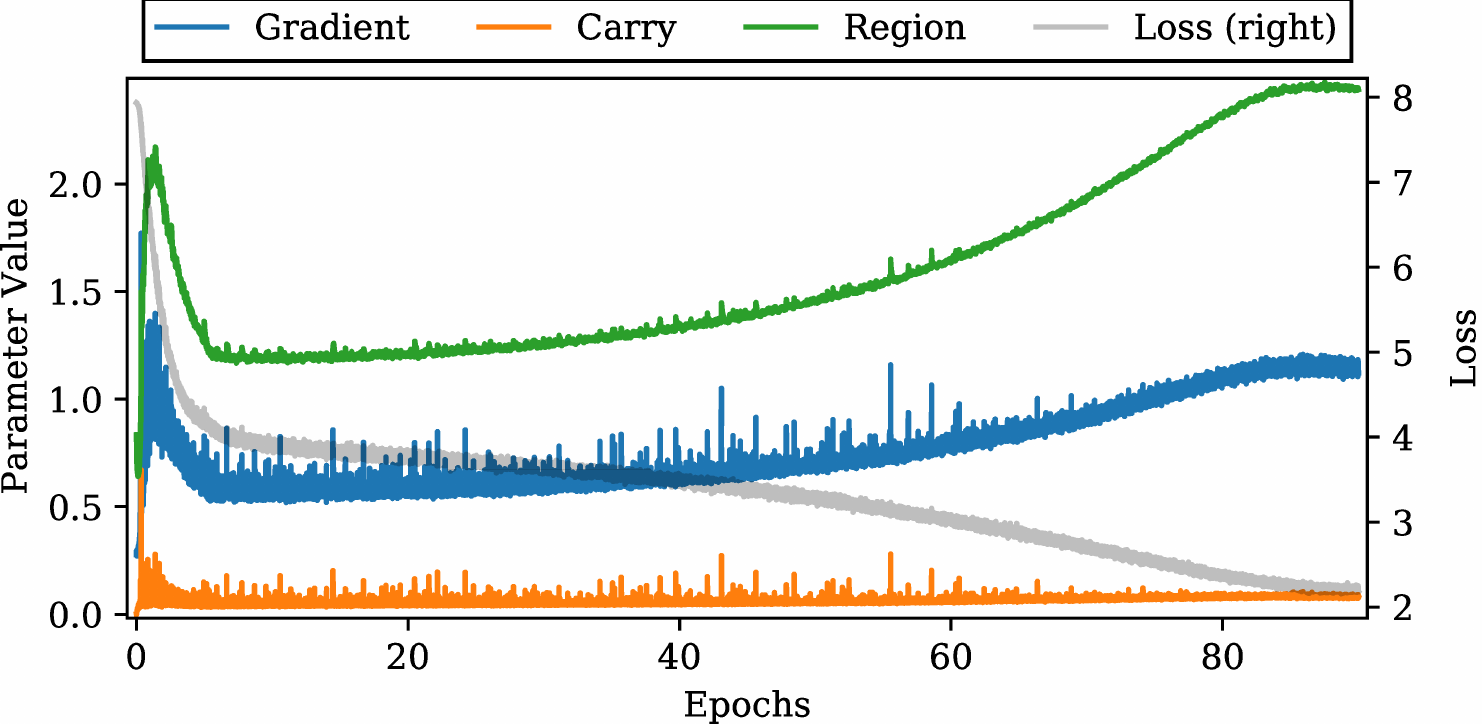}
  \caption{Euclidean norm of gradient, carry buffer and parameter-wise clip region (left)
    and training loss (right) throughout training for ResNet50v2 on ImageNet with U-Clip
    and EWMA clip region with parameters $(a, b) = (1, 2)$.
    See~\cref{sec:imagenet} for details.}
  \label{fig:imagenet-trajectory}
\end{figure}

%%%%%%%%%%%%%%%%%%%%%%%%%%%%%%%%%%%%%%%%%%%%%%%%%%%%%%%%%%%%%%%
\section{Conclusion}\label{sec:conclusion}
%%%%%%%%%%%%%%%%%%%%%%%%%%%%%%%%%%%%%%%%%%%%%%%%%%%%%%%%%%%%%%%

We have presented \uclip{}, a novel method for achieving on-average
unbiased stochastic gradient clipping. 
\uclip{} has desirable theoretical properties and solid empirical validation.
The simplicity of \uclip{} allows for solid theoretical analysis using elementary techniques,
which we see as a key strength.
In addition, our initial investigations find empirical performance that 
is competitive with standard methods, warranting further study.
We hope that the reader is motivated to experiment with and extend \uclip{}.
We give suggestions for future work in~\cref{sec:future-work}.

\ifack
%%%%%%%%%%%%%%%%%%%%%%%%%%%%%%%%%%%%%%%%%%%%%%%%%%%%%%%%%%%%%%%
\section*{Acknowledgements}
%%%%%%%%%%%%%%%%%%%%%%%%%%%%%%%%%%%%%%%%%%%%%%%%%%%%%%%%%%%%%%%
This work was performed while BE was an intern at DeepMind.
We would like to thank, in alphabetical order,
the following people for helpful conversations and support: 
Andr\'{a}s Gy\"{o}rgy, Charline Le Lan,
Claire Vernade,
Denizalp Goktas,
Flore Sentenac,
Ilja Kuzborskij,
Laurent Orseau,
Mihaela Rosca,
Sahra Ghalebikesabi
and Sam Smith,
and especially Ilja for valuable feedback on drafts.
\fi

\ifarxiv
    \printbibliography
    \else
    \bibliography{references}
    \bibliographystyle{icml2023}
\newpage
    \appendix
    \onecolumn
    \twocolumnfalse
\fi
\mainpaperfalse

%%%%%%%%%%%%%%%%%%%%%%%%%%%%%%%%%%%%%%%%%%%%%%%%%%%%%%%%%%%%%%%
\section{Suggestions for Future Work}\label{sec:future-work}
%%%%%%%%%%%%%%%%%%%%%%%%%%%%%%%%%%%%%%%%%%%%%%%%%%%%%%%%%%%%%%%

\begin{itemize} 
    \item Extension of theory to unbounded gradients, possibly with
        heavy-tailed noise which could be particularly relevant in deep
        learning~\citep{simsekli2019tail}.  
    \item Explore the relationship between \uclip{},~\cref{lemma:gradient-transformation}
        and Follow the Regularized Leader with delayed feedback~\citep{joulani2016delay}.
    \item Extension of~\cref{lemma:gradient-transformation} to time varying learning
            rate and tuning the learning rate to get a scale free bound, possibly using
            techniques from~\citet{orseau2021isotuning}.
            In addition, one could adapt~\cref{lemma:gradient-transformation} to
            achieve convergence to a stationary point on non-convex functions.
            Our results could also be generalized to the minibatch setting.
    \item Generalize carry bounds to online (not necessarily stochastic) setting.
        One could use the historical sample mean/variance of gradients 
        as analogues for the statistical quantities
        $\gbar$ and $\sigma^2$, e.g.~$\gbar_t = \frac1t \sum_{i=1}^tg_i$.
    \item Investigate relationship of \uclip{} to Follow the Regularized Leader with
        delayed feedback~\citep{joulani2016delay}.
    \item Tuning of the
            experiments to give optimal performance and experiments on large
            language models, where clipping is particularly relevant.
    \item \uclip{} in alternative application domains, such as reward clipping in
            reinforcement learning.
    \item Exploration, both theoretical and empirical, of alternative
            clipping schemes.  For instance \emph{symmetric clipping}, where
            the gradient is clipped to some interval (e.g.~$\pm 2$ standard
        deviations) centred on an estimate of $\gbar$.
    \item Exploration of \emph{U-Sign}, in which the $\clip$ function is replaced
        $\sign$ element-wise. This gives the update equations (c.f.~\cref{eq:deltarec} and surrounding
        description)
        \begin{align*}
            u_t &= \sign(g_t + \Delta_t) \\
           \Delta_{t+1}&= \Delta_t + g_t - u_t.
        \end{align*}
        While clipping helps with the exploding gradient problem, U-Sign would also
        help prevent vanishing updates. Like \uclip{}, it may be on-average unbiased.
\end{itemize}

%%%%%%%%%%%%%%%%%%%%%%%%%%%%%%%%%%%%%%%%%%%%%%%%%%%%%%%%%%%%%%%
\section{Additional Theoretical Results}
%%%%%%%%%%%%%%%%%%%%%%%%%%%%%%%%%%%%%%%%%%%%%%%%%%%%%%%%%%%%%%%

\subsection{Convergence Result}
Recall that WLOG we are free to take $\gamma_t \ge G$ element-wise.

In~\cref{thm:uclip-convergence}, $\gamma_t$ is a scalar specifying the same
clip region for each coordinate. Per-coordinate clip regions are possible by
applying~\cref{prop:carry-bounded} element-wise in the proof.

\begin{theorem}[\uclip{} convergence: component-wise clipping]\label{thm:uclip-convergence}
  Consider \uclip{} with learning rate $\eta = \frac{1}{\sqrt{T}}$
  and clipping regions $G \ge \gamma_t > \abs{\gbar_t}$ component-wise.
  The expected regret is bounded by
  \begin{align*}
      \E[R_T]
    &\le 
    \frac{8 G^4 d^{3/2}\bgamma}{\alpha^3\sqrt{T}}
    + \frac{4(G + \gamma_+)\sqrt{d}\bgamma + d\tgamma + \norm{x_1 - x_*}^2 }{2\sqrt{T}}
    +  \left(\frac{4dG^4}{\alpha^3} + G + \gamma_+\right)\frac{\norm{x_1 - x_*}}{T}
  \end{align*}
  where $x_* = x_*(T) = \argmin_{x\in\X}\sum_{t=1}^T f_t(x)$,
  $\alpha = \inf_t(\gamma_t - \abs{\gbar_t})$,
    $\bgamma = \frac1T \sum_{t=1}^T \gamma_t$ and
    $\tgamma = \frac1T \sum_{t=1}^T \gamma_t^2$.
\end{theorem}
\begin{proof}[Proof of~\cref{thm:uclip-convergence}]
  We may apply~\cref{prop:carry-bounded} element-wise to $\Delta_t \in \R^d$.
  For any $\epsilon > 0$
  \[
    \P(\norm{\Delta_t} \ge \epsilon + G + \gamma_+ )
    \le \P(\cup_{i=1}^d \{\abs{(\Delta_{t})_i} \ge  \epsilon + G + \gamma_+\})
    \le 4dG^2 \alpha^{-2} e^{-\frac{\epsilon \alpha}{G^2}}
  \]
  by a union bound and the bound on $c_\alpha$ in the statement
  of~\cref{prop:carry-bounded}.
  The same calculation from the proof of~\cref{prop:carry-bounded} then gives
  \[
    \E[\norm{\Delta_t}] 
    \le \frac{4dG^4}{\alpha^3} + G + \gamma_+
    \eqqcolon \Delta_+
  \]
  We can now apply~\cref{lemma:gradient-transformation} with $\eta = \frac{1}{\sqrt{T}}$
  to bound the the expected regret of \uclip{}. Observe that
  $\frac1T \sum_{t=1}^T\Gamma_t \le \sqrt{d}\frac1T \sum_{t=1}^T \gamma_t$ and 
  $\frac1T \sum_{t=1}^T \Gamma_t^2 \le d\frac1T \sum_{t=1}^T \gamma_t^2$.
    We define $\bgamma = \frac1T \sum_{t=1}^T \gamma_t$ and
    $\tgamma = \frac1T \sum_{t=1}^T \gamma_t^2$.
    With this, item~\ref{point:gtlemma-1} 
  of~\cref{lemma:gradient-transformation} becomes
  \[
    \E\left[\frac1T\sum_{t=1}^T (f_t(x_t) - f_t(x_*))\right]
    \le
    \frac{2\Delta_+ \sqrt{d}\bgamma}{\sqrt{T}}
    + \frac{d\tgamma}{2\sqrt{T}}
    + \frac{\norm{x_1 - x_*}^2}{2\sqrt{T}}
    + \frac{\Delta_+}{T}\norm{x_1 - x_*}
  \]
  which, upon substituting in $\Delta_+$ and using the definition of $R_T$, becomes
  \[
    \E[R_T]
    \le
    \frac{8 G^4 d^{3/2}\bgamma}{\alpha^3\sqrt{T}}
    + \frac{4(G + \gamma_+)\sqrt{d}\bgamma + d\tgamma + \norm{x_1 - x_*}^2 }{2\sqrt{T}}
    +  \left(\frac{4dG^4}{\alpha^3} + G + \gamma_+\right)\frac{\norm{x_1 - x_*}}{T}.
  \]
\end{proof}

\subsection{Omitted Proofs}
We restate the results for the reader's convenience.
\subsection{Proof of~\cref{lemma:gradient-transformation}}\label{sec:proof-gradient-transformation}
\begin{lemma*}[{\cref{lemma:gradient-transformation}}]
    \gtlemma{}
\end{lemma*}
\begin{proof}[Proof of~\cref{lemma:gradient-transformation}]
  We start by using the subgradient inequality following from the convexity of
  the $f_t$ and then use the parallelogram rule for inner products
  \begin{align*}
      f_t(x_t) - f_t(x_*) 
      &\le \gbar_t ^\top ( x_t - x_*) \\
      &= 
      g_t ^\top ( x_t - x_*) 
      - \xi_t^\top(x_t - x_*) \\
      &= 
      (g_t - u_t)^\top (x_t - x_*)
      + u_t^\top (x_t - x_*)
      - \xi_t^\top(x_t - x_*) \\
      &=  (g_t - u_t)^\top (x_t - x_*) 
      + \frac{1}{\eta}(x_t - x_{t+1})^\top (x_t - x_*) 
      - \xi_t^\top(x_t - x_*) \\
      &=  (g_t - u_t)^\top (x_t - x_*) 
      + \frac{1}{2\eta}
      \left(
      \norm{x_{t+1} - x_t}^2  + \norm{x_t - x_*}^2 - \norm{x_{t+1} - x_*}^2
      \right)
      - \xi_t^\top(x_t - x_*) \\
      &=  (g_t - u_t)^\top (x_t - x_*) 
      + \frac{1}{2\eta}
      \left( \norm{x_{t} - x_*}^2 - \norm{x_{t+1} - x_*}^2  \right) 
      + \frac{\eta}{2}\norm{u_t}^2
      - \xi_t^\top(x_t - x_*). 
  \end{align*}
  Then, also by convexity and using the assumption on the updates,
  \begin{align*}
    &\frac1T \sum_{t=1}^T (f_t(x_t) - f_t(x_*))\\
    &\le  
    \frac1T \sum_{t=1}^T (g_t - u_t)^\top (x_t - x_*) 
    + \frac{1}{2\eta T}\norm{x_{1} - x_*}^2 
    - \frac{1}{2\eta T}\norm{x_{T+1} - x_*}
    + \frac{\eta}{2T}\sum_{t=1}^T \norm{u_t}^2 
    - \frac1T\sum_{t=1}^T \xi_t^\top(x_t - x_*)\\
    &\le  
    \frac1T \sum_{t=1}^T (g_t - u_t)^\top (x_t - x_*) 
    + \frac{1}{2\eta T}\norm{x_{1} - x_*}^2 
    + \frac{\eta }{2T}\sum_{t=1}^T \Gamma_t^2
    - \frac1T\sum_{t=1}^T \xi_t^\top(x_t - x_*)
    \label{eq:star}\tag{$\star$}
  \end{align*}
  almost surely.
  Using $\Delta_1 = 0$ and the update rule we consider the first term,
  \begin{align*}
    \frac1T \sum_{t=1}^T (g_t - u_t)^\top (x_t - x_*) 
    &= \frac1T \sum_{t=1}^T(\Delta_{t+1} - \Delta_t) ^\top (x_t - x_*) \\
    &= 
    \frac1T \sum_{t=1}^T(\Delta_{t+1} - \Delta_t) ^\top x_t
    - \frac1T \sum_{t=1}^T(\Delta_{t+1} - \Delta_t) ^\top x_*
    \\
    &= \frac1T \sum_{t=1}^T(\Delta_{t+1} - \Delta_t) ^\top x_t
    - \frac1T \Delta_{T+1}^\top x_* \\
    &= \frac1T \sum_{t=1}^T\Delta_{t+1}^\top (x_t - x_{t+1})
    + \frac1T \Delta_{T+1} x_{T+1} 
    - \frac1T \Delta_{T+1}^\top x_* \\
    &= \frac{\eta}{T} \sum_{t=1}^T \Delta_{t+1}^\top u_t
    + \frac1T \Delta_{T+1}^\top (x_{T+1} - x_*) \\
    &= 
    \frac{\eta}{T} \sum_{t=1}^T \Delta_{t+1}^\top u_t
    + \frac{\eta}{T}\sum_{t=1}^T\Delta_{T+1}^\top u_t
    + \frac1T \Delta_{T+1}^\top(x_1 - x_*)\\
    &\le 
    \frac{\eta}{T} \sum_{t=1}^T \norm{\Delta_{t+1}} \Gamma_t
    + \frac{\eta}{T}\sum_{t=1}^T\norm{\Delta_{T+1}} \Gamma_t
    + \frac1T \norm{\Delta_{T+1}}\norm{x_1 - x_*}
  \end{align*}
  almost surely.
  Putting this back into~\cref{eq:star} gives
  \begin{align*}
    \frac1T \sum_{t=1}^T (f_t(x_t) - f_t(x_*))
    &\le  
    \frac{\eta}{T} \sum_{t=1}^T \norm{\Delta_{t+1}} \Gamma_t
    + \frac{\eta}{T}\sum_{t=1}^T\norm{\Delta_{T+1}} \Gamma_t
    + \frac1T \norm{\Delta_{T+1}}\norm{x_1 - x_*}\\
    &\phantom{\le} + \frac{1}{2\eta T}\norm{x_{1} - x_*}^2 
    + \frac{\eta }{2T}\sum_{t=1}^T \Gamma_t^2
    - \frac1T\sum_{t=1}^T \xi_t^\top(x_t - x_*). \label{eq:spade}\tag{$\spadesuit$}
  \end{align*}
  We complete the proof by specializing~\cref{eq:spade} to the two cases in the statement.

  First~\ref{point:gtlemma-1}.
  The final term in~\cref{eq:spade} has mean 0, since 
  by the law of total expectation
  $\E[\xi^\top(x_t - x_*)] = \E[\E[\xi\vert{}x_t]^\top(x_t -x_*)]=0$.
  So taking expectations and using the condition $\E[\norm{\Delta_t}] \le D_t$
  along with $D_t \ge D_{t-1}$ gives the result.

  Now item~\ref{point:gtlemma-2}. With probability at least $1-\delta$,
  $\norm{\Delta_t} \le D_t$. Along with $D_t$ non-decreasing this gives,
  from~\cref{eq:spade}
  \begin{align*}
    \frac1T \sum_{t=1}^T (f_t(x_t) - f_t(x_*))
    &\le  
    \frac{2\eta D_{T+1}}{T} \sum_{t=1}^T \Gamma_t
    + \frac{D_{T+1}}{T}\norm{x_1 - x_*}
    + \frac{1}{2\eta T}\norm{x_{1} - x_*}^2 
    + \frac{\eta }{2T}\sum_{t=1}^T \Gamma_t^2\\
    &\phantom{\le} - \frac1T\sum_{t=1}^T \xi_t^\top(x_t - x_*). 
  \end{align*}
  It remains only to control the final term. We have
  $\E[\sum_{t=1}^T \xi_t^\top (x_t - x_*)\vert{} x_1, \dots, x_T] = 0$
  and that the quantity $z_T = \frac1T \sum_{t=1}^T \xi_t^\top (x_t - x_*)$ is
  sub-Gaussian with variance proxy $\frac{\sigma^2R^2}{T}$: for any $a$
  \begin{align*}
    &\E\left[\exp\left(a \frac1T \sum_{t=1}^T \xi_t^\top (x_t - x_*)\right)\right]\\
    &= \E\left[\prod_{t=1}^T\exp\left(a \frac1T \xi_t^\top (x_t - x_*)\right)\right] \\ 
    &= \E\left[
      \E\left[\prod_{t=1}^T\exp\left(a \frac1T \xi_t^\top (x_t - x_*)\right)
    \midvert x_1, \dots, x_T\right] 
    \right] \quad \text{tower property}\\ 
    &= \E\left[
    \prod_{t=1}^T
      \E\left[\exp\left(a \frac1T \xi_t^\top (x_t - x_*)\right)
    \midvert x_1, \dots, x_T\right] 
    \right] \quad \text{conditional independence}\\ 
    &\le \E\left[
    \prod_{t=1}^T
      \E\left[
        \exp\left(\frac{\sigma^2 a^2}{2T^2} \norm{x_t - x_*}^2\right)
    \midvert x_1, \dots, x_T\right] 
    \right] \quad \text{conditionally sub-Gaussian}\\ 
    &\le \E\left[
    \prod_{t=1}^T
      \E\left[
        \exp\left(\frac{\sigma^2 a^2}{2T^2} R^2 \right)
    \midvert x_1, \dots, x_T\right] 
  \right] \quad \text{bounded iterates}\\ 
  &= \exp\left(\frac{a^2 \sigma^2 R^2}{2T}\right).
  \end{align*}
  The standard tail bound (Hoeffding) then gives
  \[
    \P(\abs{z_T} \le \epsilon) \le 2 e^{-\frac{\epsilon^2 T}{2R^2 \sigma^2}}.
  \]
  The result is immediate from a simple calculation and a union bound.
\end{proof}

\subsection{Proof of~\cref{prop:carry-bounded}}\label{sec:proof-carry-bounded}
\begin{proposition*}[{\cref{prop:carry-bounded}}]
    \propcarrybounded{}
\end{proposition*}
\begin{proof}[Proof of~\cref{prop:carry-bounded}]
  For any time $t$ and for $i\in [0, t-2]$, we define the events
  \[
    K_i = \left\{
      \bigwedge\limits_{j=i+1}^{t -1} (\Delta_{j} > G + \gamma_j)
    \right\}
  \cap \{\Delta_i \le G + \gamma_i \}
  \]
  with $K_{t-1} = \{\Delta_{t-1} \le G + \gamma_{t-1} \}$.
  Note that $\{\Delta_t \ge \epsilon + G + \gamma_+\} \subset \cup_{i=0}^{t-1} K_i$.
  The update rule is
  \[
    \Delta_{i+1} = \Delta_i + g_i - \clip(\Delta_i + g_i, \gamma_i).
  \]
  If $K_i$ happens then $\Delta_{i+1} > G + \gamma_{i+1} > 0$.
  There is now only one possibility for $u_i$.
  In particular, if $\clip(\Delta_i + g_i, \gamma_i) = \Delta_i + g_i$ then 
  $\Delta_{i+1} = 0$, and similarly if $u_i = -\gamma_i$
  then $\Delta_i + g_i \le -\gamma_i$ so $\Delta_{i+1} < 0 $, hence
  we must instead have
  $
    \Delta_{i+1} = \Delta_i + g_i - \gamma_i
    $.
  Using this and the same idea iteratively, then using the expression for the gradients,
  it follows that $K_i$ implies
  \[
    \Delta_{t} = 
    \Delta_i - \sum_{j=i}^{t-1}(\gamma_j - \gbar_j) + \sum_{j=i}^{t-1} \xi_j 
  \]
  So
  \begin{align*}
    \P(\Delta_{t} \ge \epsilon + G + \gamma_+\cap{} K_i)
    &= \P\left(
    \Delta_i - \sum_{j=i}^{t-1}(\gamma_j - \gbar_j) + \sum_{j=i}^{t-1} \xi_j 
    \ge \epsilon + G + \gamma_+
    \;\bigcap\; 
    K_i\right) \\
    &\le 
    \P\left(\sum_{j=i}^{t-1} \xi_j \ge \epsilon + 
    \sum_{j=i}^{t-1}(\gamma_j - \gbar_j) 
    \;\bigcap\; 
    K_i\right) \\
    &\le 
    \P\left(\sum_{j=i}^{t-1} \xi_j \ge \epsilon + 
    \sum_{j=i}^{t-1}(\gamma_j - \gbar_j) 
    \right) \\
    &\le \P\left(\sum_{j=i}^{t-1} \xi_j \ge \epsilon + (t - i)\alpha \right).
  \end{align*}
  By conditioning on the history $x_1, \dots, x_T$ we may exploit
  independence and use Hoeffding's inequality with $\abs{\xi_i} \le G$ to
  get 
  \begin{align*}
    \P\left(\sum_{j=i}^{t-1} \xi_j \ge \epsilon + (t - i)\alpha \right)
    &= \E\left[
      \P\left(\sum_{j=i}^{t-1} \xi_j \ge \epsilon + (t - i)\alpha 
      \midvert{}
      x_1, \dots, x_T
        \right)
    \right]\\
    &\le \E\left[
    \exp\left(-\frac{2(\epsilon + (t-i)\alpha)^2}{4G^2(t-i)}\right)
    \right]\\
  &= \exp\left(-\frac{2(\epsilon + (t-i)\alpha)^2}{4G^2(t-i)}\right).
  \end{align*}
  Then
  \begin{align*}
    \P(\Delta_t \ge \epsilon + G + \gamma_+)
    &= \sum_{i=0}^{t-1} 
    \P(\Delta_t \ge \epsilon + G + \gamma_+ \cap{} K_i)
    \tag{$\dagger$}\label{eq:truncate}\\
    &\le \sum_{i=0}^{t-1} 
    \exp\left(-\frac{(\epsilon + (t-i)\alpha)^2}{2G^2(t-i)}\right)\\
    &= \sum_{j=1}^{t} 
    \exp\left(-\frac{(\epsilon + j\alpha)^2}{2G^2j}\right)\\
    &\le 
    \exp\left(-\frac{\epsilon^2 + 2\alpha\epsilon t}{2G^2t}\right)
    \sum_{j=1}^{t} \exp\left(-\frac{j\alpha^2}{2G^2}\right)\\
    &\le 
    \exp\left(-\frac{\epsilon^2 + 2\alpha\epsilon t}{2G^2t}\right)
    \frac{\exp\left(-\frac{\alpha^2}{2G^2}\right)}{
      1 - \exp\left(-\frac{\alpha^2}{2G^2}\right) } 
  \end{align*}
  Using the same method, one can calculate that
  \[
    \P(\Delta_t \le - \epsilon - G - \gamma_+)
    \le 
    \exp\left(-\frac{\epsilon^2 + 2\alpha\epsilon t}{2G^2t}\right)
    \frac{\exp\left(-\frac{\alpha^2}{2G^2}\right)}{
      1 - \exp\left(-\frac{\alpha^2}{2G^2}\right) } 
  \]
  and the result follows from a union bound.

  To calculate the expectation, we will use the weaker form
  \[
    \P(\abs{\Delta_t} \ge \epsilon + G + \gamma_+) 
    \le 
    2c_\alpha
    \exp\left(-\frac{\alpha\epsilon}{G^2}\right).
  \]
  Then
  \begin{align*}
    \E[\abs{\Delta_t}]
    &= \int_0^\infty \P(\abs{\Delta_t} > \epsilon)\dd \epsilon\\
    &= \int_0^\infty \P(\abs{\Delta_t} > \epsilon + G + \gamma )\dd \epsilon
    + \int_0^{G+\gamma} \P(\abs{\Delta_t} > \epsilon  )\dd \epsilon\\
    &\le  \frac{2c_\alpha G^2}{\alpha} + G + \gamma_+.
  \end{align*}
\end{proof}

\subsection{Proof of~\cref{prop:adaptive-clipping}}\label{sec:proof-adaptive-clipping}
\begin{proposition*}[{\cref{prop:adaptive-clipping}}]
    \propadaptiveclipping{}
\end{proposition*}

\begin{proof}[Proof of~\cref{prop:adaptive-clipping}]
  Define $\gamma_+ = \sup_i \gamma_i \le (1 + \beta)G$.
  For any time $t$ and for $i\in [0, t-2]$, we define the events
  \[
    K_i = \left\{
      \bigwedge\limits_{j=i+1}^{t -1} (\Delta_{j} > G + \gamma_j)
    \right\}
  \cap \{\Delta_i \le G + \gamma_i \}
  \]
  with $K_{t-1} = \{\Delta_{t-1} \le G + \gamma_{t-1} \}$.
  Note that $\{\Delta_t \ge \epsilon + G + \gamma_+\} \subset \cup_{i=0}^{t-1} K_i$.
  The update rule is
  \[
    \Delta_{i+1} = \Delta_i + g_i - \clip(\Delta_i + g_i, \gamma_i).
  \]
  If $K_i$ happens then $\Delta_{i+1} > G + \gamma_{i+1} > 0$.
  There is now only one possibility for $u_i$.
  In particular, if $\clip(\Delta_i + g_i, \gamma_i) = \Delta_i + g_i$ then 
  $\Delta_{i+1} = 0$, and similarly if $u_i = -\gamma_i$
  then $\Delta_i + g_i \le -\gamma_i$ so $\Delta_{i+1} < 0 $, hence
  we must instead have
  $
    \Delta_{i+1} = \Delta_i + g_i - \gamma_i
    $.
  Using this and the same idea iteratively, then using the expression for the gradients,
  it follows that $K_i$ implies
  \[
    \Delta_{t} = 
    \Delta_i - \sum_{j=i}^{t-1}(\gamma_j - \gbar_j) + \sum_{j=i}^{t-1} \xi_j 
  \]
  So
  \begin{align*}
    \P(\Delta_{t} \ge \epsilon + G + \gamma_+\cap{} K_i)
    &= \P\left(
    \Delta_i - \sum_{j=i}^{t-1}(\gamma_j - \gbar_j) + \sum_{j=i}^{t-1} \xi_j 
    \ge \epsilon + G + \gamma_+
    \;\bigcap\; 
    K_i\right) \\
    &\le 
    \P\left(\sum_{j=i}^{t-1} \xi_j \ge \epsilon + 
    \sum_{j=i}^{t-1}(\gamma_j - \gbar_j) 
    \;\bigcap\; 
    K_i\right) \\
    &\le 
    \P\left(\sum_{j=i}^{t-1} \xi_j \ge \epsilon + 
    \sum_{j=i}^{t-1}(\gamma_j - \gbar_j) 
    \right) \\
    &\le 
    \P\left(
    \sum_{j=i}^{t-1} \xi_j \ge \epsilon + 
    \beta\sum_{j=i}^{t-1}\sigma_j^2
    \right).
  \end{align*}
  Assume temporarily that $\sum_{j=i}^{t-1}\sigma_j^2 > 0$, then
  by assumption on $\xi_i \vert{} x_i$ and
  Hoeffding's inequality we have 
  \[
    \P\left(\sum_{j=i}^{t-1} \xi_j \ge \epsilon + 
    \beta\sum_{j=i}^{t-1}\sigma_j^2
    \midvert{}
    x_i, \dots, x_{t-1}
    \right)
    \le \exp\left(-\frac{(\epsilon + 
    \beta\sum_{j=i}^{t-1}\sigma_j^2)^2
    }{2 \sum_{j=i}^{t-1}\sigma_j^2
    }\right)
    \le \exp\left(- \epsilon\beta - \frac12(t-i)\beta^2 \sigmamin^2\right)
  \]
  Where $\sigmamin^2 = \inf_i \sigma_i^2$.
  Now note that if $\sum_{j=i}^{t-1}\sigma_j^2 = 0$ then
  $\sum_{j=i}^{t-1} \xi_j =0 $ almost surely and the 
  above result remains true.
  Using $\P(A) = \E[\1\{A\}] = \E[\E[\1\{A\} \vert{} \F]] = \E[\P(A \vert{} \F)]$ 
  for any event $A$ and $\sigma$-field
  $\F$ we integrate out $x_i, \dots, x_{t-1}$ and get
  \begin{align*}
    \P(\Delta_t \ge \epsilon + G + \gamma_+)
    &= \sum_{i=1}^{t-1} 
    \P(\Delta_t \ge \epsilon + G + \gamma_+ \cap{} K_i)\\
    &\le \exp(-\epsilon\beta) \sum_{i=1}^{t-1}\exp( -(t-i)\beta^2 \sigmamin^2/2)\\
    &\le \exp(-\epsilon\beta) \sum_{j=1}^{t-1}\exp( -j\beta^2 \sigmamin^2/2)\\
    &\le \exp(-\epsilon\beta) 
    \min\left\{t, \frac{1}{1 - \exp(\beta^2\sigmamin^2/2)} \right\}
    \\
    &\le \exp(-\epsilon\beta) \min\{t, \beta^{-2}\sigmamin^{-2}/2\} 
  \end{align*}
  Using the same method one can calculate that
  $
    \P(\Delta_t \le - \epsilon - G - \gamma_+)
    \le \exp(-\epsilon\beta) \min\{t, \beta^{-2}\sigmamin^{-2}/2\}
  $.
  Then $\gamma_+ \le (1+\beta)G$ and a union bound gives
  \[
    \P(\abs{\Delta_t} \ge  \epsilon + (2+\beta)G)
    \le
    \P(\abs{\Delta_t} \le  \epsilon + G + \gamma_+)
    \le 2\min\{t, \beta^{-2}\sigmamin^{-2}/2\} \exp(-\epsilon\beta) 
  \]
  from which the first result is immediate. 
  Doing the same integration
  at end of the proof of~\cref{prop:carry-bounded} and using 
  $\min\{t, \beta^{-2}\sigmamin^{-2}/2\}  \le \beta^{-2}\sigmamin^{-2}/2$
  gives 
  \[
    \E[\abs{\Delta_t}] \le \frac{8}{\beta^3\sigmamin^2} + (2+\beta)G.
  \]
\end{proof}

%%%%%%%%%%%%%%%%%%%%%%%%%%%%%%%%%%%%%%%%%%%%%%%%%%%%%%%%%%%%%%%
\section{Additional Experimental Details}
%%%%%%%%%%%%%%%%%%%%%%%%%%%%%%%%%%%%%%%%%%%%%%%%%%%%%%%%%%%%%%%

\subsection{CIFAR10}\label{sec:cifar-exp-details}
The network has two convolutional layers with 32 and 64 features respectively,
initialized with the LeCun normal initialization, also known as the truncated normal 
initialization~\citep{klambauer2017self}, with initial bias 0.
Each convolutional later has a $3\times3$ kernel and followed by average pooling over
a $2\times2$ window with $2\times2$ strides.
The convolution layers are follows my fully connected layers with features
of width $250, 250, 10$, initialized with Gaussian Glorot 
initialization~\citep{glorot2010understanding} and initial bias of $0.01$.
All of the activations are ReLU.

Each experiment is run for 5 random seeds on a V100 GPU, with learning rates
$\enum{1}{-1}$, $\enum{1}{-2}$, $\enum{1}{-3}$,$\enum{1}{-4}$,
$\enum{1}{-5}$ and on
on batch sizes 1, 5, 10, 50, 100, 500, 1000, 5000, 10000, 15000
for epochs 100, 300, 300, 300, 300, 500, 500, 500, 500, 500 per batch size
(to give enough steps to converge).
An experiment consists of training the model using one of the base optimizers
either alone, with clipping (no carry) or with \uclip{}.
In the two cases where clipping is used, we consider 
component and norm clipping with constant clip region (over time and parameters)
$\gamma = .25, .5, 1, 5, 10$, and adaptive clipping with parameters
$a, b = (2, 0), (10, 0), (1, 1), (1, 2), (1, 3)$ with each of Welford and
EWMA estimation.

\subsection{ImageNet}\label{sec:imagenet-exp-details}
The base setup is the Haiku implementation of ResNet50v2 with batch norm (decay rate 0.9)
on ImageNet. The network is trained to minimize
cross-entropy loss with weight decay ($\enum{1}{-4}$) and label smoothing ($0.1$) 
with Nesterov momentum (decay 0.9) for 90 epochs and a cosine schedule
learning rate with initial learning rate 0.1 and 5 warm-up epochs.
We train on on a 4x4 TPU pod.
Batch size is 128 per device, so 4096 in total.

%%%%%%%%%%%%%%%%%%%%%%%%%%%%%%%%%%%%%%%%%%%%%%%%%%%%%%%%%%%%%%%
\section{Parallelization}
%%%%%%%%%%%%%%%%%%%%%%%%%%%%%%%%%%%%%%%%%%%%%%%%%%%%%%%%%%%%%%%

In reference to single program, multiple data parallelization of \uclip{} there 
are two options.\footnote{
In the single program, multiple data paradigm, the model is copied across multiple
devices, gradients are computed on each device using different training data
and the gradients are aggregated on a separate device. This allows for a batch size
$\texttt{per\_device\_batch\_size} \times \texttt{num\_devices}$, sidestepping memory
constraints on a single device.}
One could clip the gradients on each device individually then aggregate,
or aggregate the gradients then clip the result.

The benefit of the former is that effect of any spurious/outsized gradients on
any device on the aggregate is limited. For instance, maybe only one example on
one device is causing issues. In this case, only the
examples on the offending device would be affected.
One could perform multiple rounds of gradient computations in this way before updating
the weights, increasing the overall batch size further.

On the other hand, the latter has the benefit that noise/large values can
cancel out ``naturally'', maybe obviating the need to clip at all on that iteration.
One might also prefer this approach since it interferes with the gradient statistics
at the latest possible moment.

%%%%%%%%%%%%%%%%%%%%%%%%%%%%%%%%%%%%%%%%%%%%%%%%%%%%%%%%%%%%%%%
\section{Additional Experimental Results: CIFAR10}\label{sec:cifar-additional-results}
%%%%%%%%%%%%%%%%%%%%%%%%%%%%%%%%%%%%%%%%%%%%%%%%%%%%%%%%%%%%%%%

\begin{figure}[htpb]
    
    \centering

    \begin{subfigure}[t]{\threefigwidth}
        \centering
  \includegraphics[width=\textwidth]{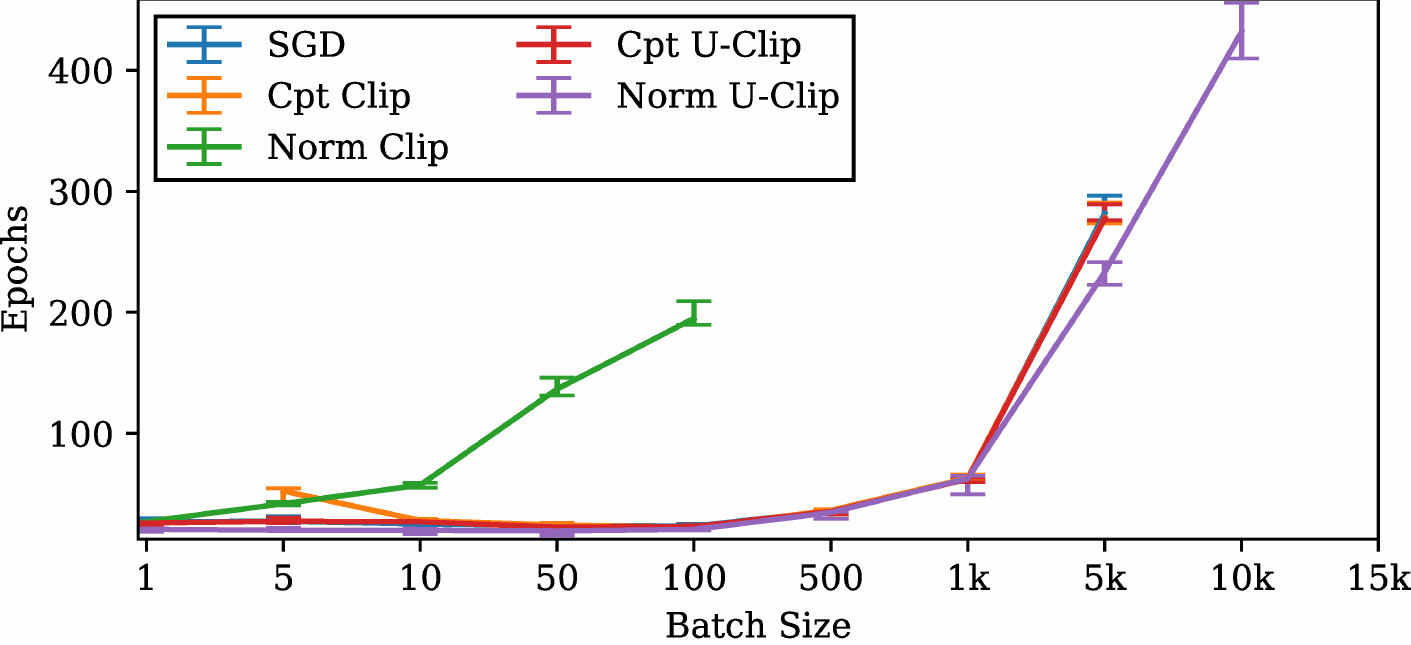}
    \end{subfigure}%
    ~ 
    \begin{subfigure}[t]{\threefigwidth}
        \centering
  \includegraphics[width=\textwidth]{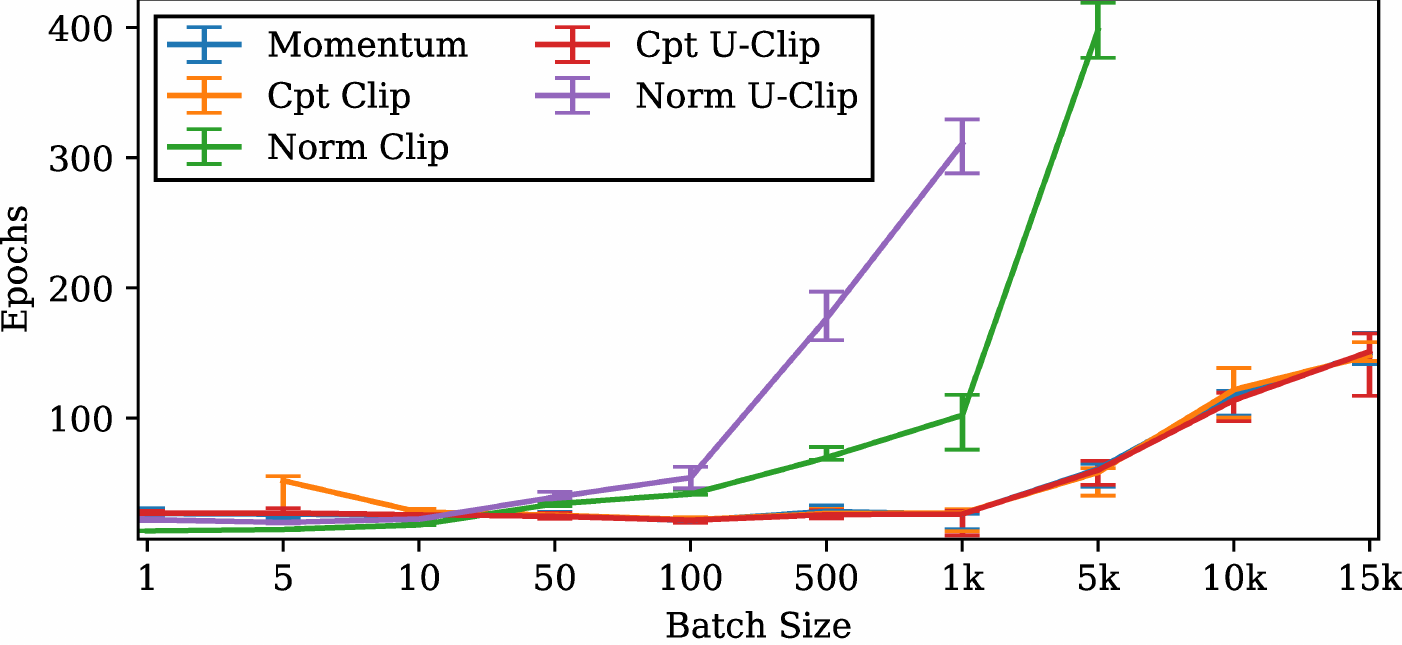}
    \end{subfigure}%
    ~ 
    \begin{subfigure}[t]{\threefigwidth}
        \centering
  \includegraphics[width=\textwidth]{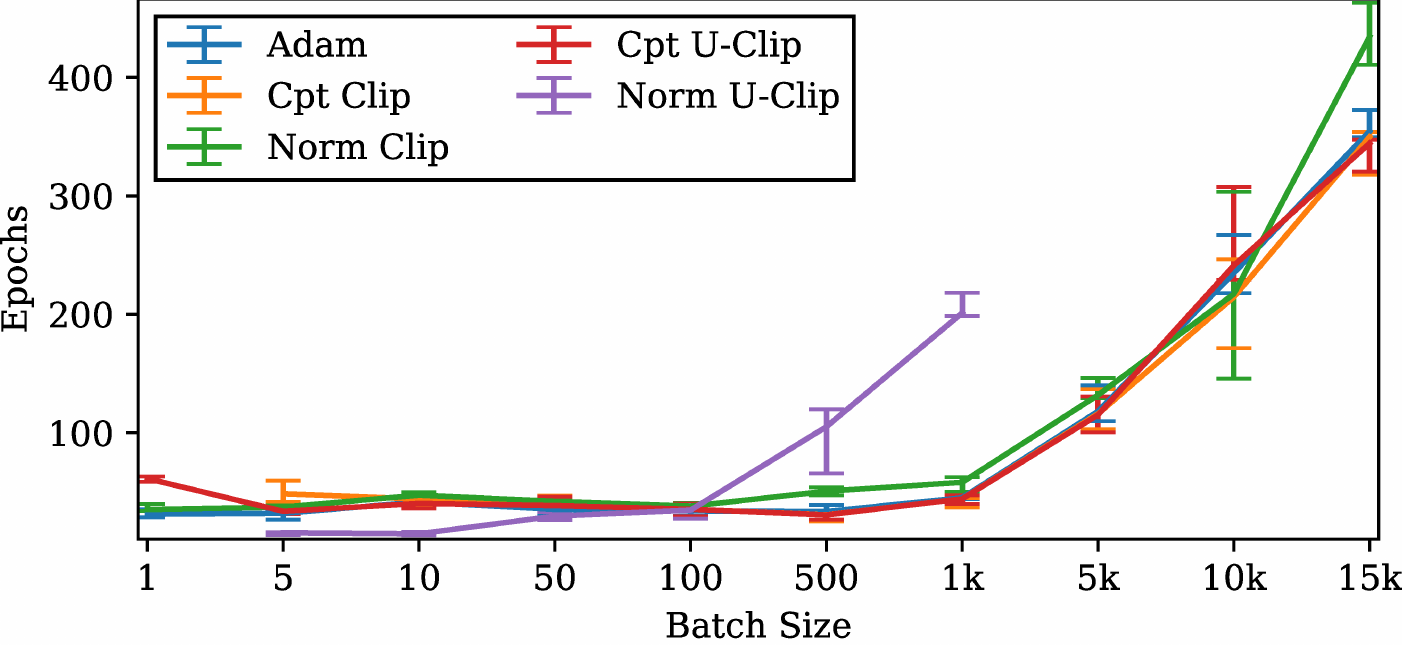}
    \end{subfigure}

    \begin{subfigure}[t]{\threefigwidth}
        \centering
  \includegraphics[width=\textwidth]{images/cifar10-epochs-to-99-sgd-constant-0.5--shortlabel.pdf}
    \end{subfigure}%
    ~ 
    \begin{subfigure}[t]{\threefigwidth}
        \centering
  \includegraphics[width=\textwidth]{images/cifar10-epochs-to-99-momentum-constant-0.5--shortlabel.pdf}
    \end{subfigure}%
    ~ 
    \begin{subfigure}[t]{\threefigwidth}
        \centering
  \includegraphics[width=\textwidth]{images/cifar10-epochs-to-99-adam-constant-0.5--shortlabel.pdf}
    \end{subfigure}

    \begin{subfigure}[t]{\threefigwidth}
        \centering
  \includegraphics[width=\textwidth]{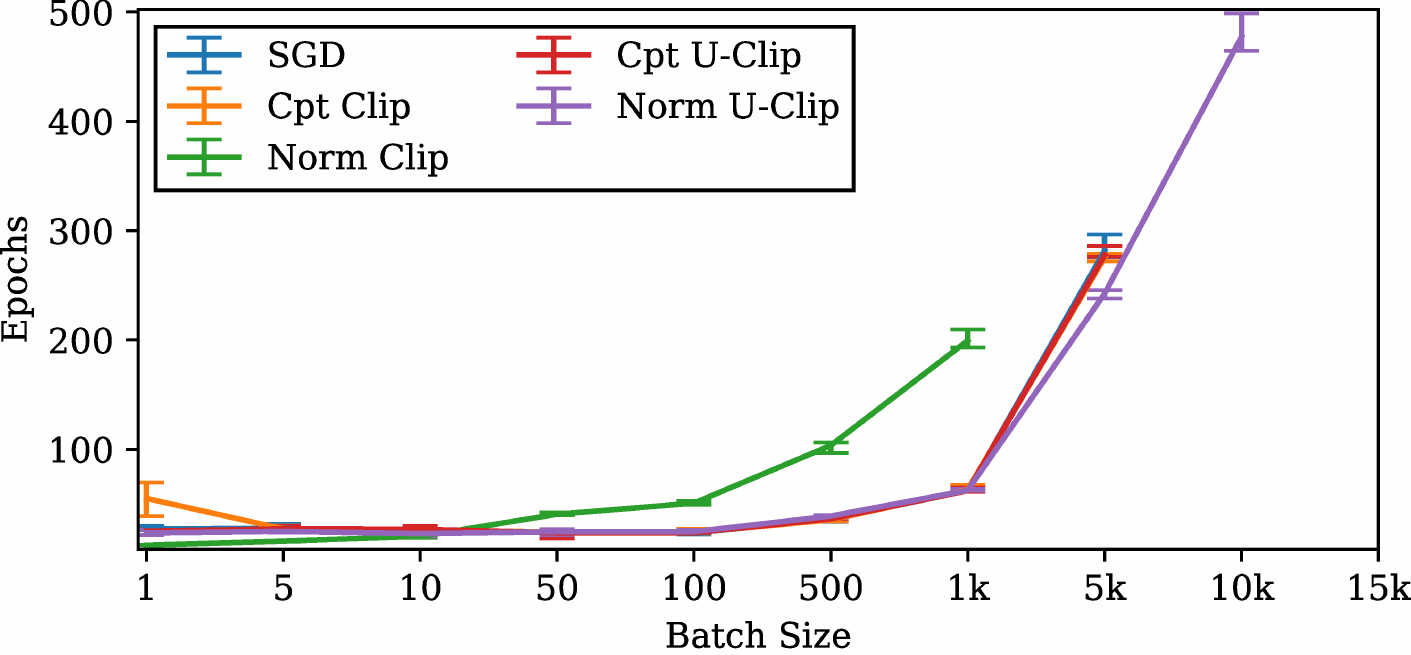}
    \end{subfigure}%
    ~ 
    \begin{subfigure}[t]{\threefigwidth}
        \centering
  \includegraphics[width=\textwidth]{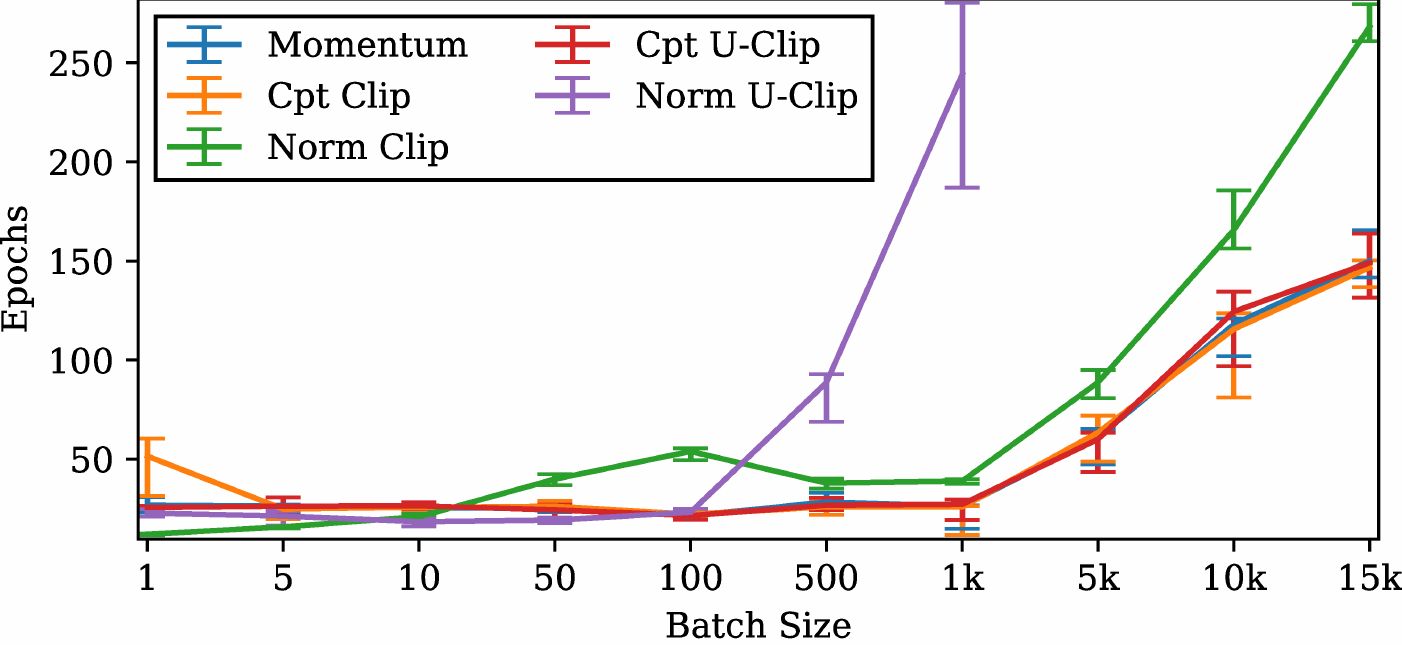}
    \end{subfigure}%
    ~ 
    \begin{subfigure}[t]{\threefigwidth}
        \centering
  \includegraphics[width=\textwidth]{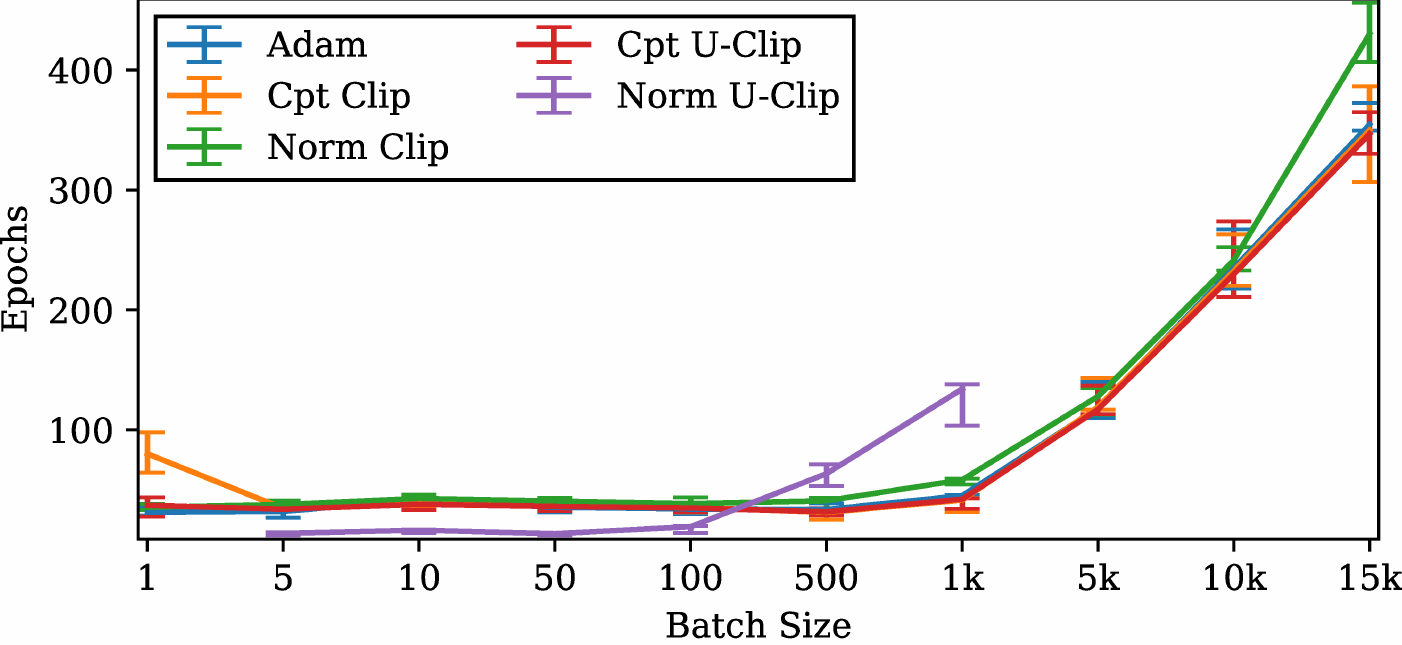}
    \end{subfigure}

    \begin{subfigure}[t]{\threefigwidth}
        \centering
  \includegraphics[width=\textwidth]{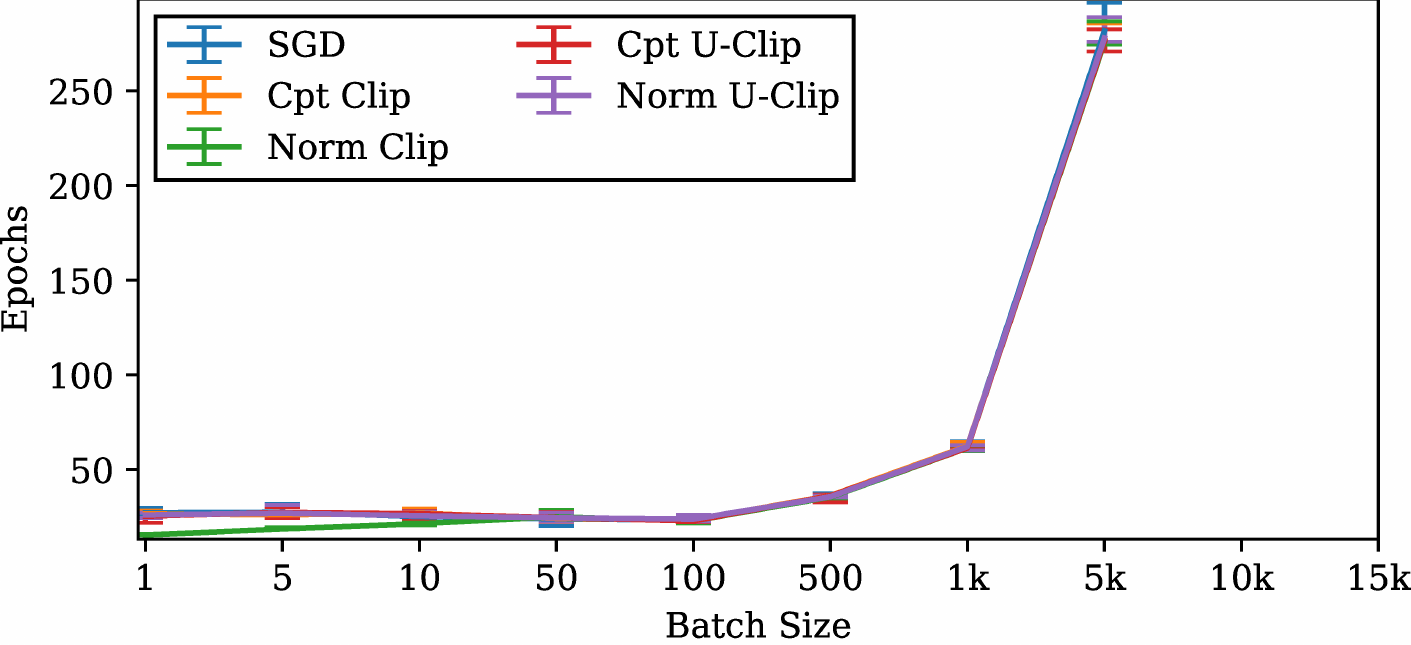}
    \end{subfigure}%
    ~ 
    \begin{subfigure}[t]{\threefigwidth}
        \centering
  \includegraphics[width=\textwidth]{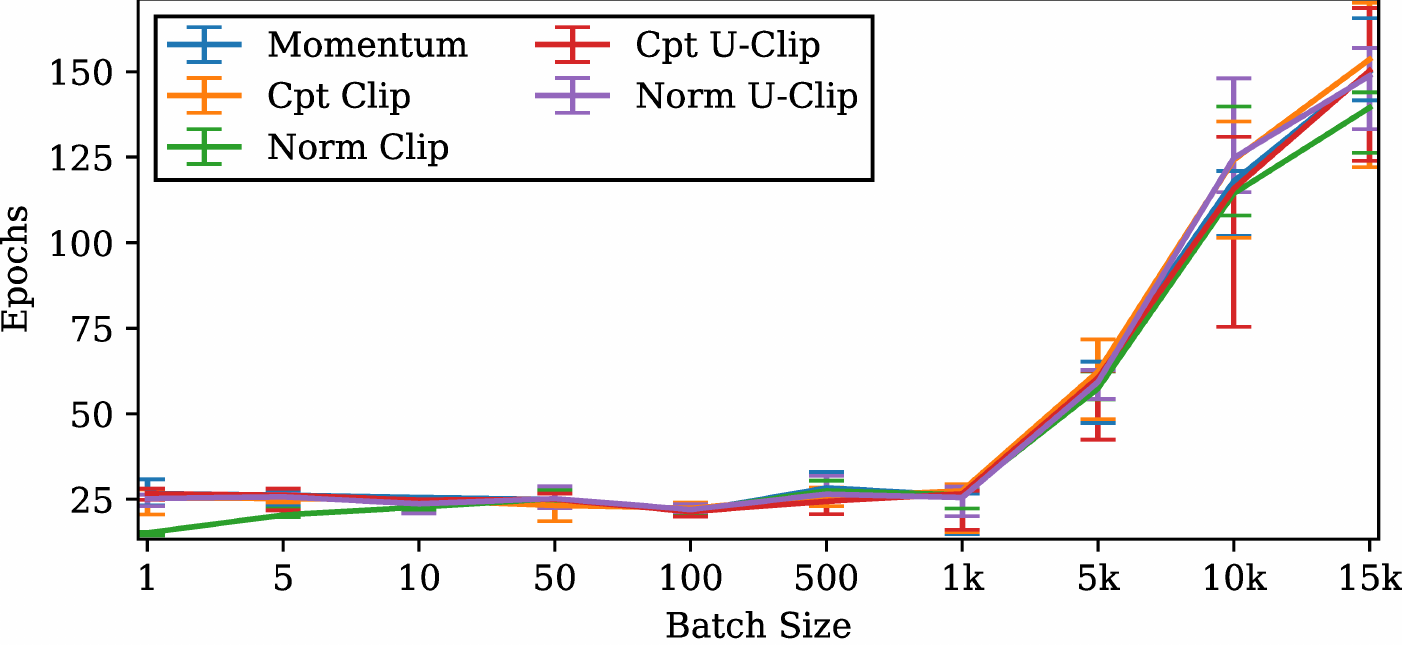}
    \end{subfigure}%
    ~ 
    \begin{subfigure}[t]{\threefigwidth}
        \centering
  \includegraphics[width=\textwidth]{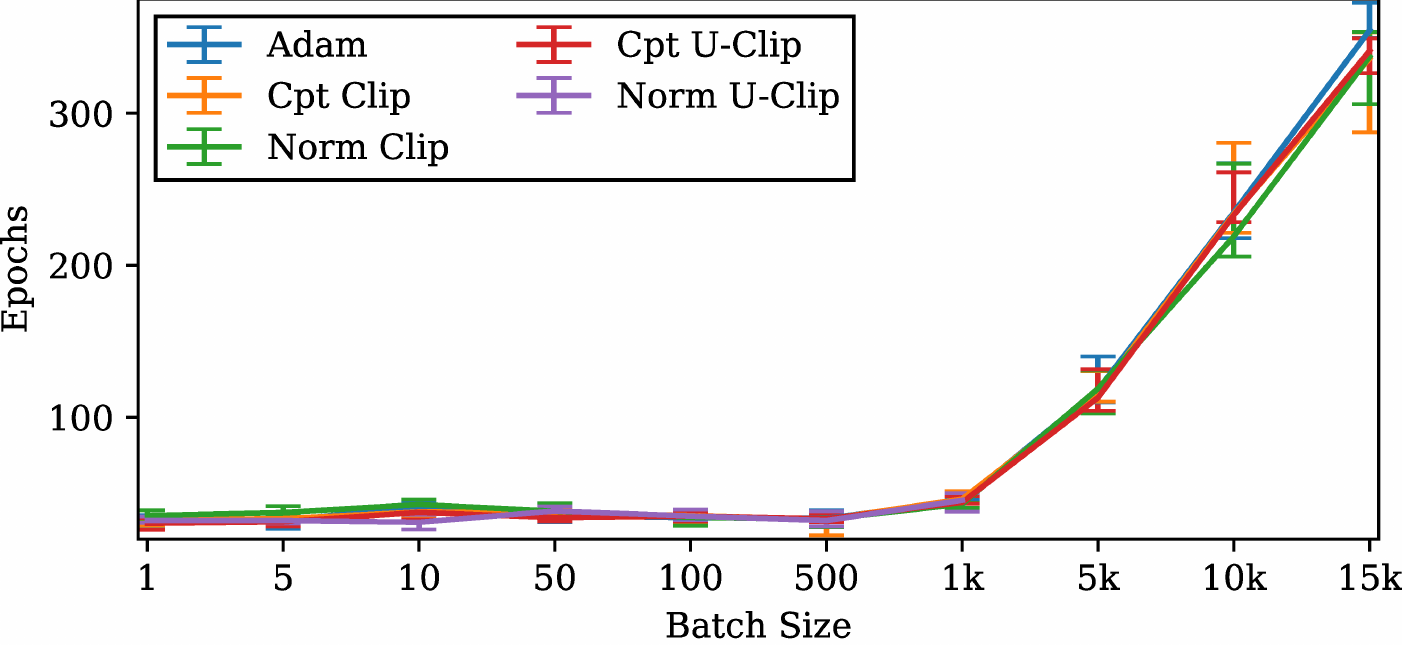}
    \end{subfigure}

    \begin{subfigure}[t]{\threefigwidth}
        \centering
  \includegraphics[width=\textwidth]{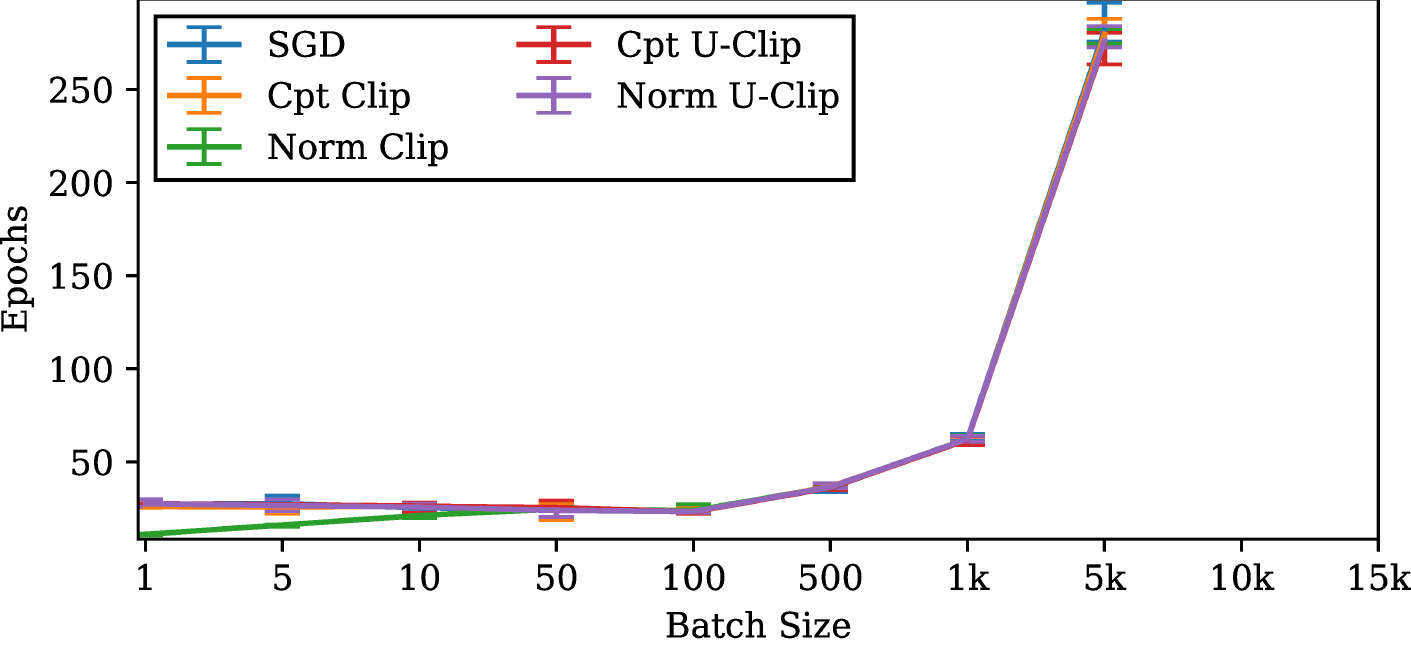}
    \end{subfigure}%
    ~ 
    \begin{subfigure}[t]{\threefigwidth}
        \centering
  \includegraphics[width=\textwidth]{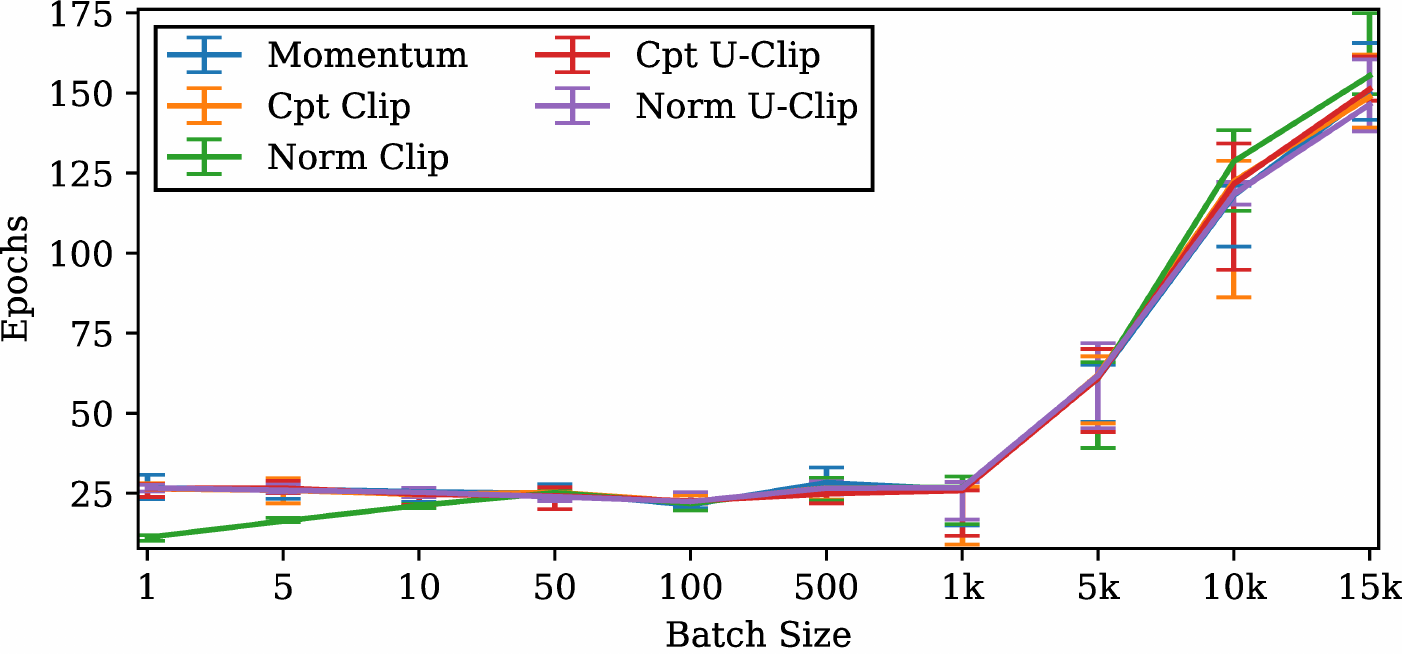}
    \end{subfigure}%
    ~ 
    \begin{subfigure}[t]{\threefigwidth}
        \centering
        \includegraphics[width=\textwidth]{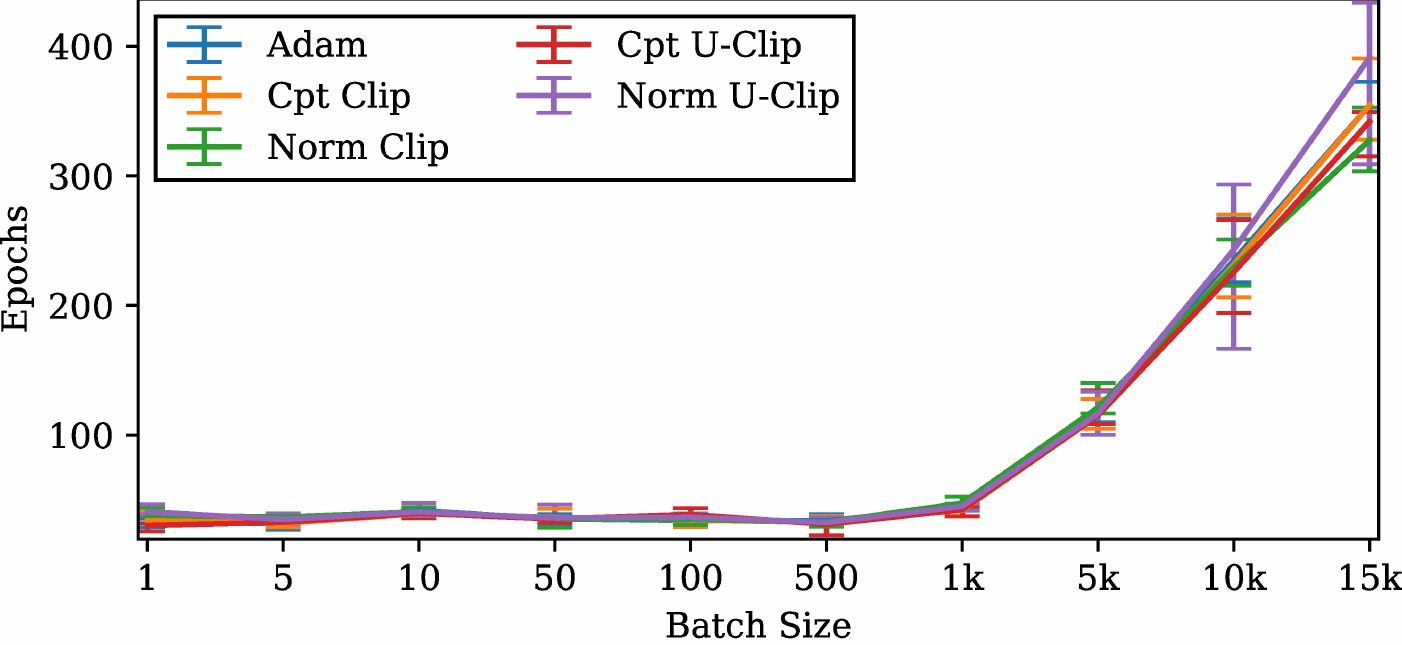}
    \end{subfigure}

    \caption{
      CIFAR10 results for constant clip regions with base optimizers.
      From left to right the base optimizer is SGD, momentum, Adam.
      From top to bottom the clip region is
      $\gamma = 0.25, 0.5, 1, 5, 10$.
      For small batches we see a benefit from norm \uclip{} with small $\gamma$.
      The same chart zoomed in on smaller batch sizes see~\cref{fig:cifar10-constant-short}.
    }
    \label{fig:cifar10-constant}
\end{figure}

\begin{figure}[htpb]
    \centering

    \begin{subfigure}[t]{\threefigwidth}
        \centering
  \includegraphics[width=\textwidth]{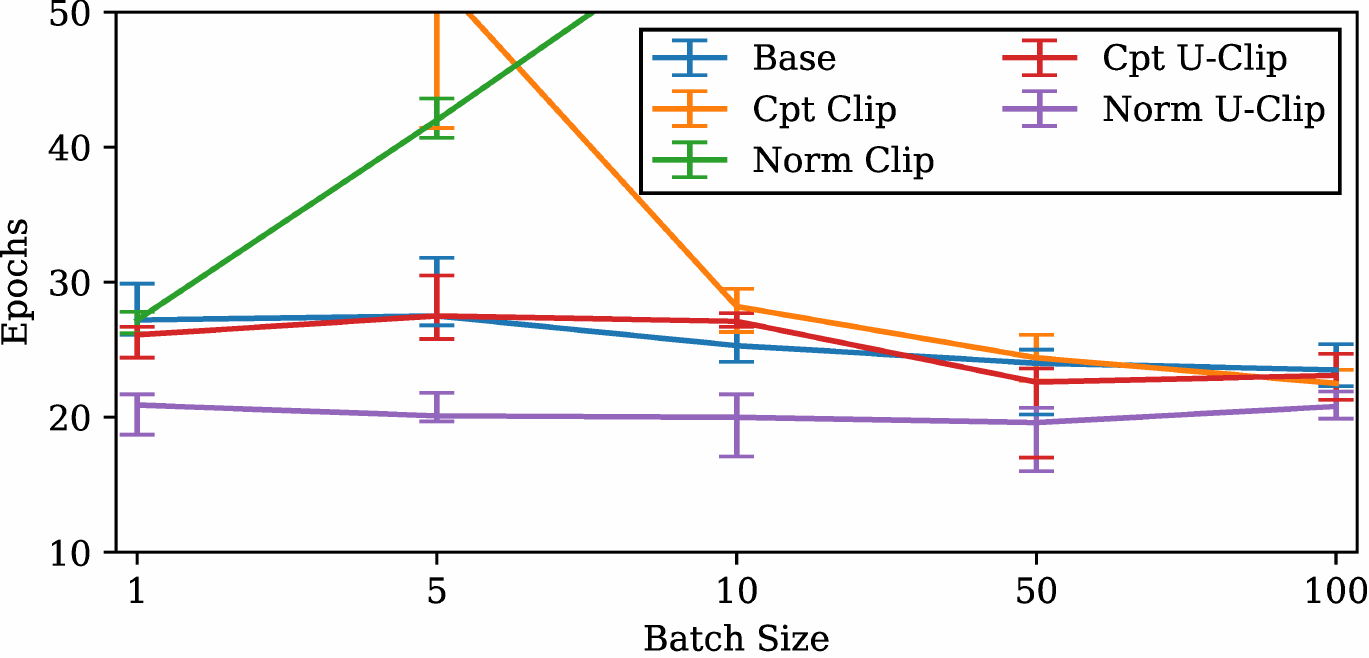}
    \end{subfigure}%
    ~ 
    \begin{subfigure}[t]{\threefigwidth}
        \centering
  \includegraphics[width=\textwidth]{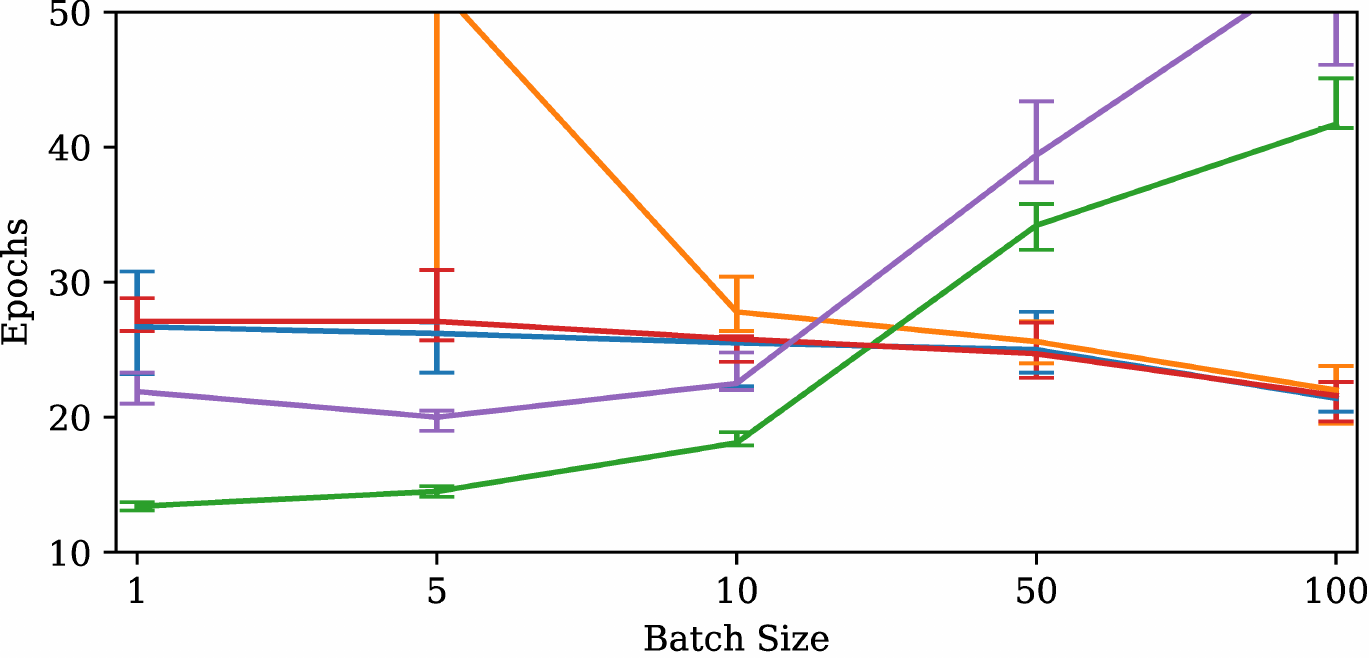}
    \end{subfigure}%
    ~ 
    \begin{subfigure}[t]{\threefigwidth}
        \centering
  \includegraphics[width=\textwidth]{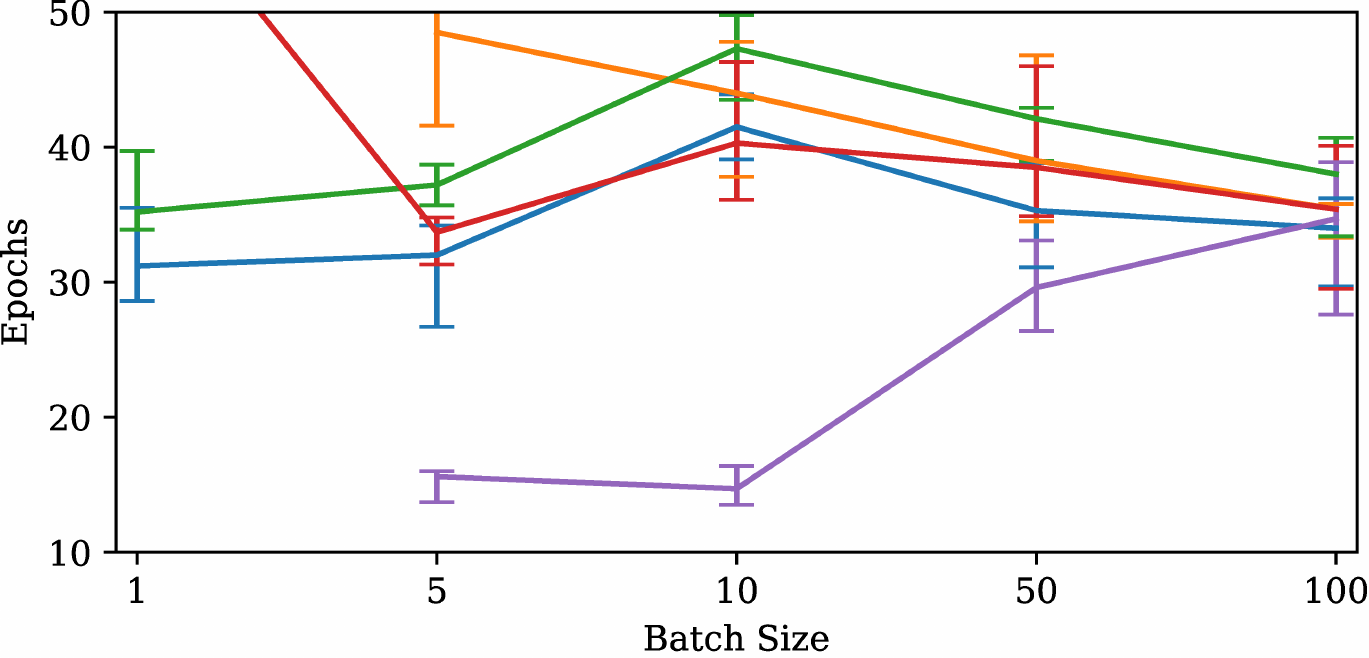}
    \end{subfigure}

    \begin{subfigure}[t]{\threefigwidth}
        \centering
  \includegraphics[width=\textwidth]{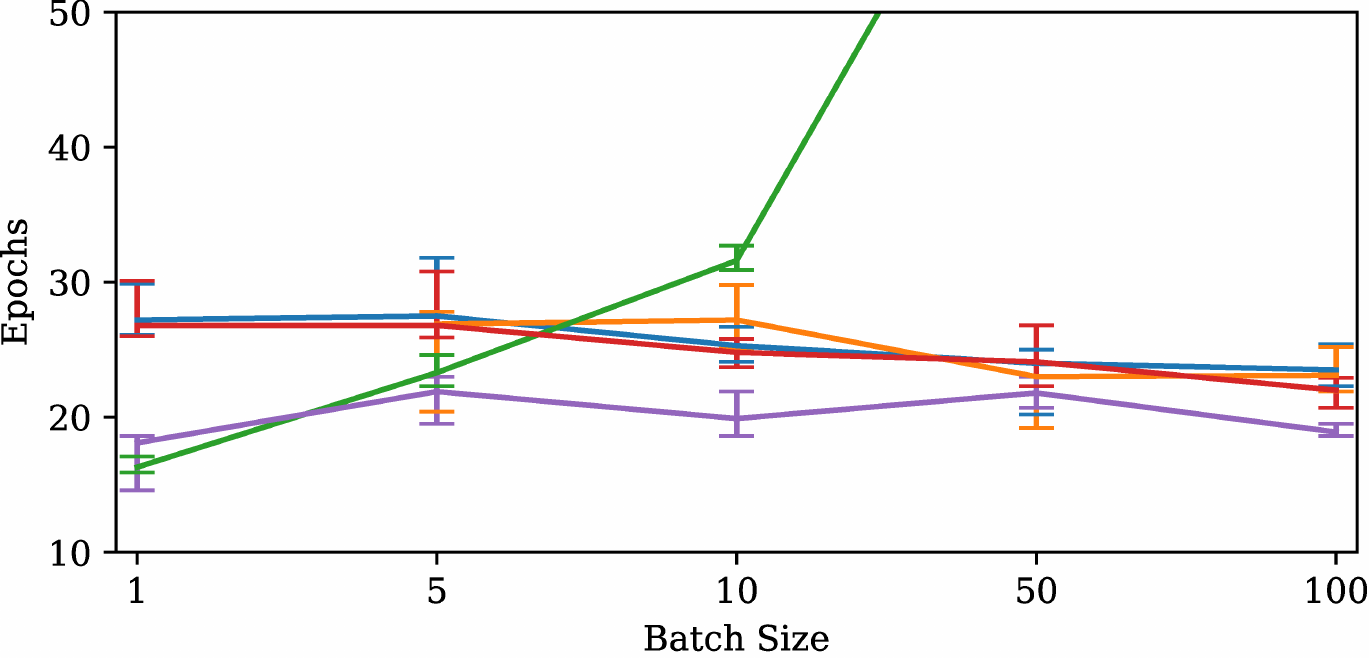}
    \end{subfigure}%
    ~ 
    \begin{subfigure}[t]{\threefigwidth}
        \centering
  \includegraphics[width=\textwidth]{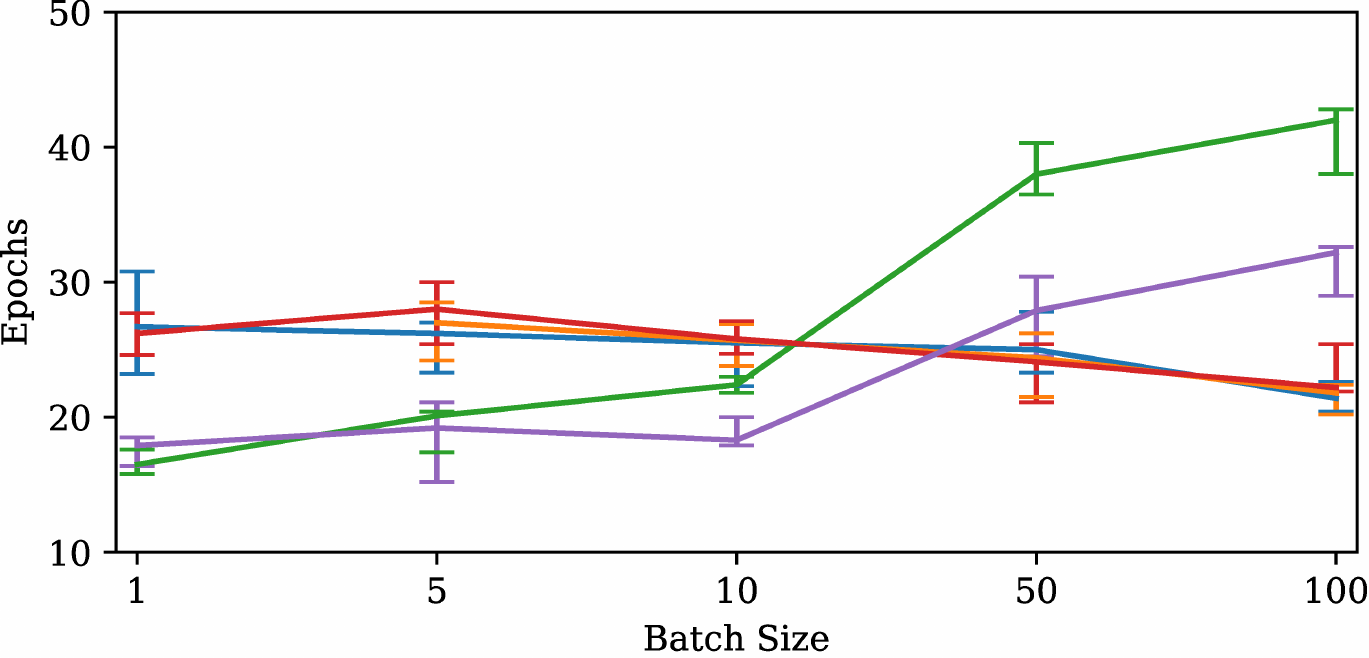}
    \end{subfigure}%
    ~ 
    \begin{subfigure}[t]{\threefigwidth}
        \centering
  \includegraphics[width=\textwidth]{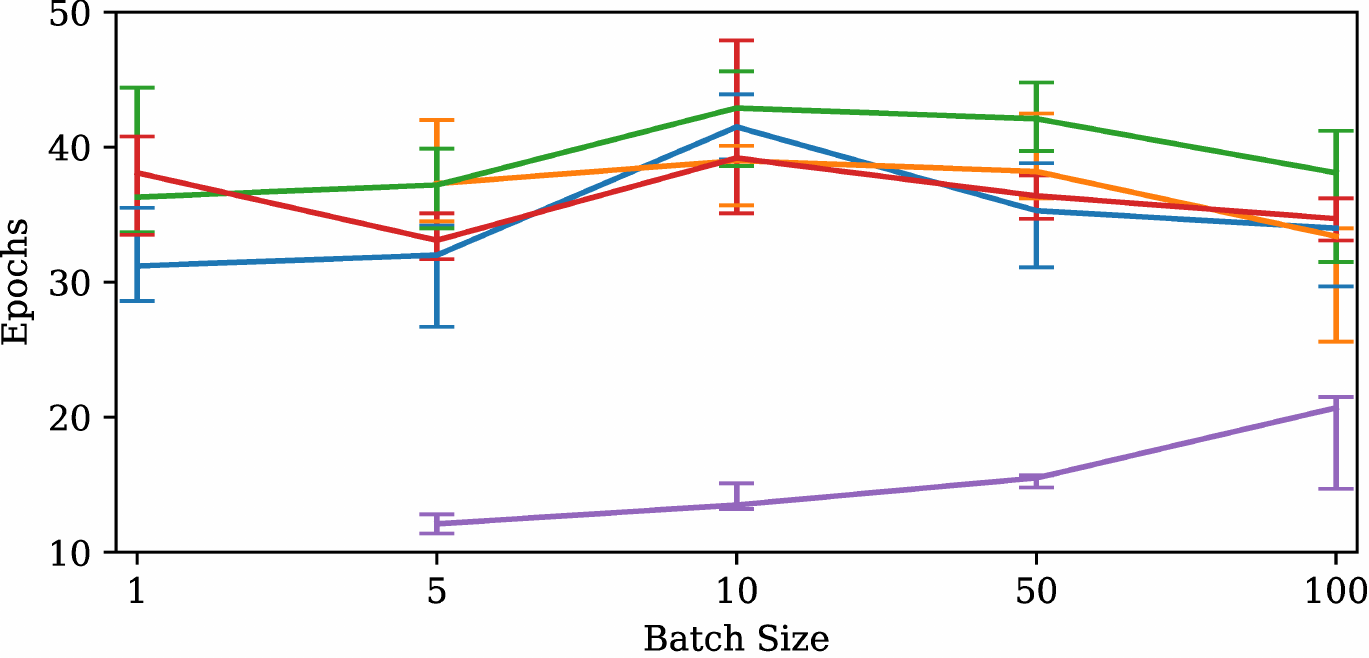}
    \end{subfigure}

    \begin{subfigure}[t]{\threefigwidth}
        \centering
  \includegraphics[width=\textwidth]{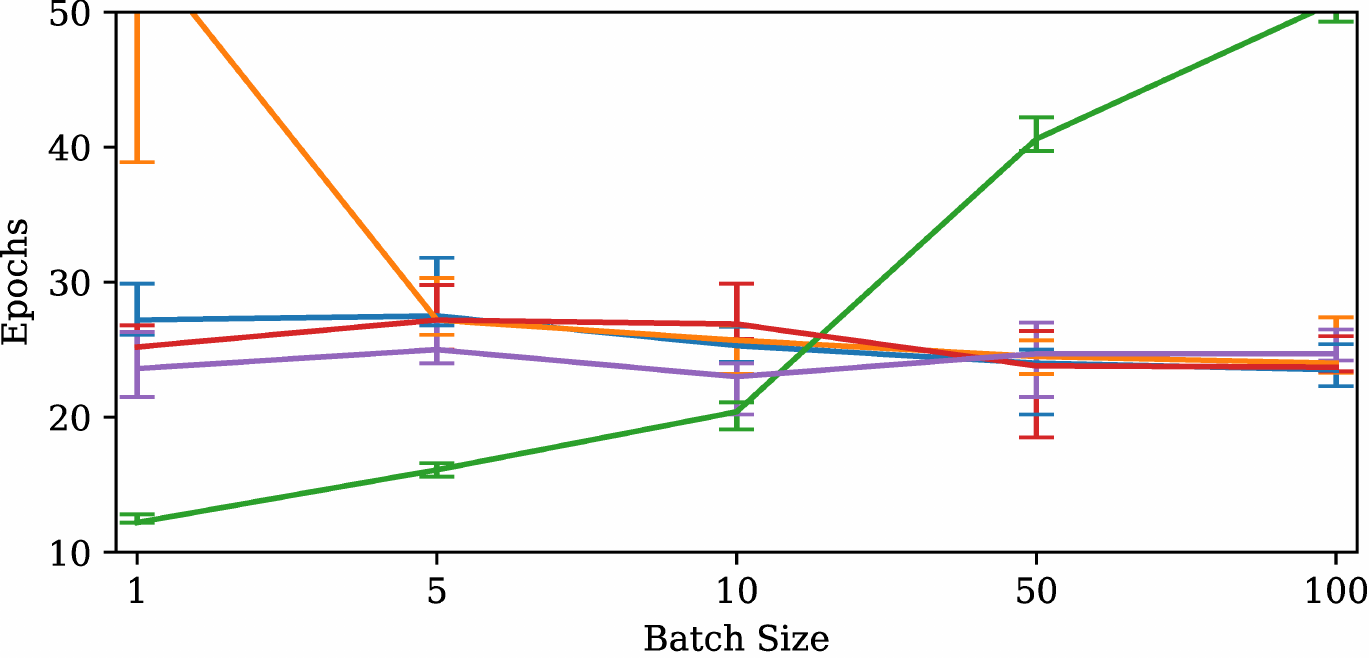}
    \end{subfigure}%
    ~ 
    \begin{subfigure}[t]{\threefigwidth}
        \centering
  \includegraphics[width=\textwidth]{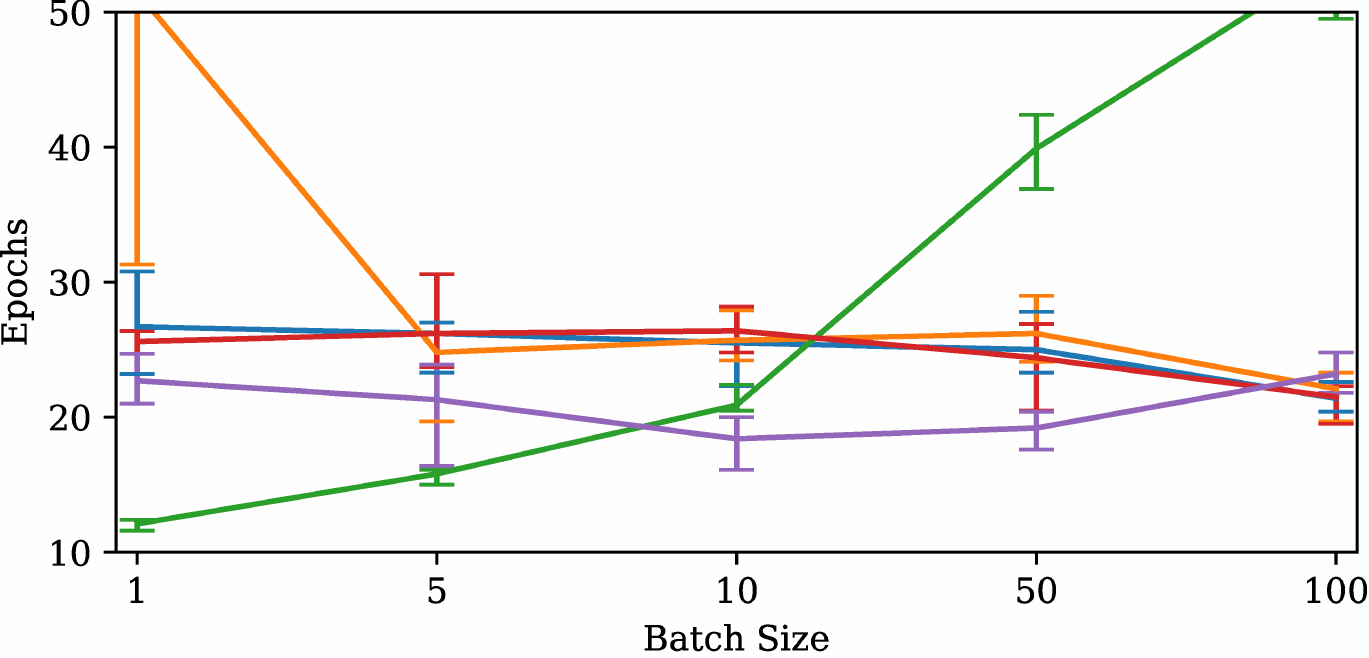}
    \end{subfigure}%
    ~ 
    \begin{subfigure}[t]{\threefigwidth}
        \centering
  \includegraphics[width=\textwidth]{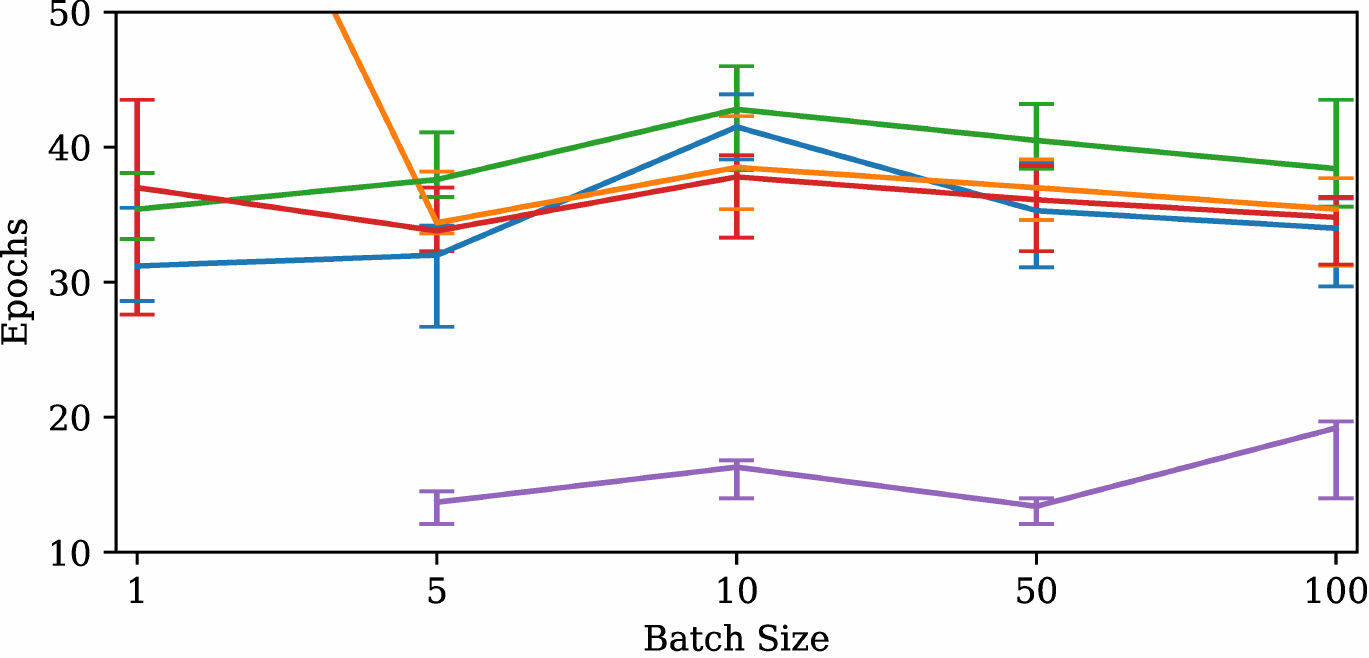}
    \end{subfigure}

    \begin{subfigure}[t]{\threefigwidth}
        \centering
  \includegraphics[width=\textwidth]{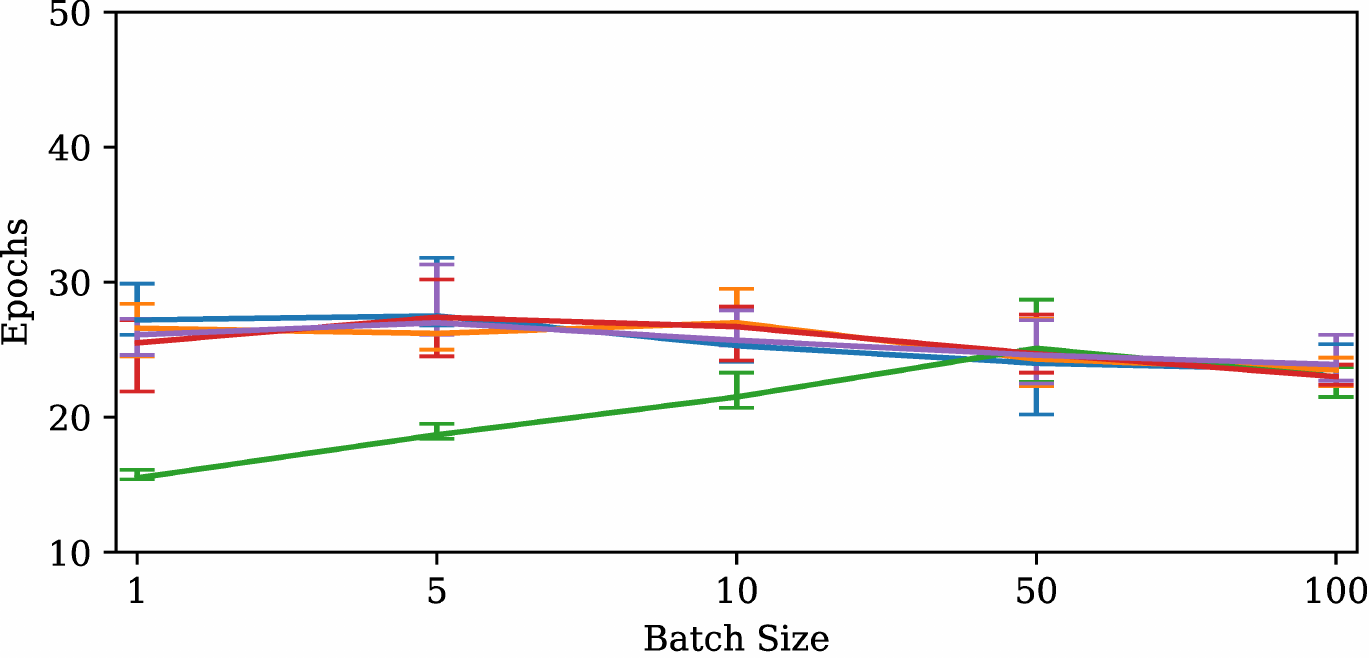}
    \end{subfigure}%
    ~ 
    \begin{subfigure}[t]{\threefigwidth}
        \centering
  \includegraphics[width=\textwidth]{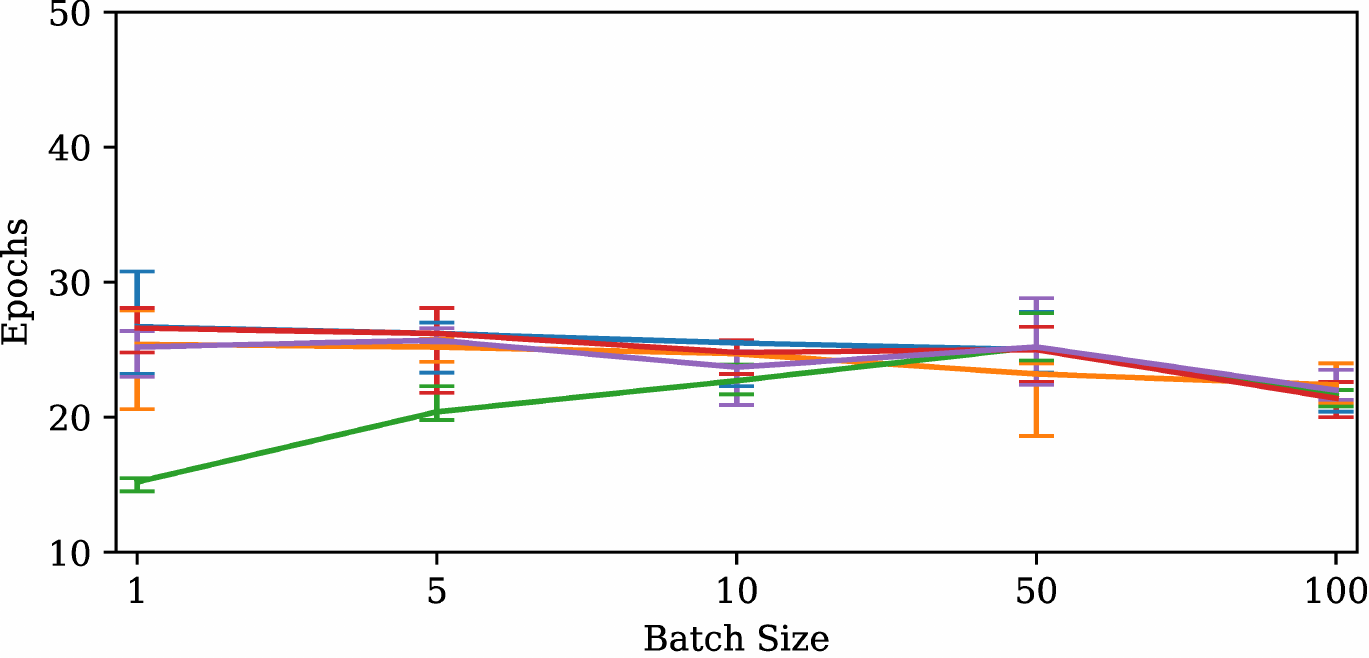}
    \end{subfigure}%
    ~ 
    \begin{subfigure}[t]{\threefigwidth}
        \centering
  \includegraphics[width=\textwidth]{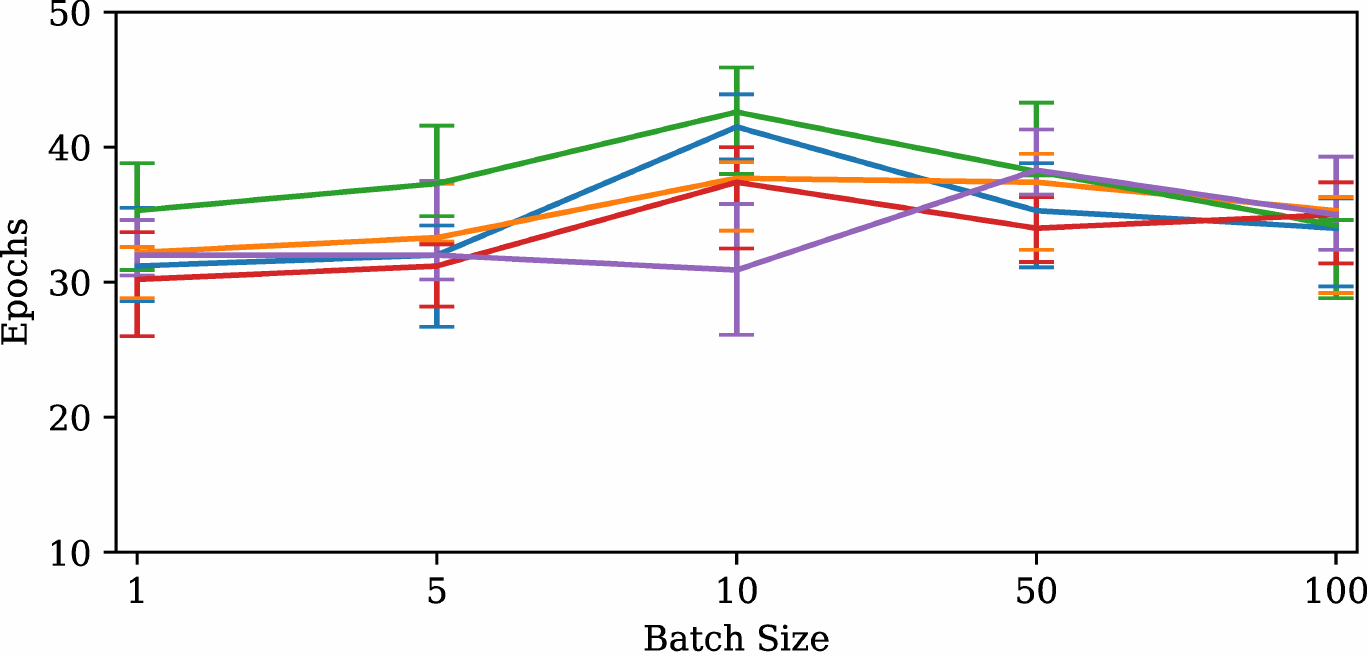}
    \end{subfigure}

    \begin{subfigure}[t]{\threefigwidth}
        \centering
  \includegraphics[width=\textwidth]{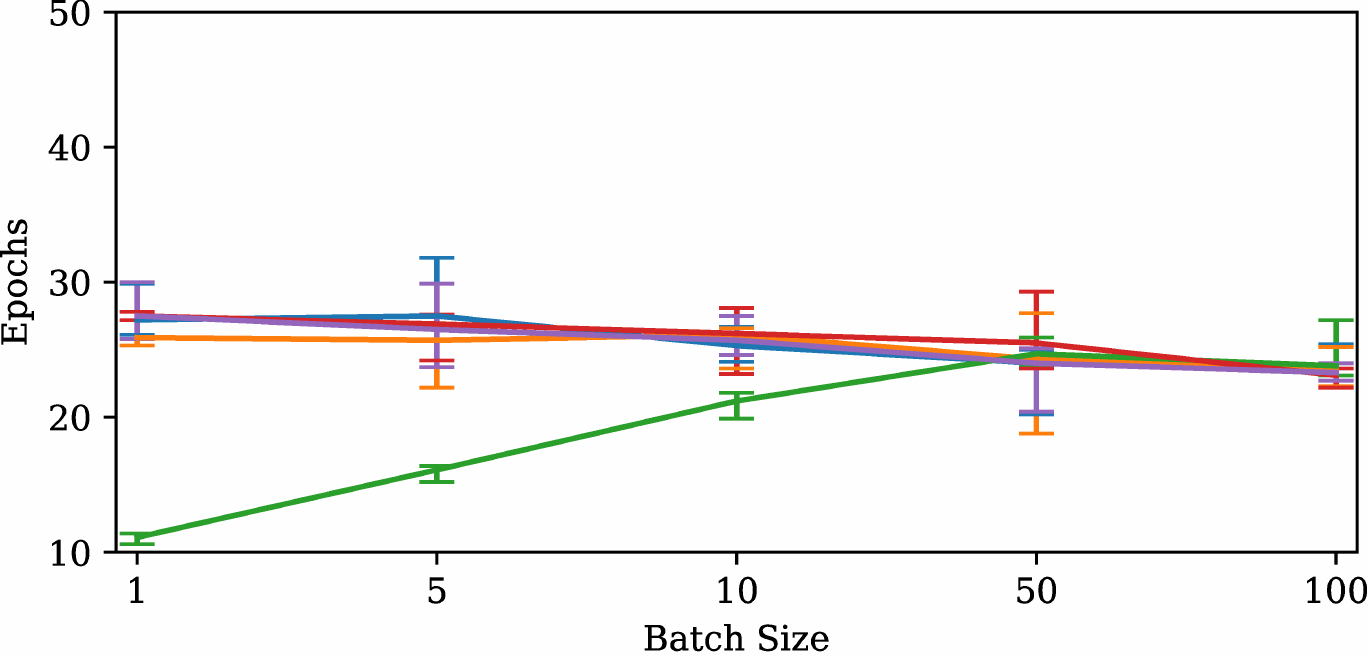}
    \end{subfigure}%
    ~ 
    \begin{subfigure}[t]{\threefigwidth}
        \centering
  \includegraphics[width=\textwidth]{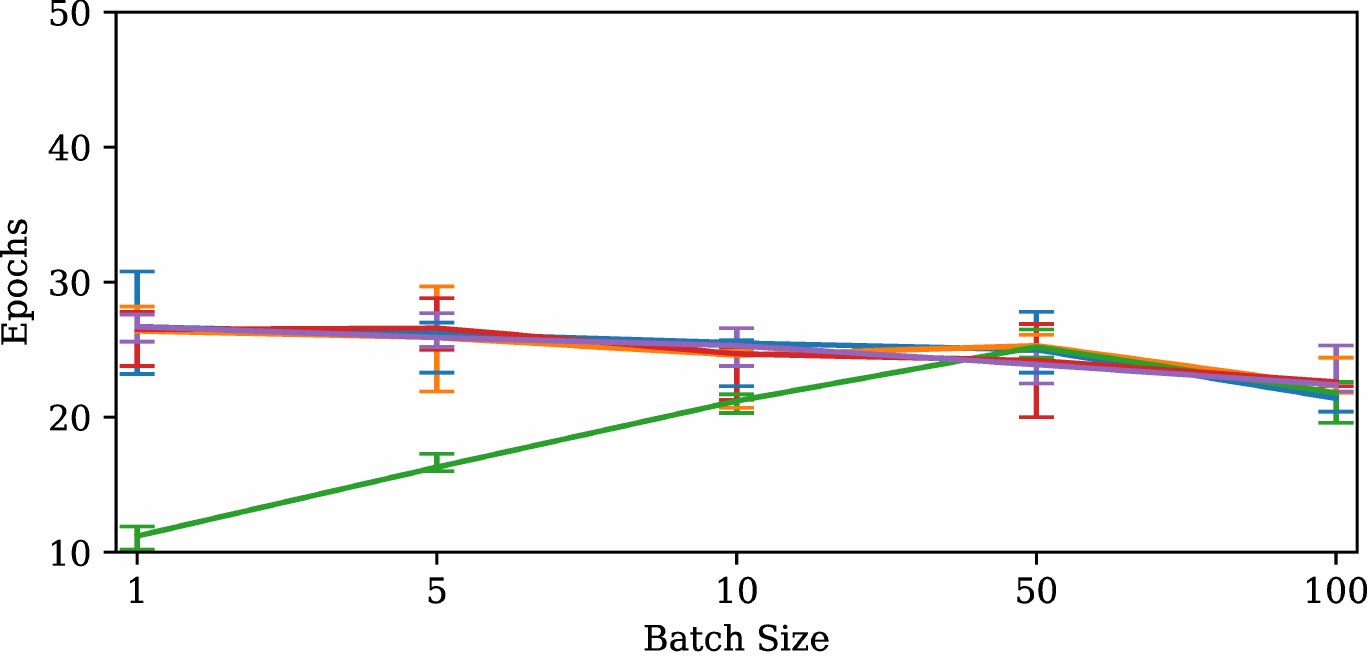}
    \end{subfigure}%
    ~ 
    \begin{subfigure}[t]{\threefigwidth}
        \centering
        \includegraphics[width=\textwidth]{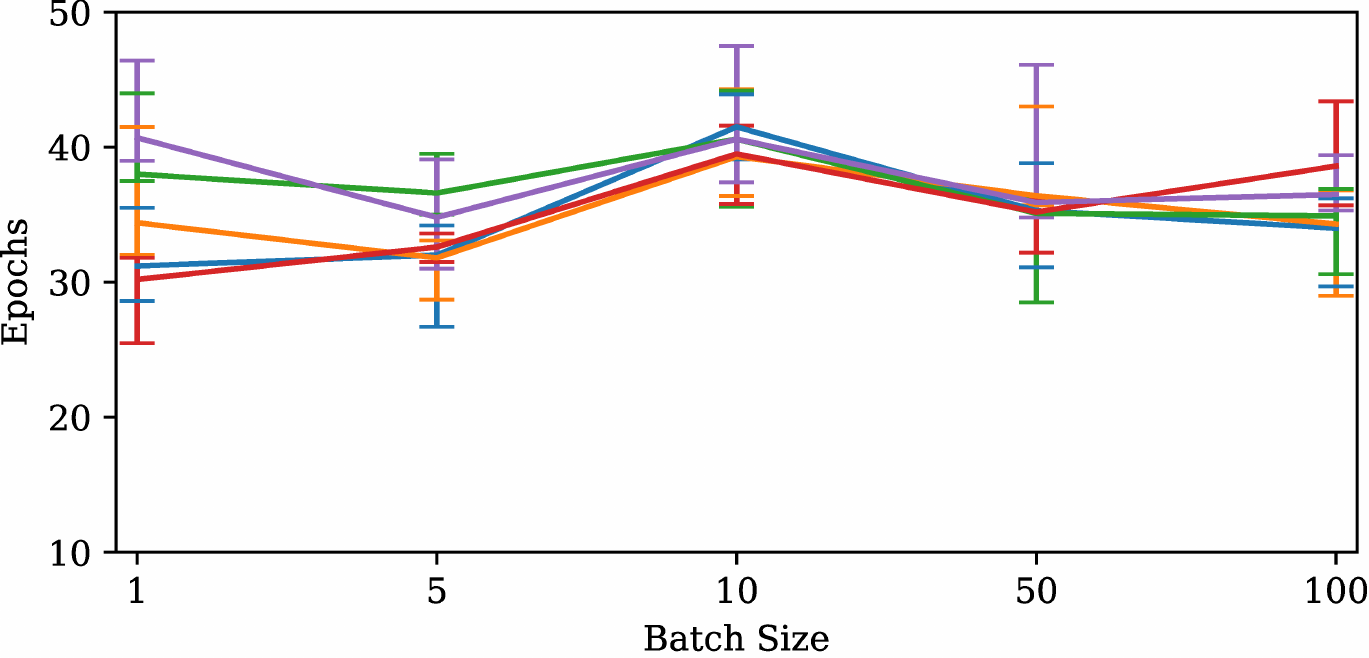}
    \end{subfigure}

    \caption{
      CIFAR10 results for constant clip regions with base optimizers, 
      zoomed in for small batches.
      From left to right the base optimizer is SGD, momentum, Adam.
      From top to bottom the clip region is
      $\gamma = 0.25, 0.5, 1, 5, 10$.
      For small batches we see a benefit from norm \uclip{} with small $\gamma$.
      For all the batch sizes see~\cref{fig:cifar10-constant}.
    }
    \label{fig:cifar10-constant-short}
\end{figure}

\begin{figure}[htpb]
    \centering
    \begin{subfigure}[t]{\twofigwidth}
        \centering
  \includegraphics[width=\textwidth]{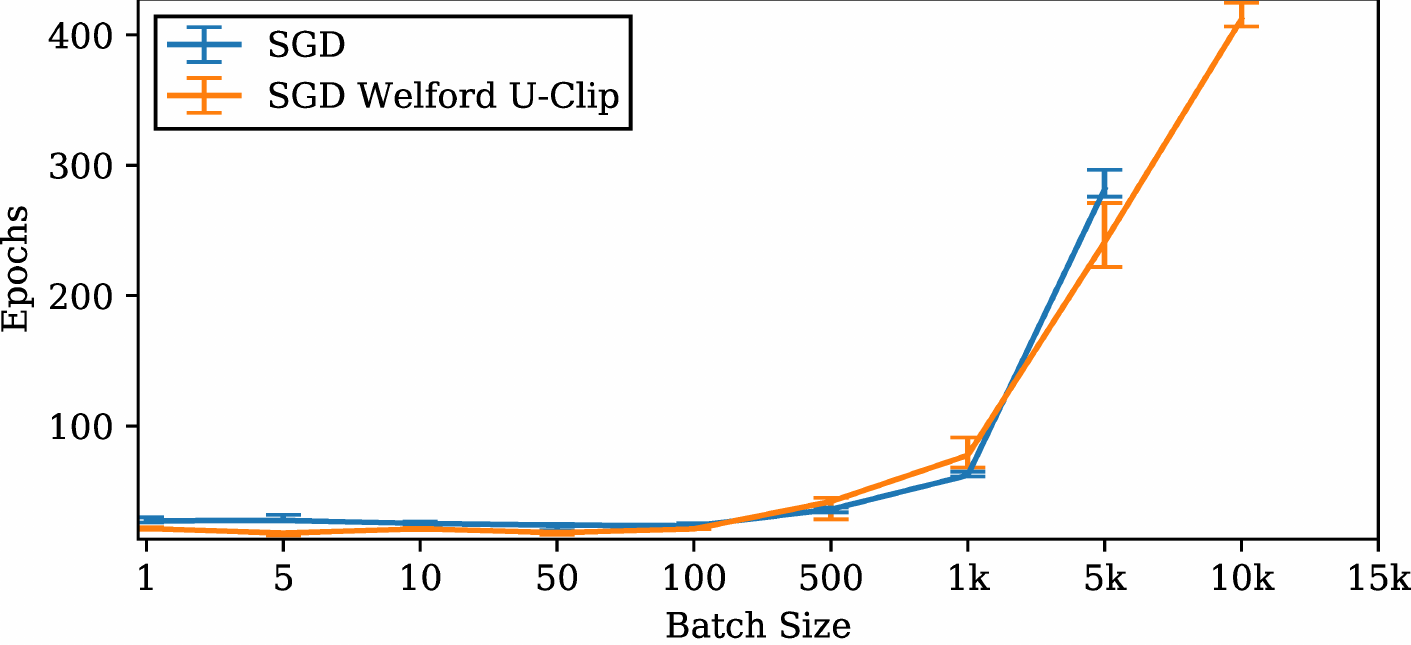}
    \end{subfigure}%
    ~ 
    \begin{subfigure}[t]{\twofigwidth}
        \centering
  \includegraphics[width=\textwidth]{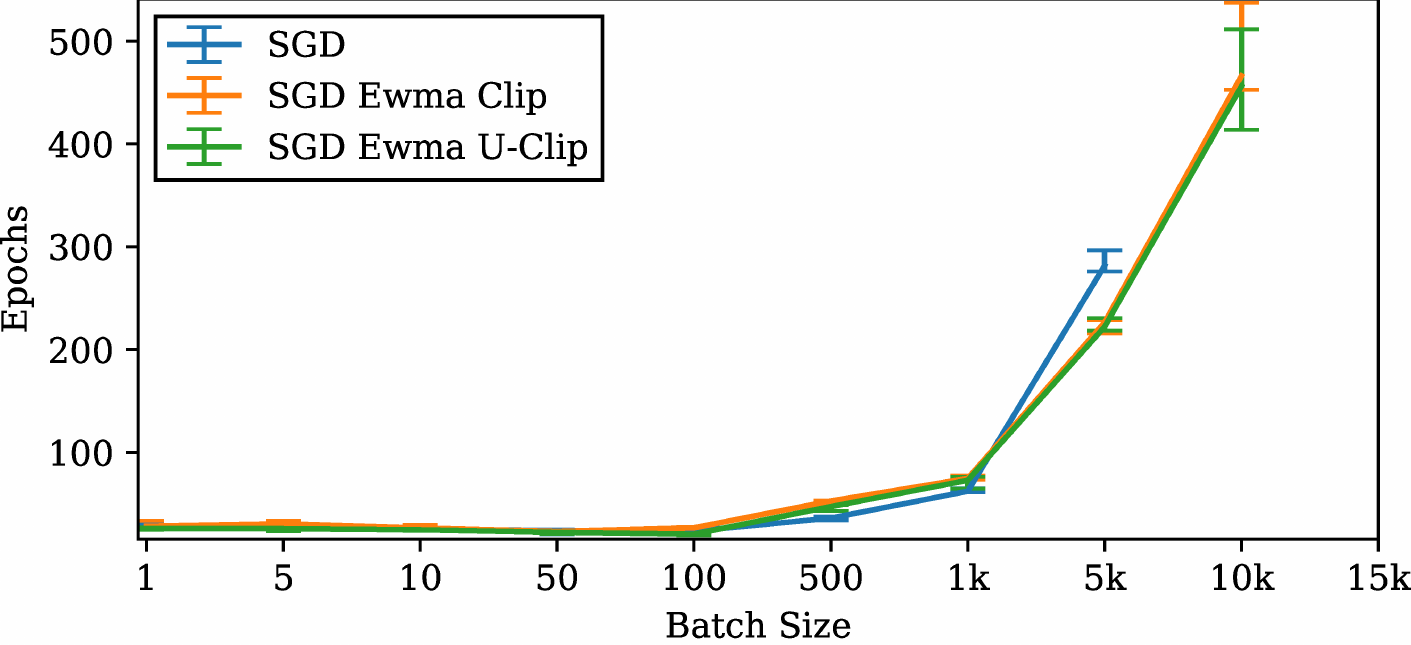}
    \end{subfigure}

    \begin{subfigure}[t]{\twofigwidth}
        \centering
  \includegraphics[width=\textwidth]{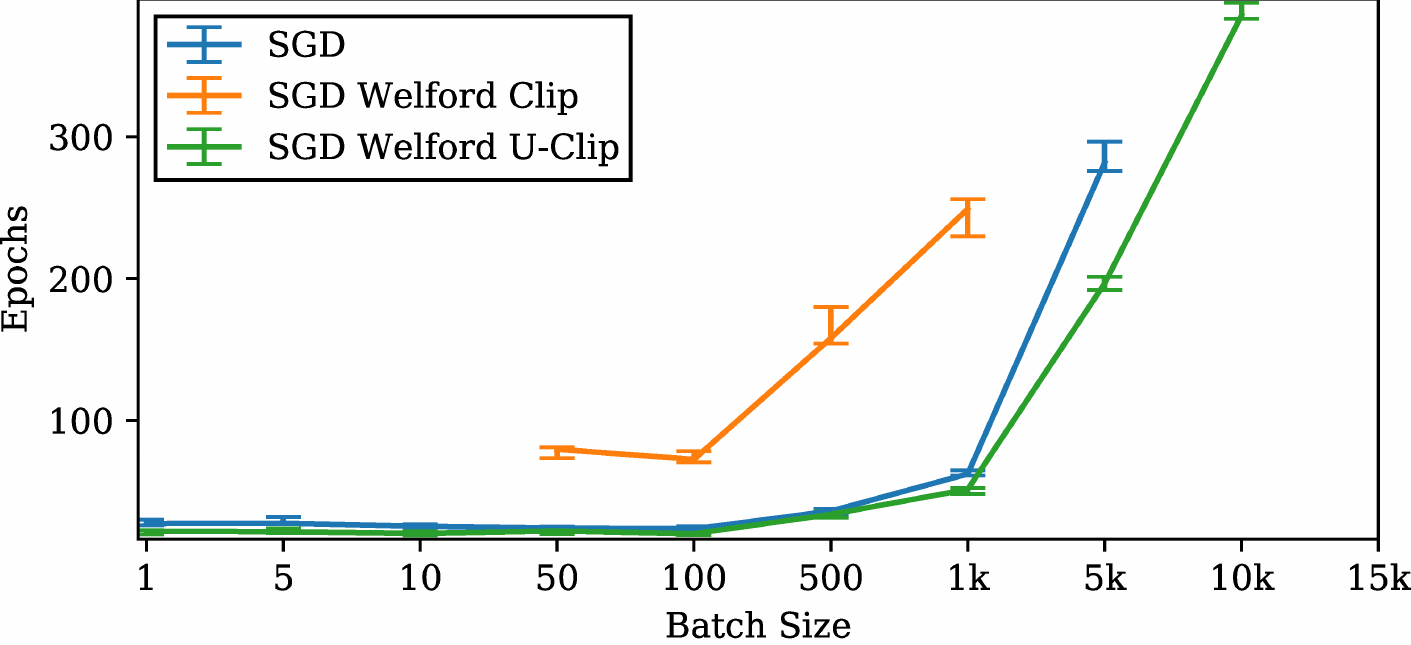}
    \end{subfigure}%
    ~ 
    \begin{subfigure}[t]{\twofigwidth}
        \centering
  \includegraphics[width=\textwidth]{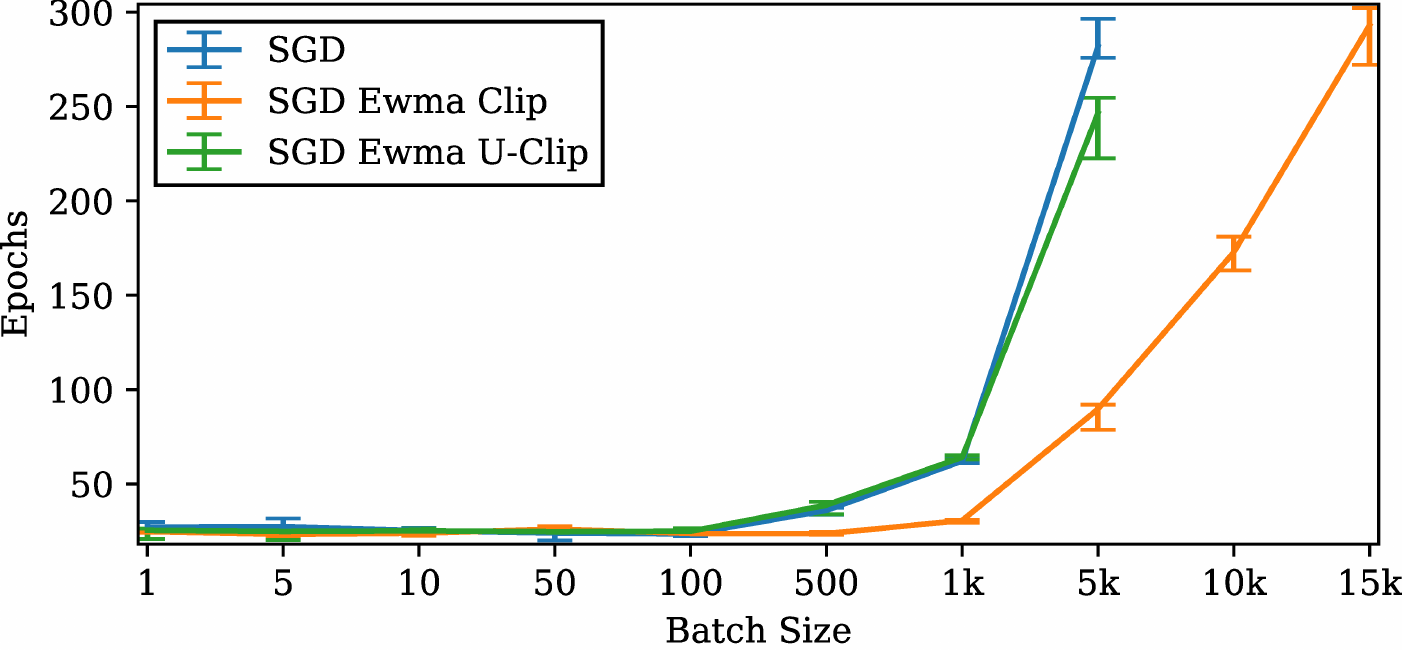}
    \end{subfigure}

    \begin{subfigure}[t]{\twofigwidth}
        \centering
  \includegraphics[width=\textwidth]{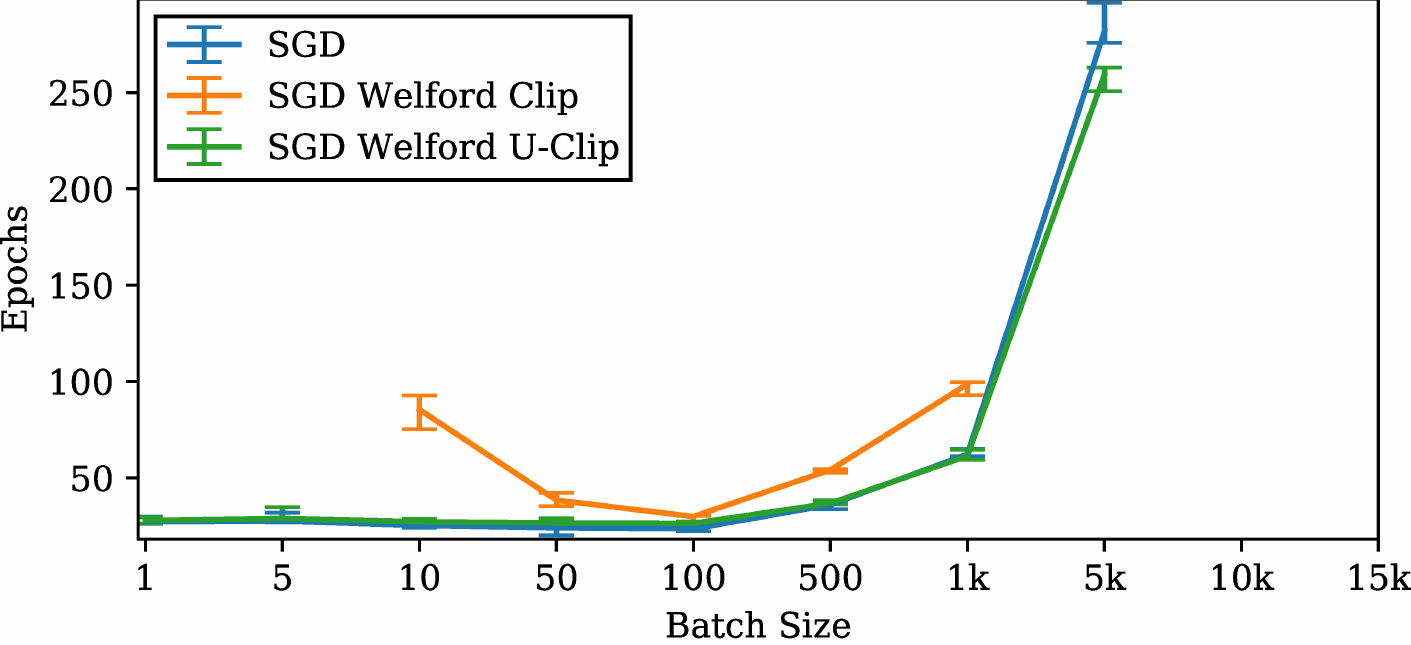}
    \end{subfigure}%
    ~ 
    \begin{subfigure}[t]{\twofigwidth}
        \centering
  \includegraphics[width=\textwidth]{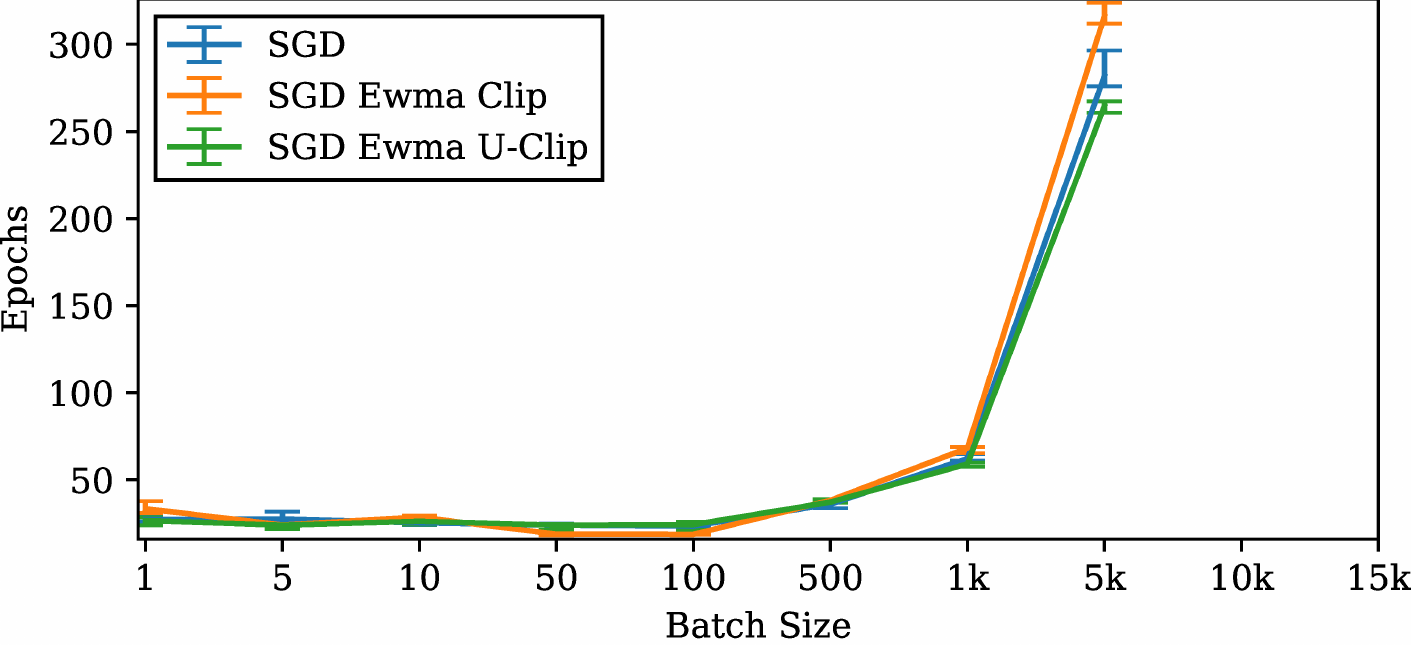}
    \end{subfigure}

    \begin{subfigure}[t]{\twofigwidth}
        \centering
  \includegraphics[width=\textwidth]{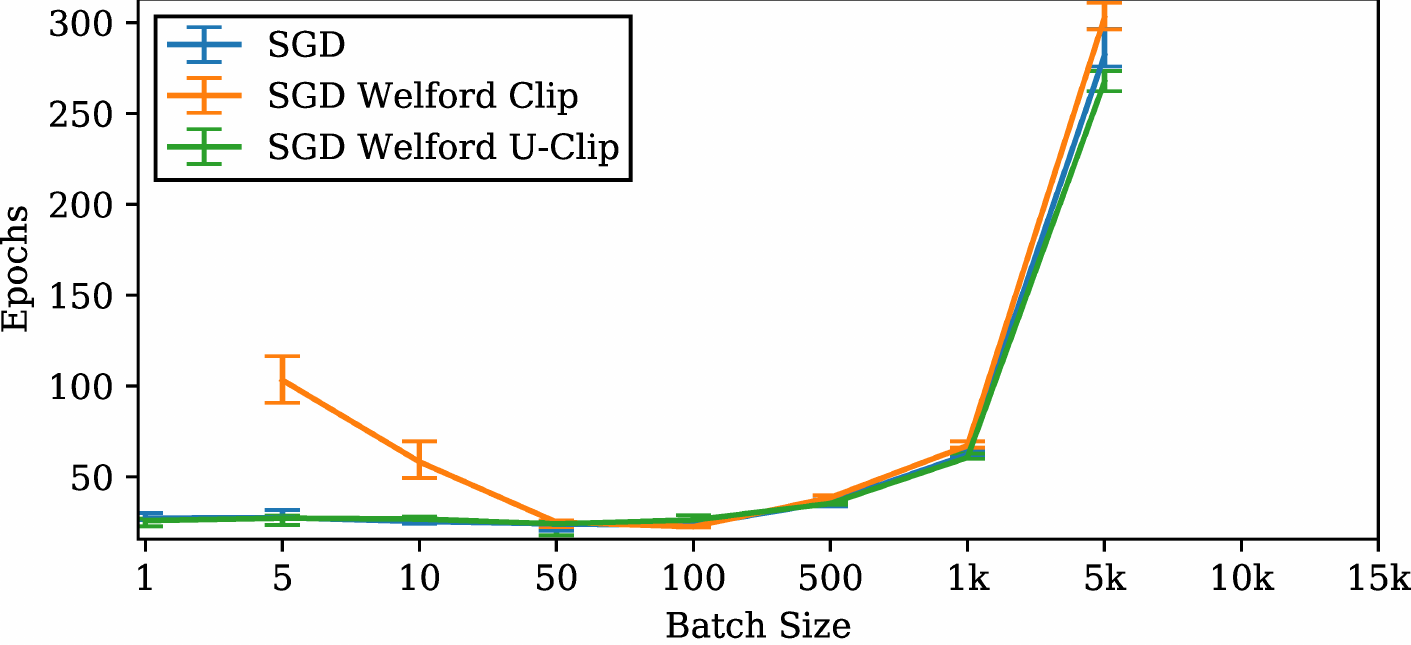}
    \end{subfigure}%
    ~ 
    \begin{subfigure}[t]{\twofigwidth}
        \centering
  \includegraphics[width=\textwidth]{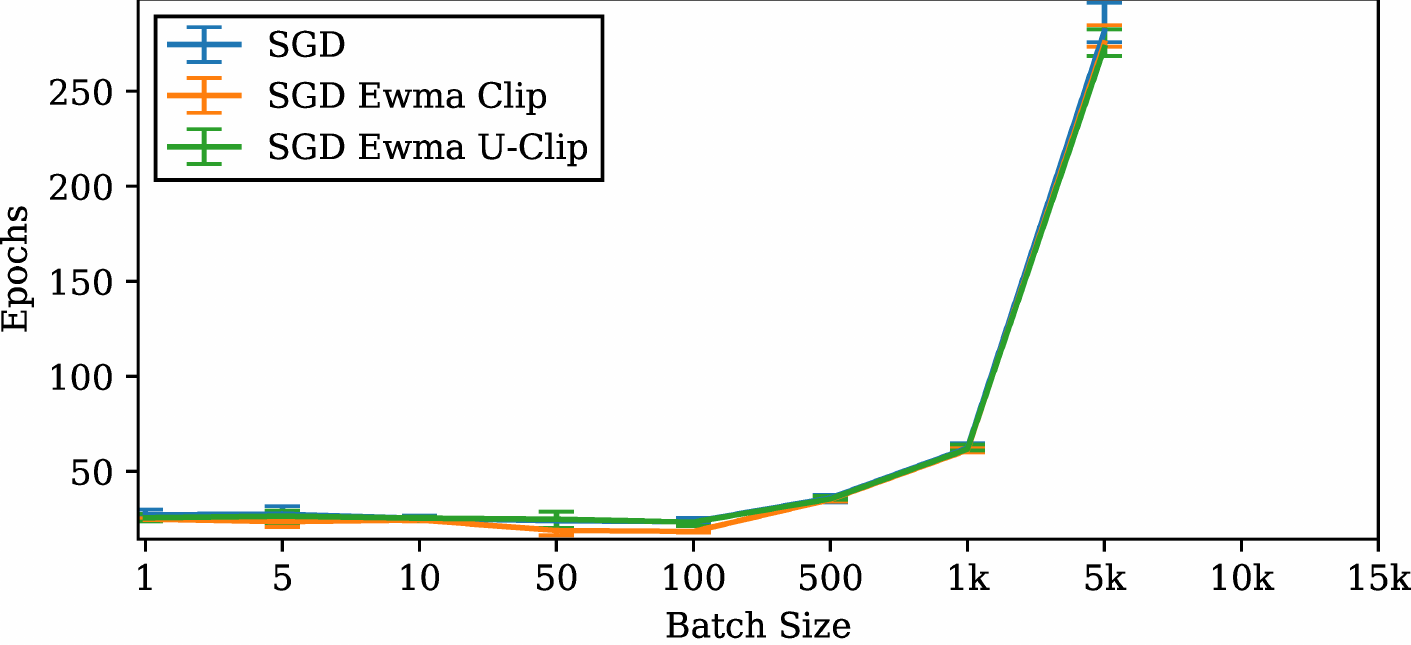}
    \end{subfigure}

    \begin{subfigure}[t]{\twofigwidth}
        \centering
  \includegraphics[width=\textwidth]{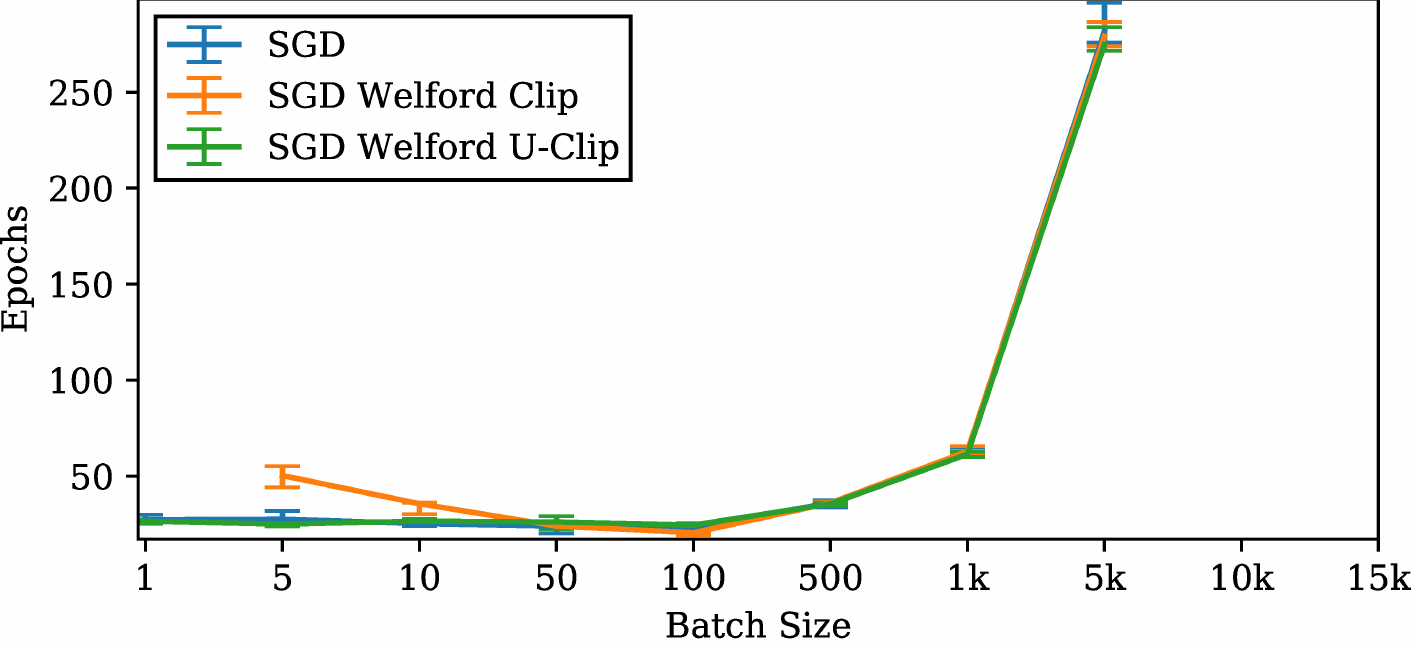}
    \end{subfigure}%
    ~ 
    \begin{subfigure}[t]{\twofigwidth}
        \centering
  \includegraphics[width=\textwidth]{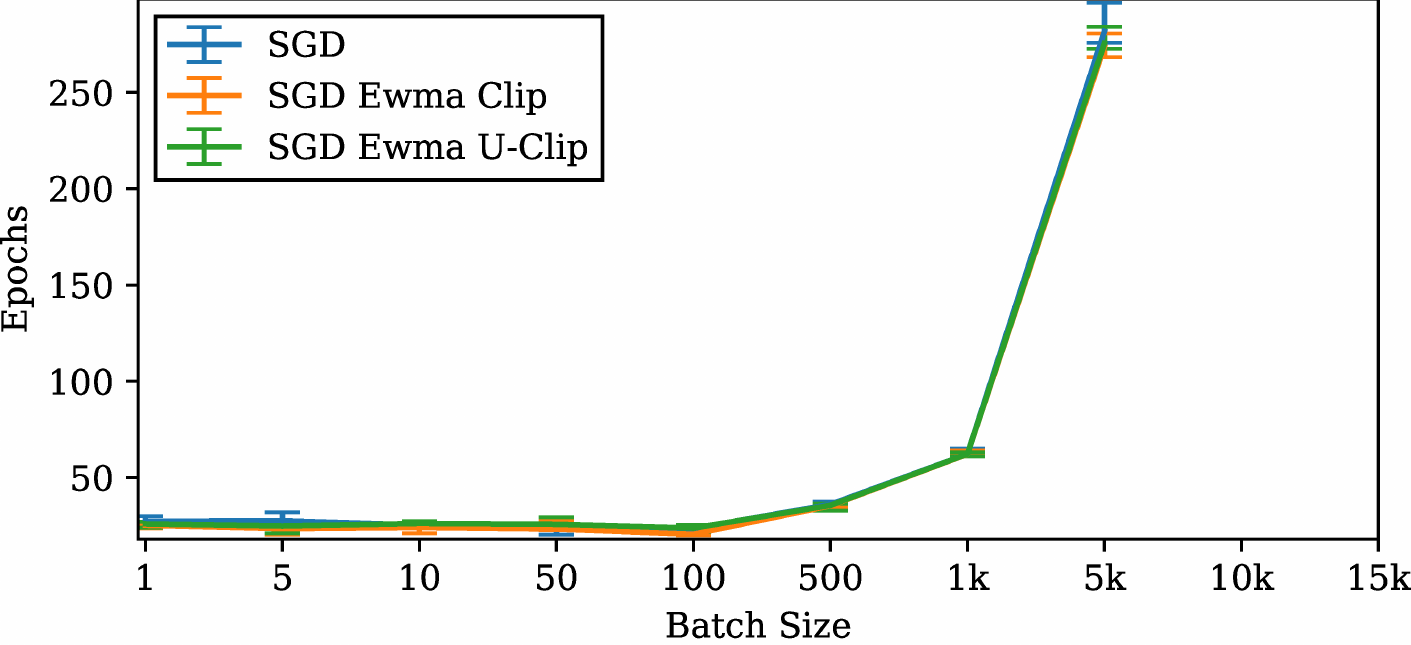}
    \end{subfigure}

    \caption{
      CIFAR10 results for adaptive clip regions with base optimizer SGD. 
      Top to bottom, $(a,b)=(2, 0), (10, 0), (1, 1), (1, 2), (1,3)$.
      Left is Welford estimation, right is EWMA estimation (decay = 0.95).
      Neither adaptive clipping not adaptive \uclip{} adds much over the base
      optimizer, although the \uclip{} seems to perform better than simply clipping
      (no carry) in general. The exception is Welford estimation with $(a,b)=(2, 0)$
  where SGD with adaptive \uclip{} never reaches 99\% train accuracy.}
    \label{fig:cifar10-adaptive-sgd}
\end{figure}

\begin{figure}[htpb]
    \centering
    \begin{subfigure}[t]{\twofigwidth}
        \centering
  \includegraphics[width=\textwidth]{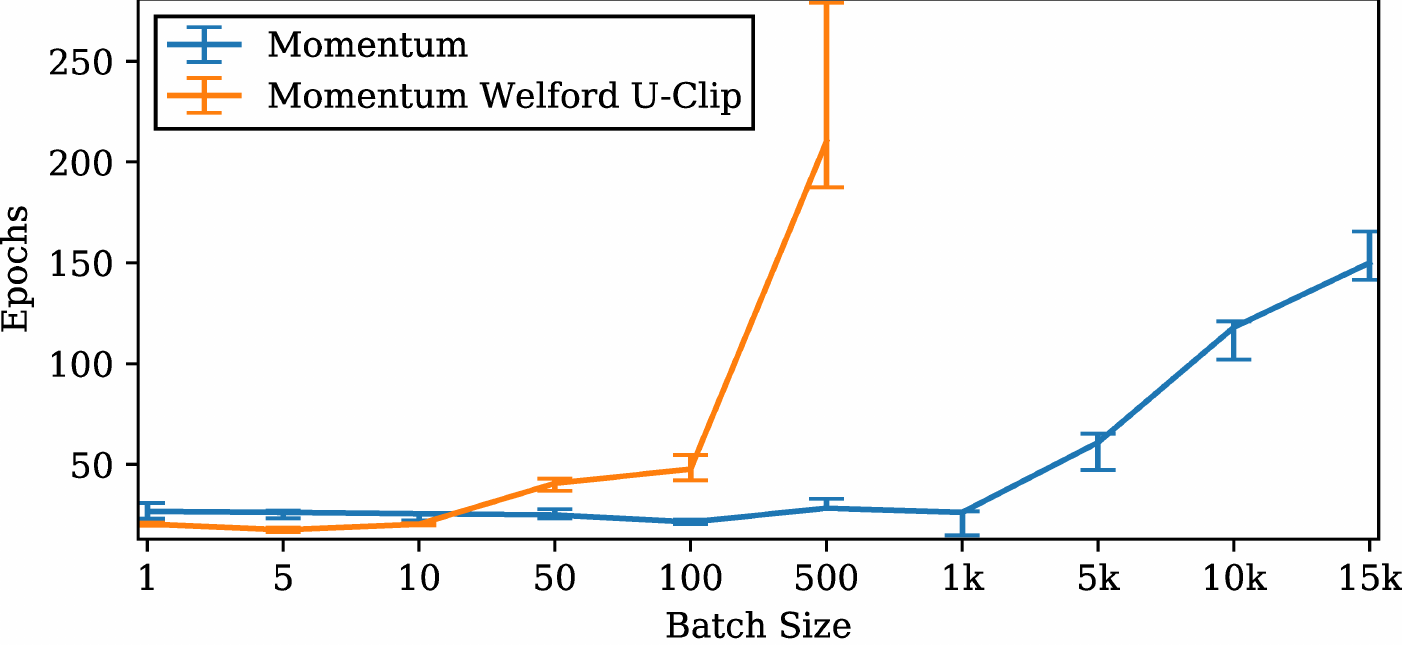}
    \end{subfigure}%
    ~ 
    \begin{subfigure}[t]{\twofigwidth}
        \centering
  \includegraphics[width=\textwidth]{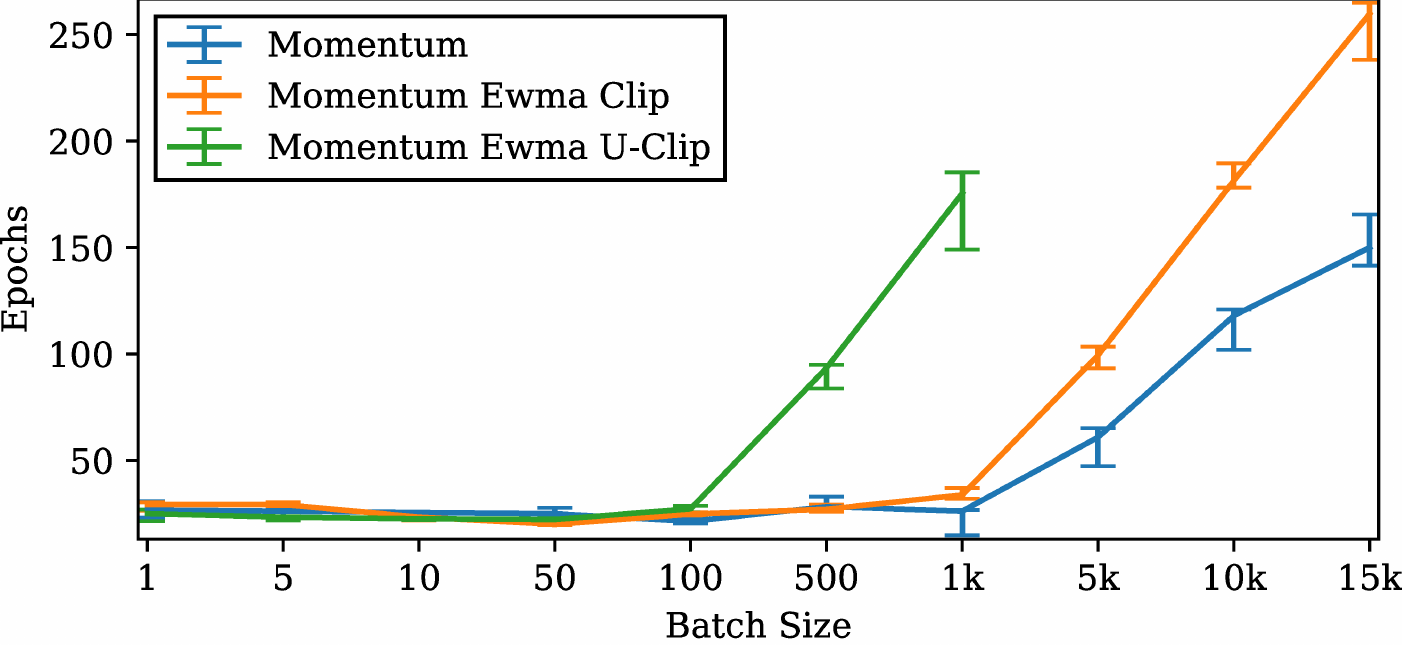}
    \end{subfigure}

    \begin{subfigure}[t]{\twofigwidth}
        \centering
  \includegraphics[width=\textwidth]{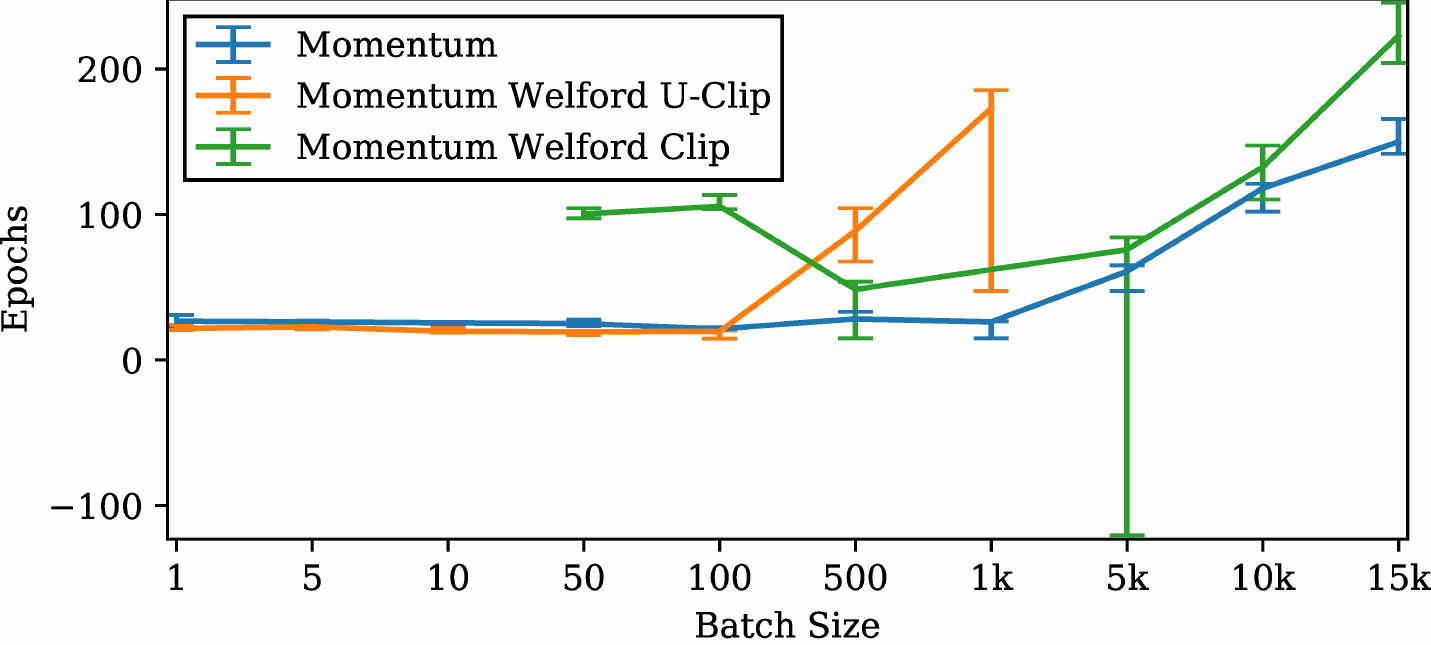}
    \end{subfigure}%
    ~ 
    \begin{subfigure}[t]{\twofigwidth}
        \centering
  \includegraphics[width=\textwidth]{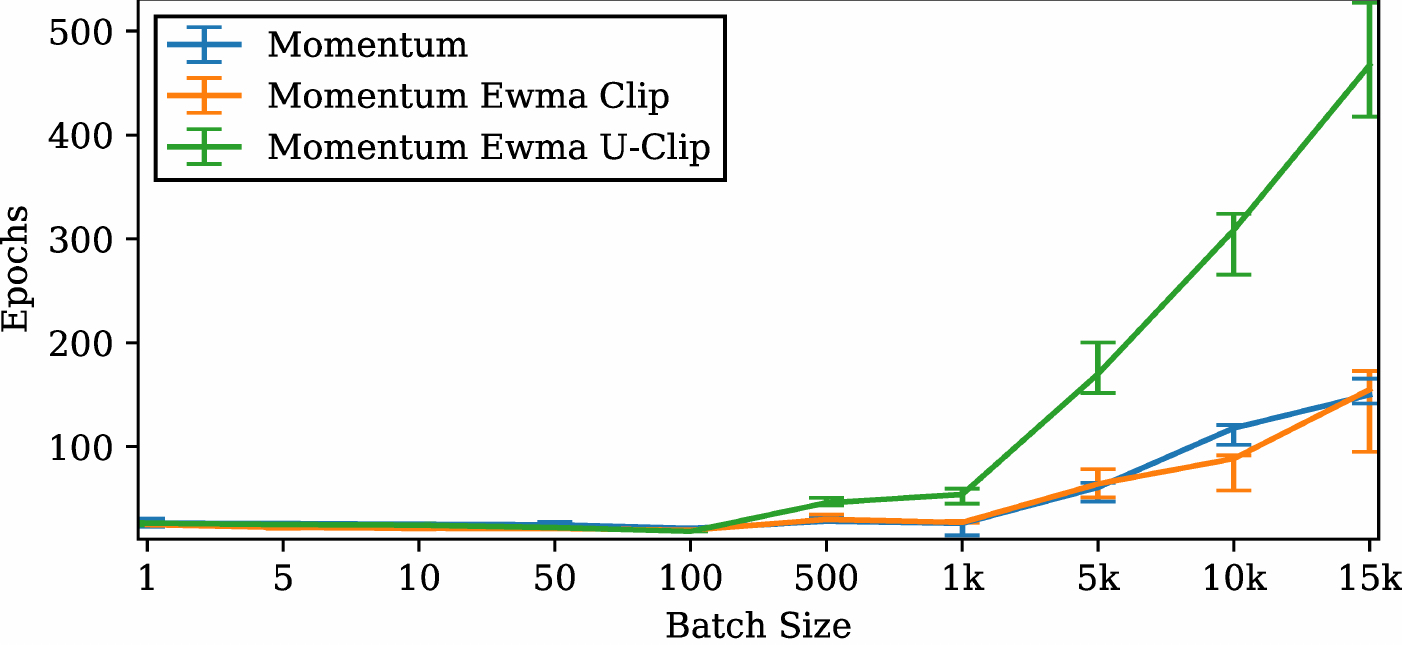}
    \end{subfigure}

    \begin{subfigure}[t]{\twofigwidth}
        \centering
  \includegraphics[width=\textwidth]{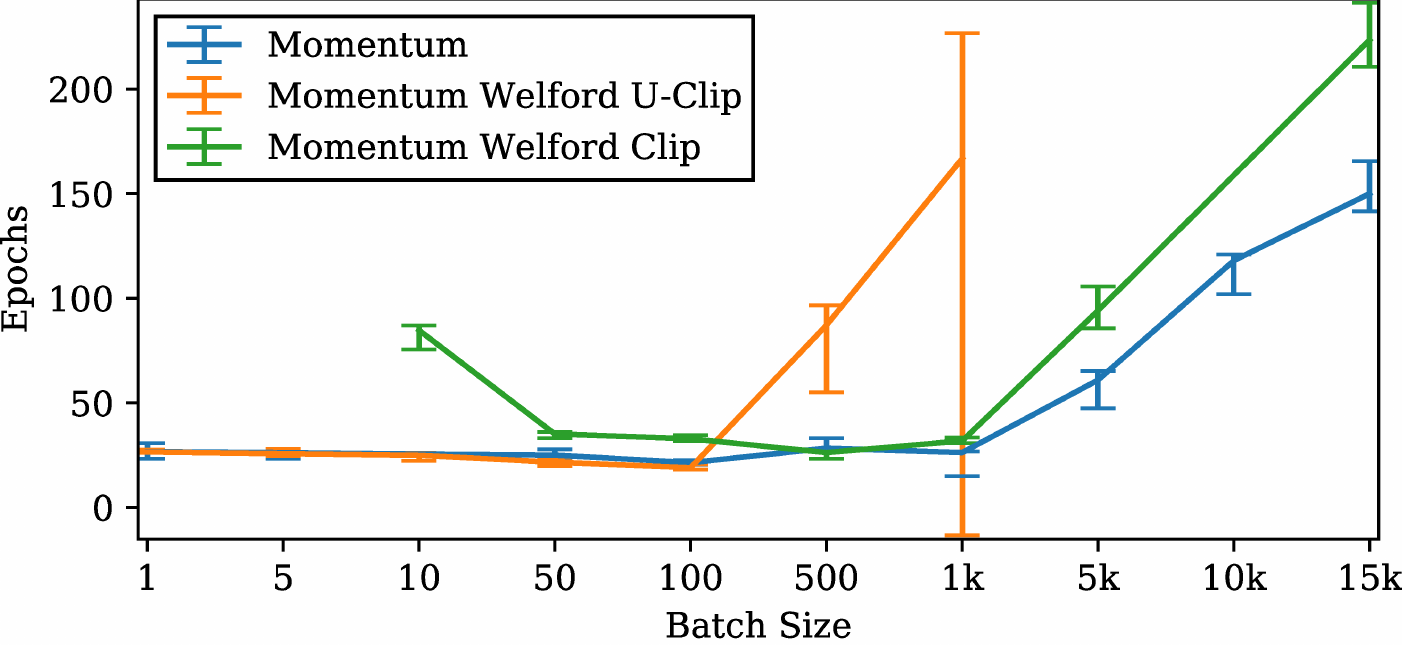}
    \end{subfigure}%
    ~ 
    \begin{subfigure}[t]{\twofigwidth}
        \centering
  \includegraphics[width=\textwidth]{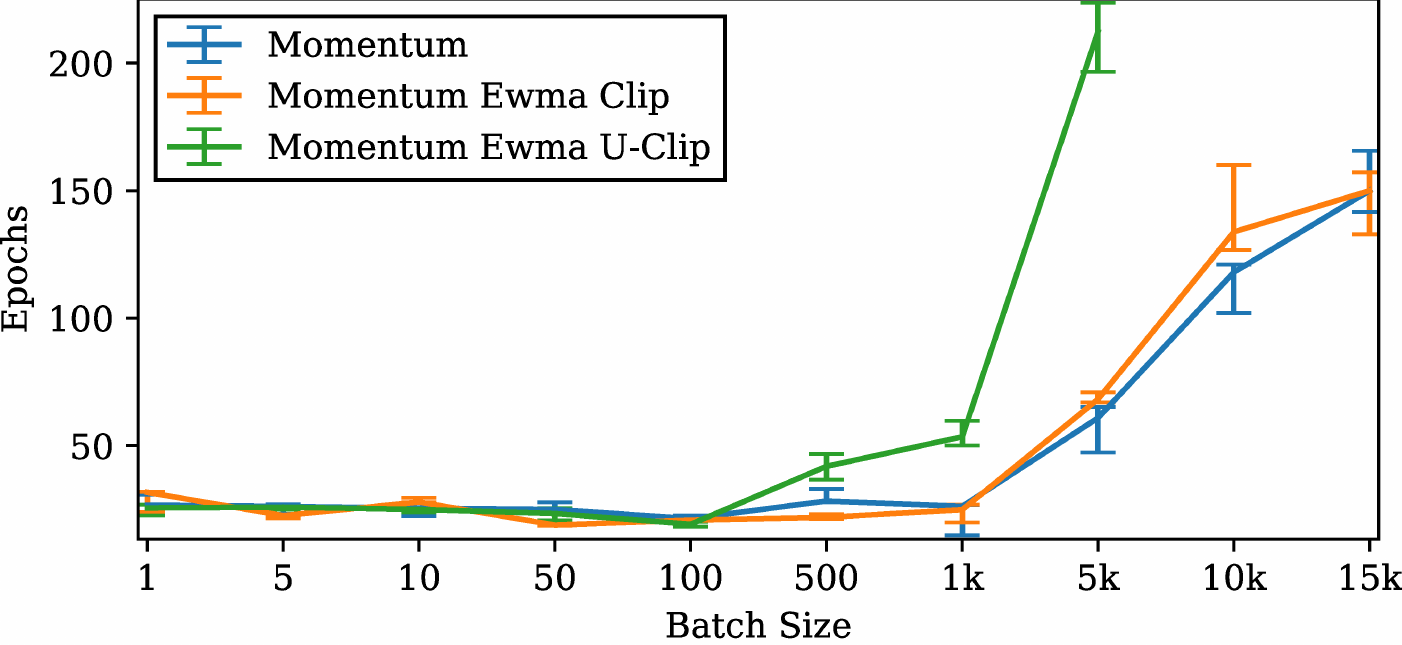}
    \end{subfigure}

    \begin{subfigure}[t]{\twofigwidth}
        \centering
  \includegraphics[width=\textwidth]{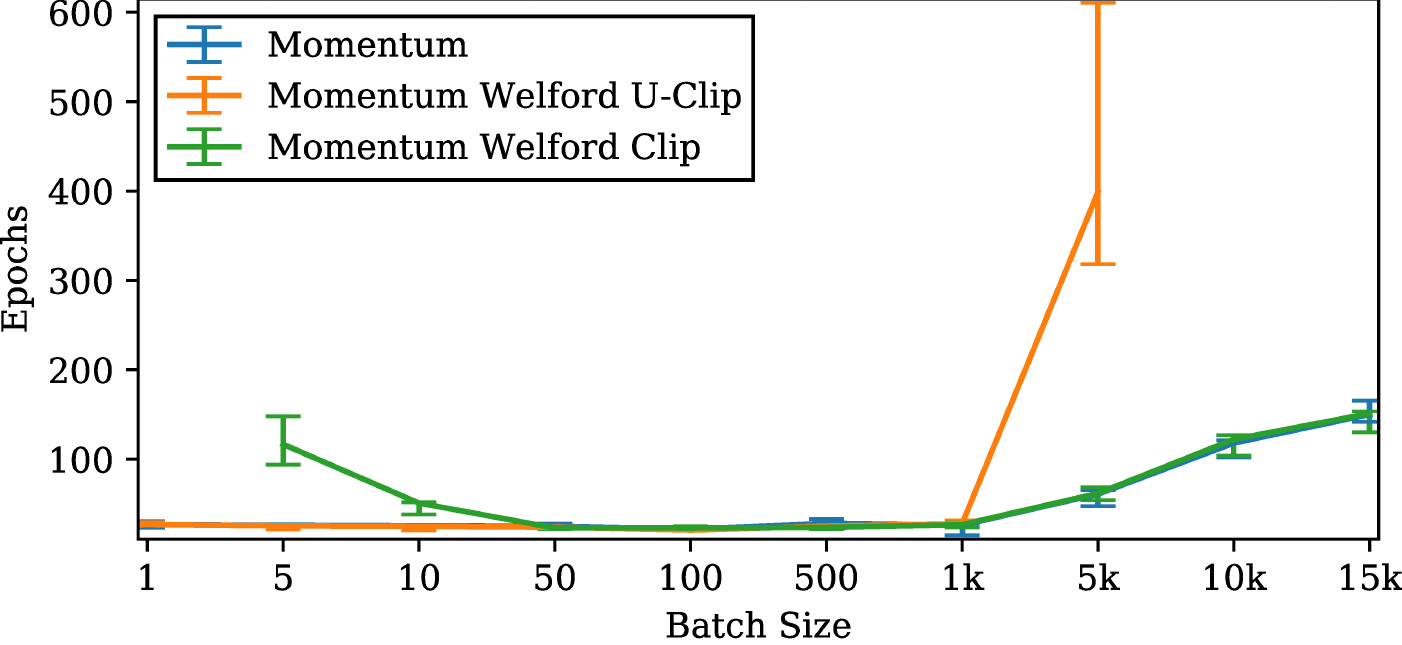}
    \end{subfigure}%
    ~ 
    \begin{subfigure}[t]{\twofigwidth}
        \centering
  \includegraphics[width=\textwidth]{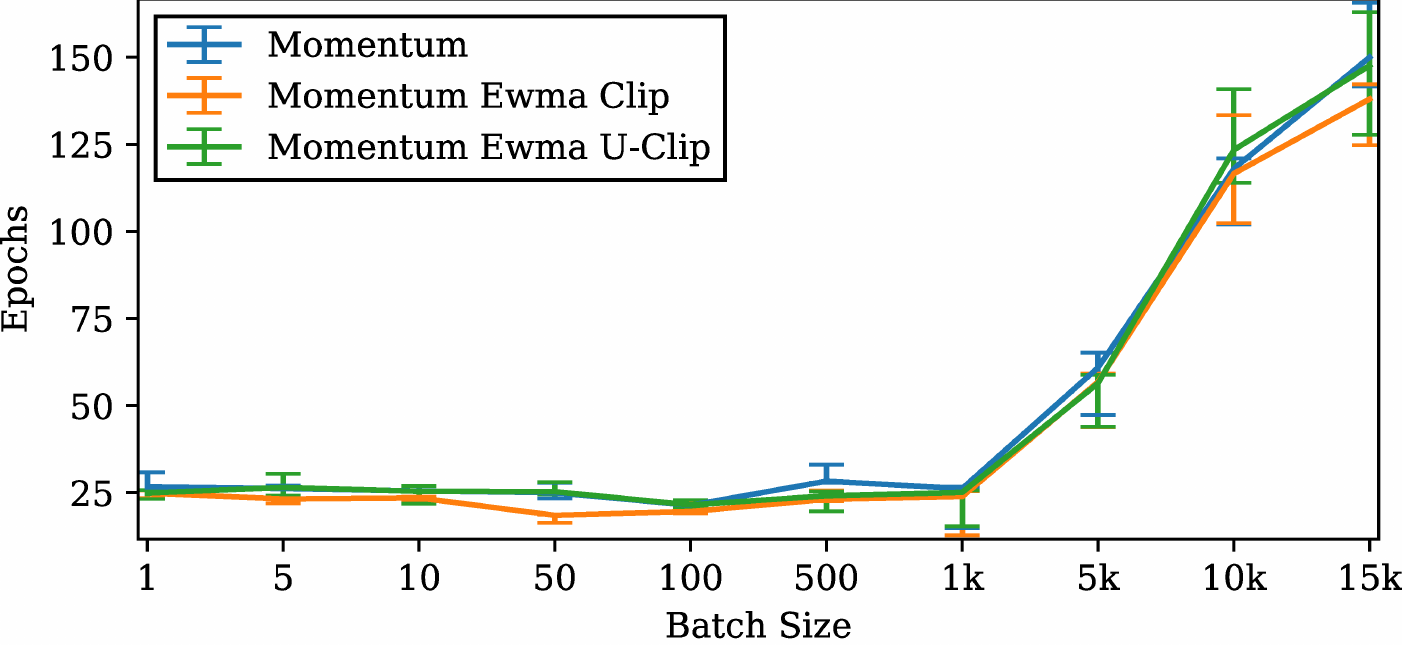}
    \end{subfigure}

    \begin{subfigure}[t]{\twofigwidth}
        \centering
  \includegraphics[width=\textwidth]{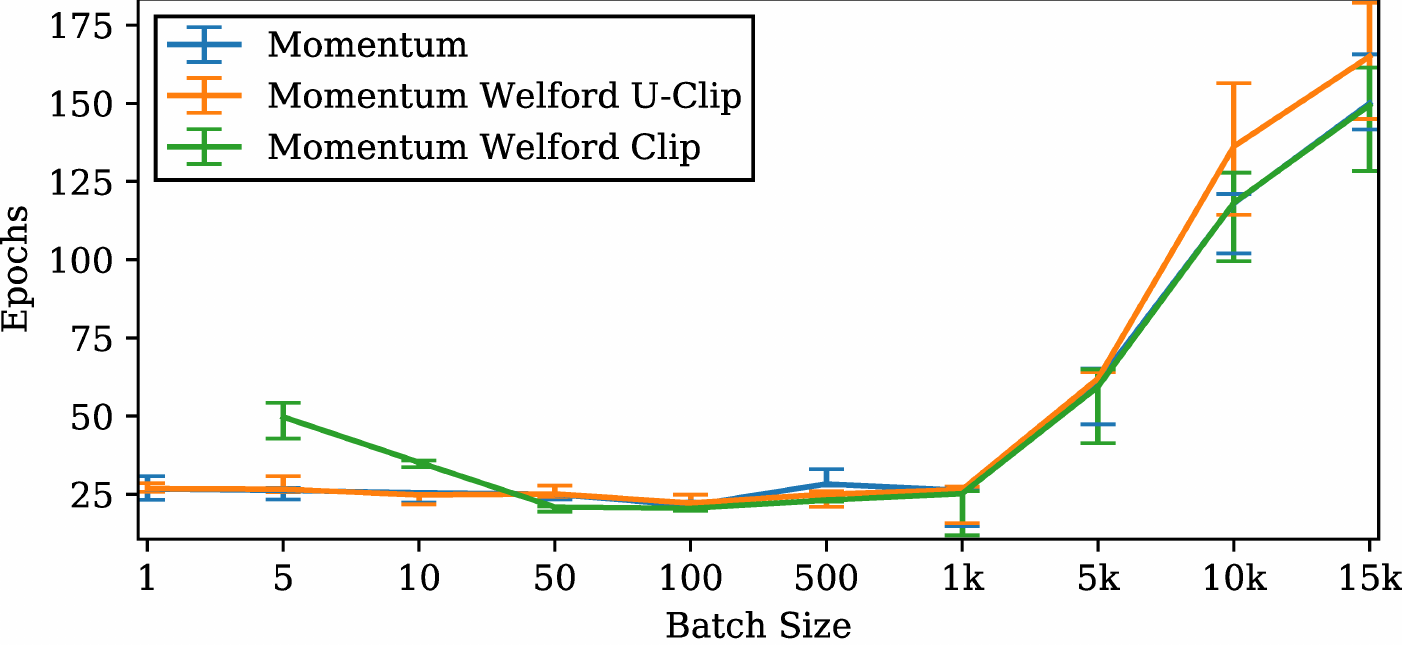}
    \end{subfigure}%
    ~ 
    \begin{subfigure}[t]{\twofigwidth}
        \centering
  \includegraphics[width=\textwidth]{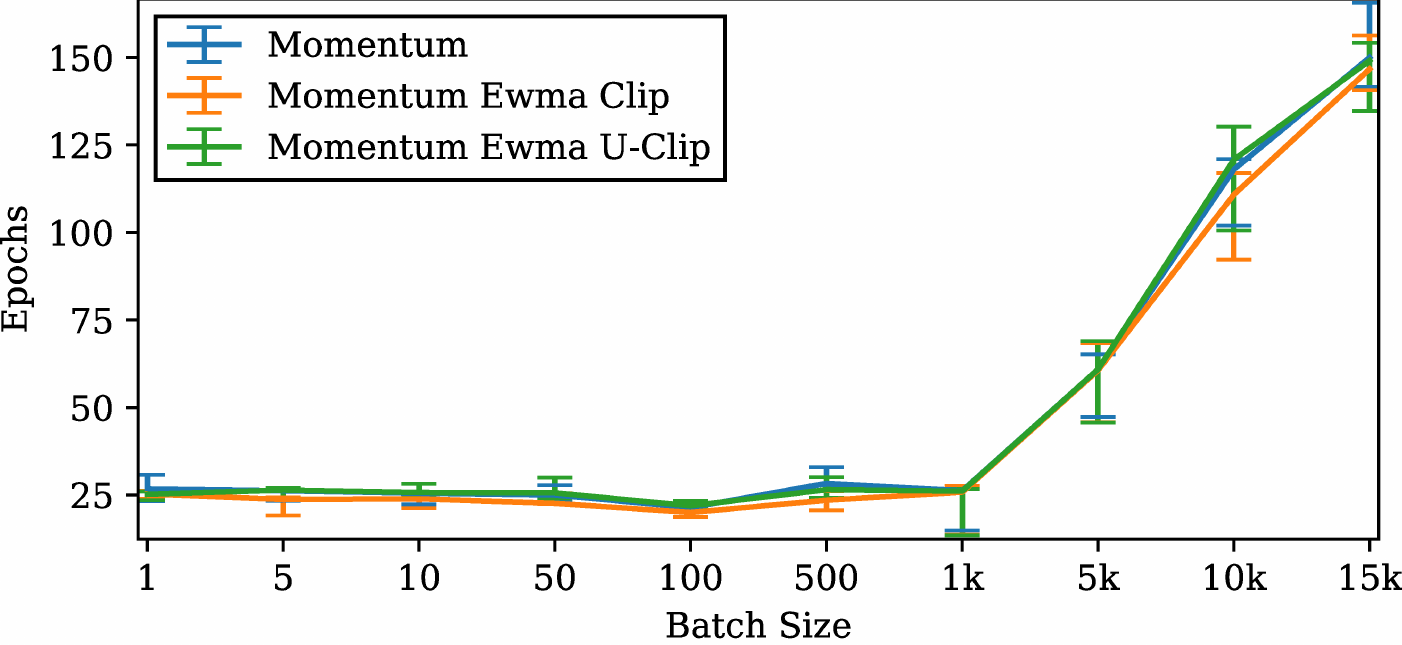}
    \end{subfigure}

    \caption{
      CIFAR10 results for adaptive clip regions with base optimizer momentum. 
      Overall momentum with adaptive \uclip{} does not outperform momentum alone.
}
    \label{fig:cifar10-adaptive-momentum}
\end{figure}

\begin{figure}[htpb]
    \centering
    \begin{subfigure}[t]{\twofigwidth}
        \centering
  \includegraphics[width=\textwidth]{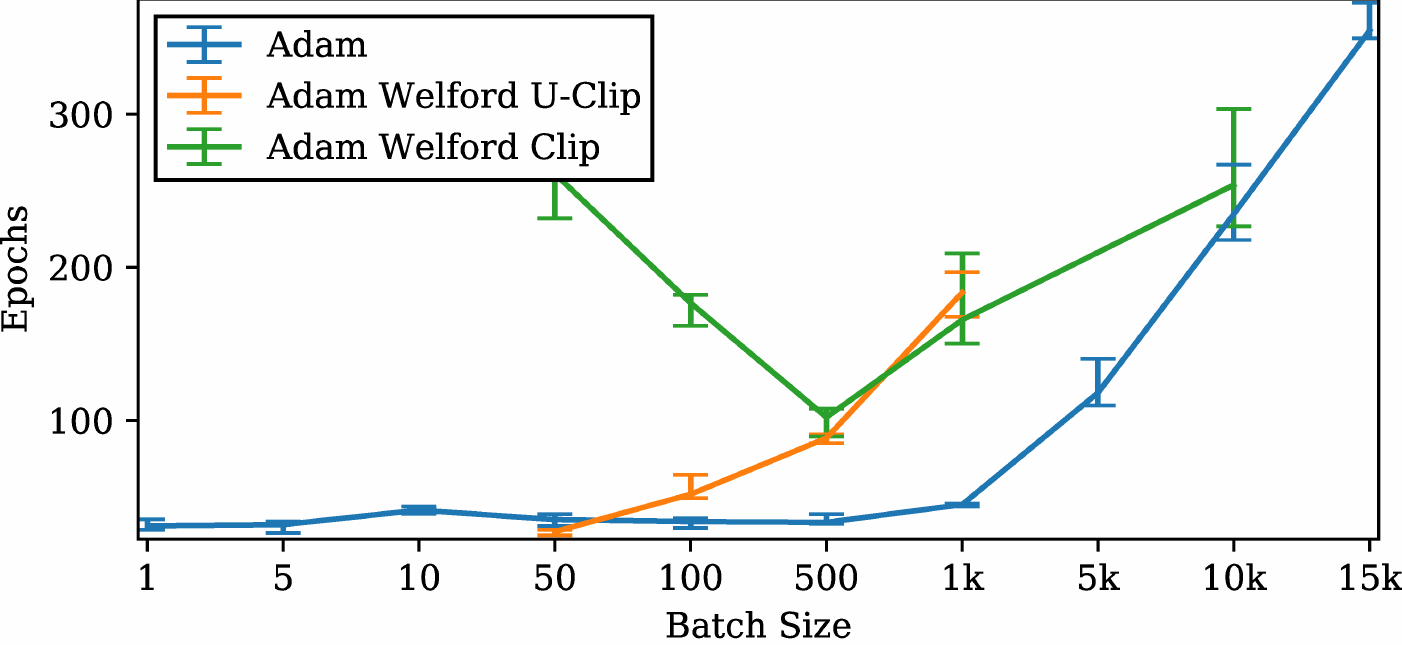}
    \end{subfigure}%
    ~ 
    \begin{subfigure}[t]{\twofigwidth}
        \centering
  \includegraphics[width=\textwidth]{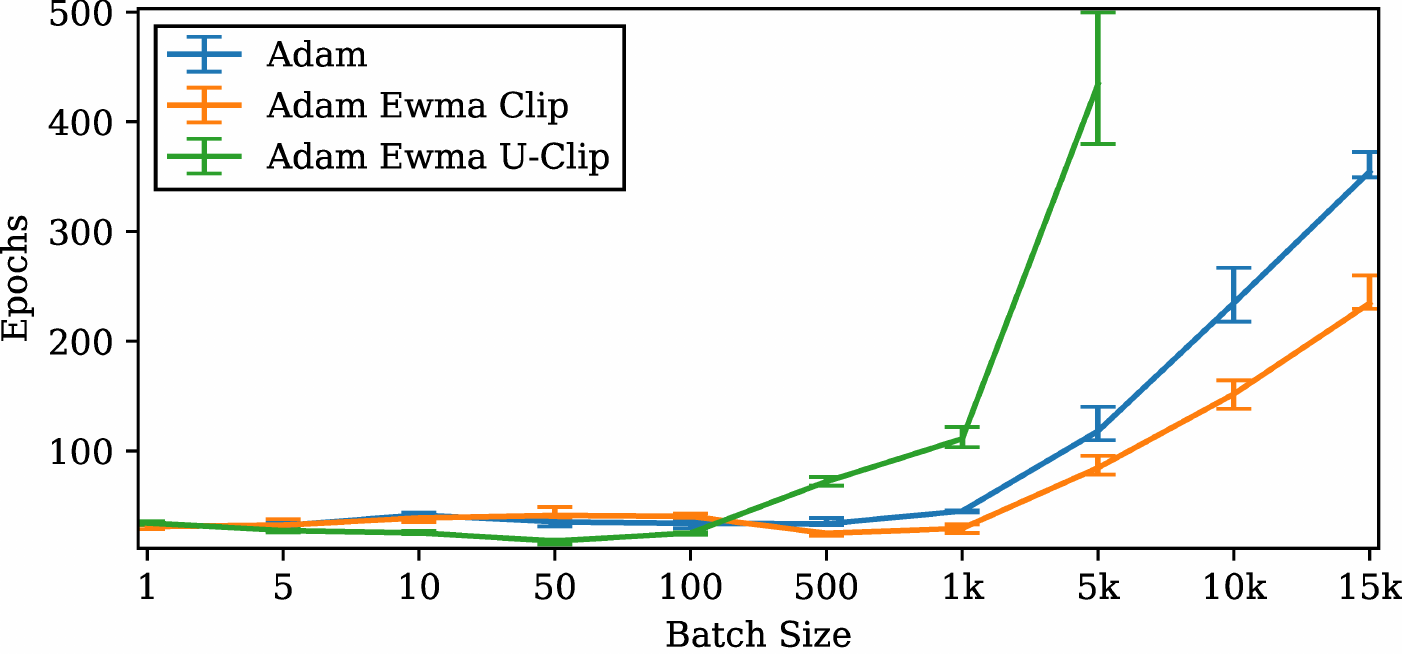}
    \end{subfigure}

    \begin{subfigure}[t]{\twofigwidth}
        \centering
  \includegraphics[width=\textwidth]{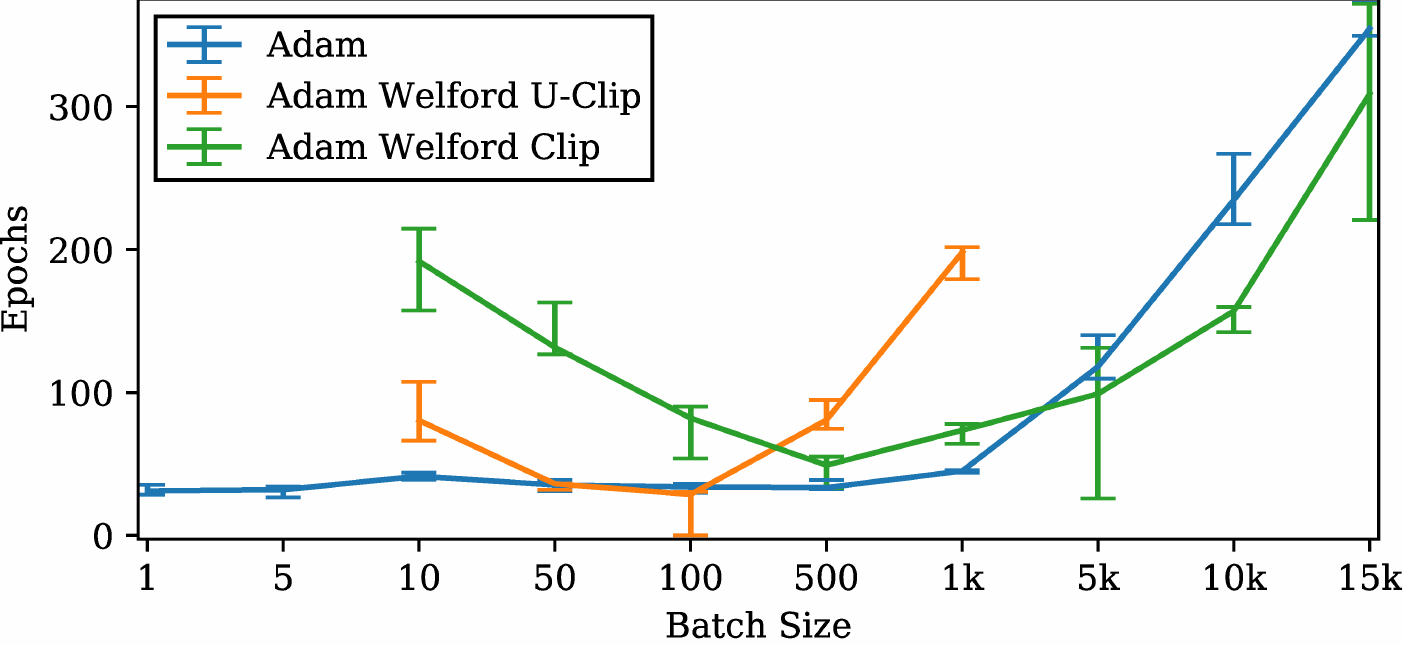}
    \end{subfigure}%
    ~ 
    \begin{subfigure}[t]{\twofigwidth}
        \centering
  \includegraphics[width=\textwidth]{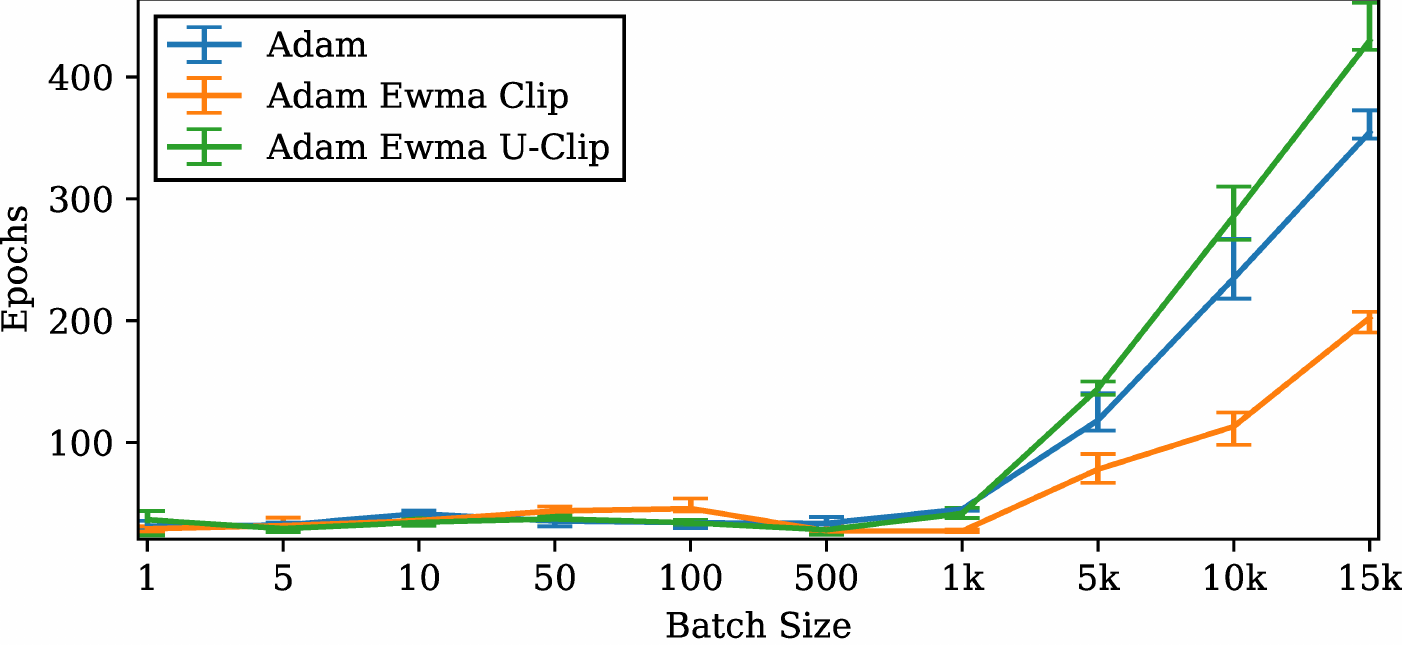}
    \end{subfigure}

    \begin{subfigure}[t]{\twofigwidth}
        \centering
  \includegraphics[width=\textwidth]{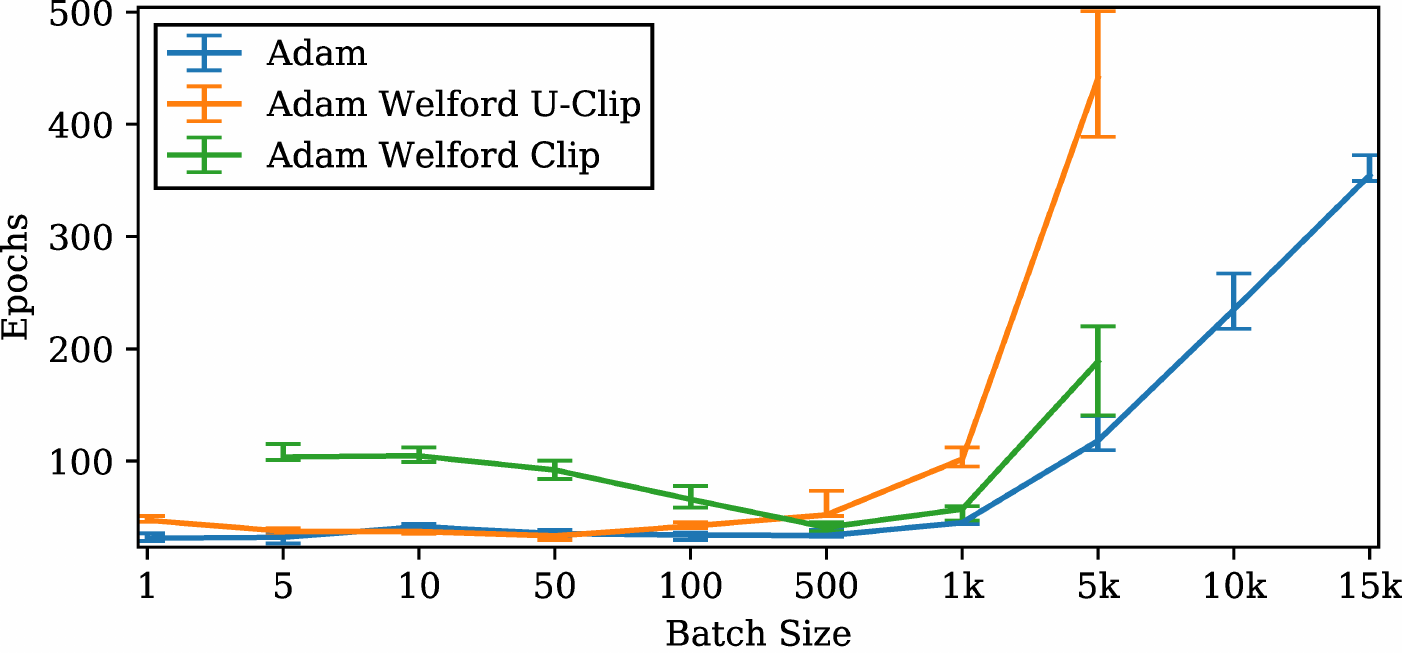}
    \end{subfigure}%
    ~ 
    \begin{subfigure}[t]{\twofigwidth}
        \centering
  \includegraphics[width=\textwidth]{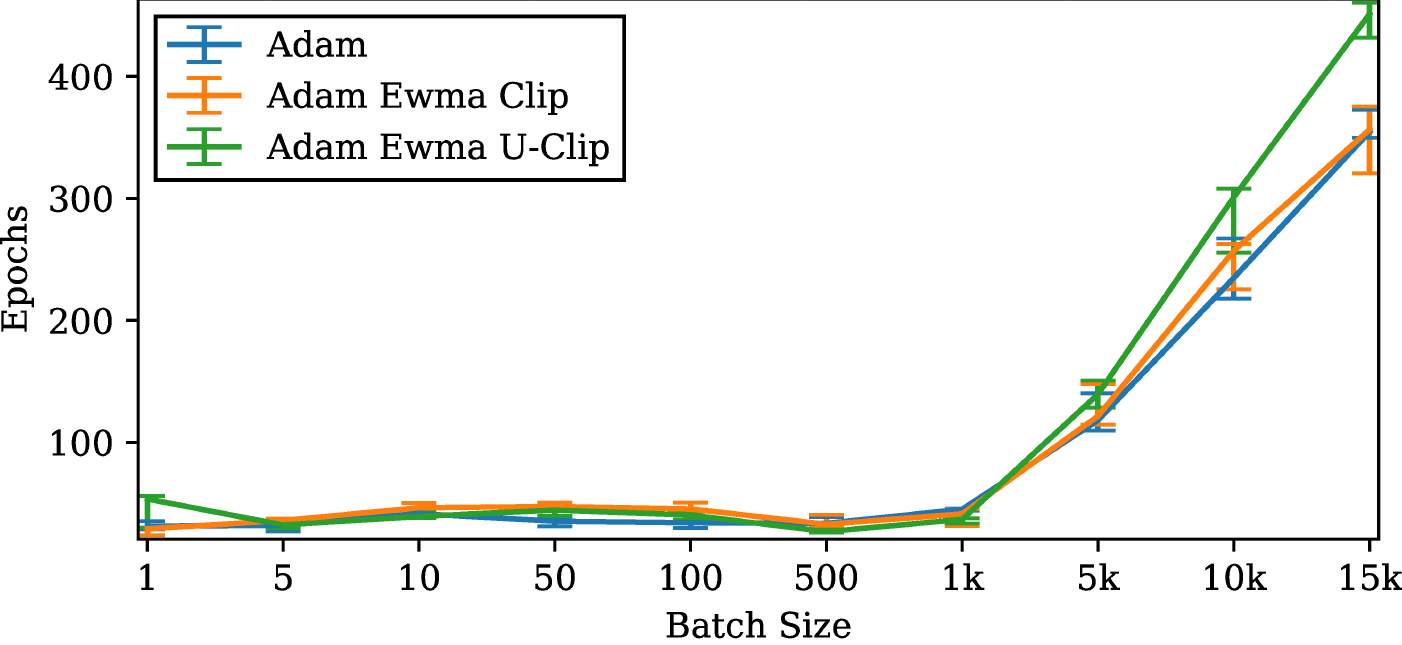}
    \end{subfigure}

    \begin{subfigure}[t]{\twofigwidth}
        \centering
  \includegraphics[width=\textwidth]{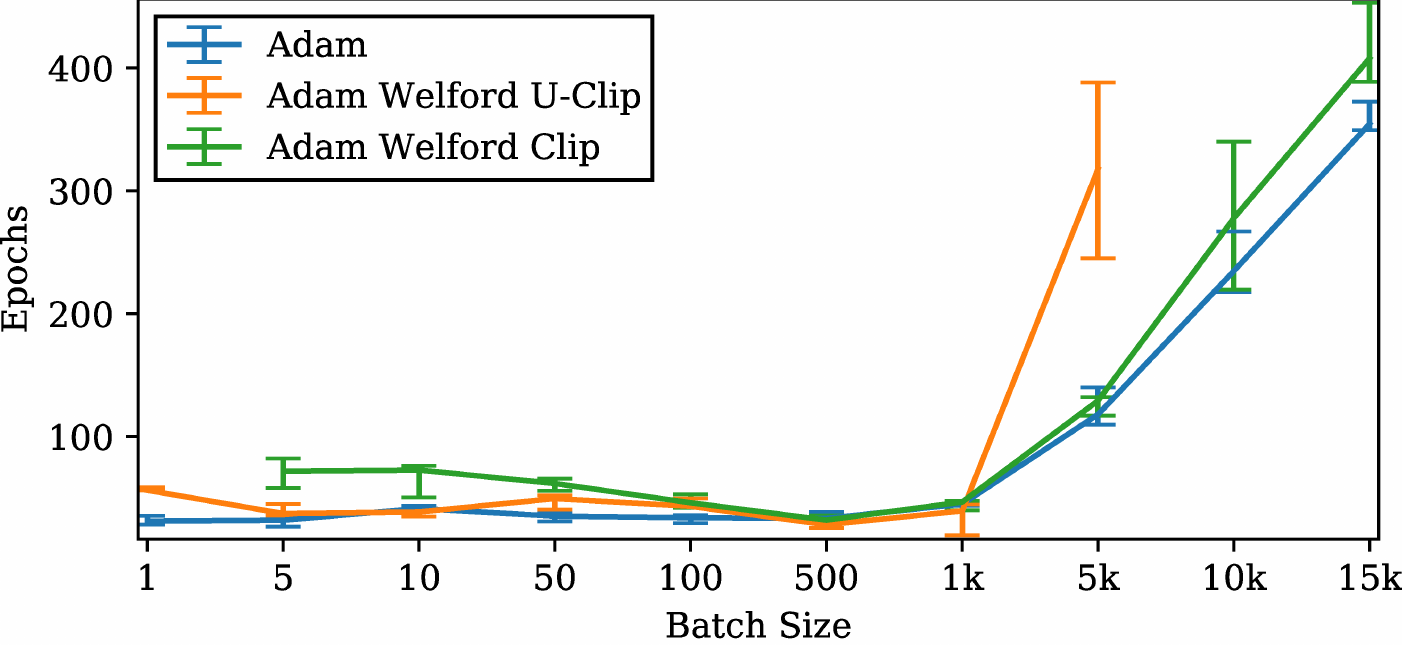}
    \end{subfigure}%
    ~ 
    \begin{subfigure}[t]{\twofigwidth}
        \centering
  \includegraphics[width=\textwidth]{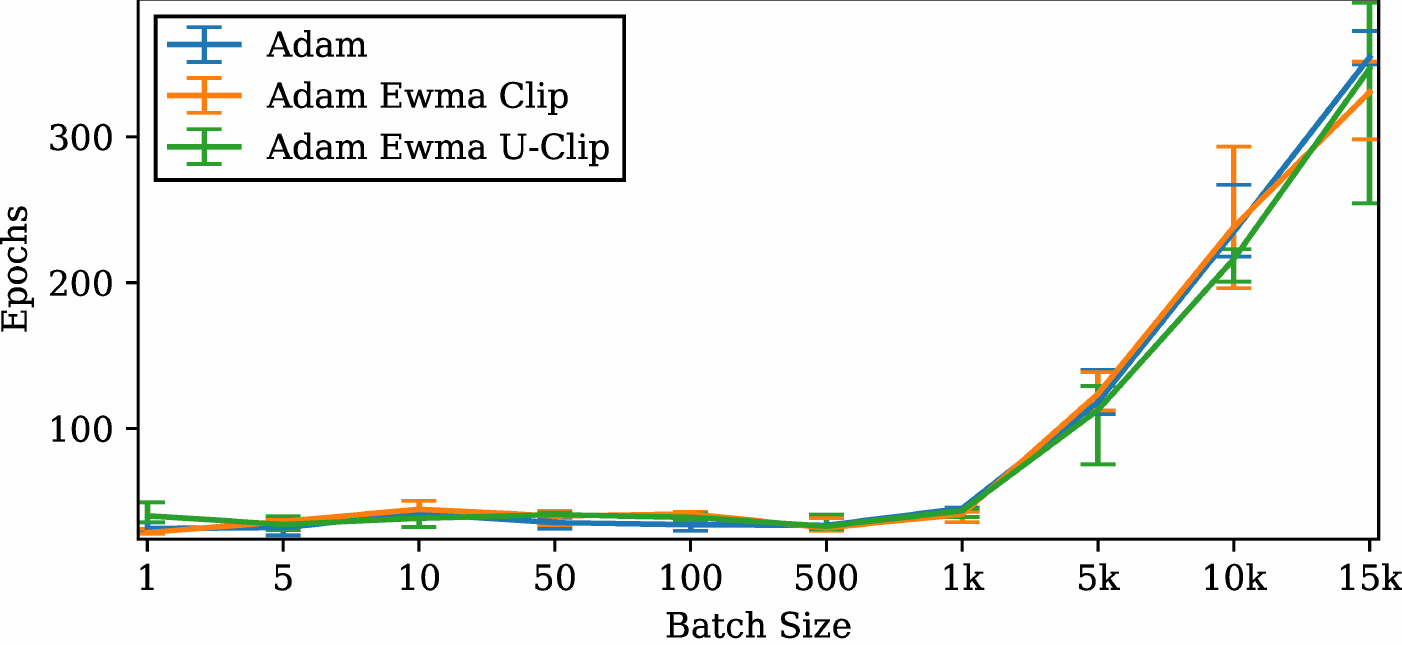}
    \end{subfigure}

    \begin{subfigure}[t]{\twofigwidth}
        \centering
  \includegraphics[width=\textwidth]{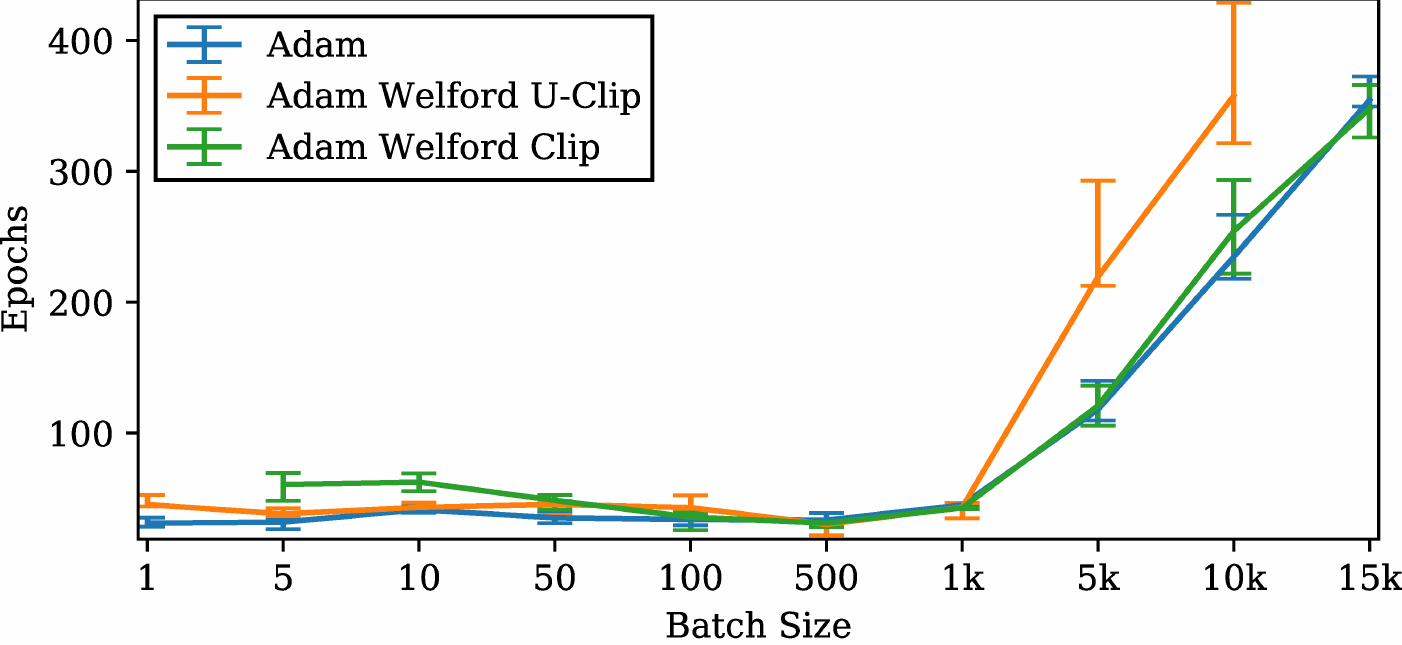}
    \end{subfigure}%
    ~ 
    \begin{subfigure}[t]{\twofigwidth}
        \centering
  \includegraphics[width=\textwidth]{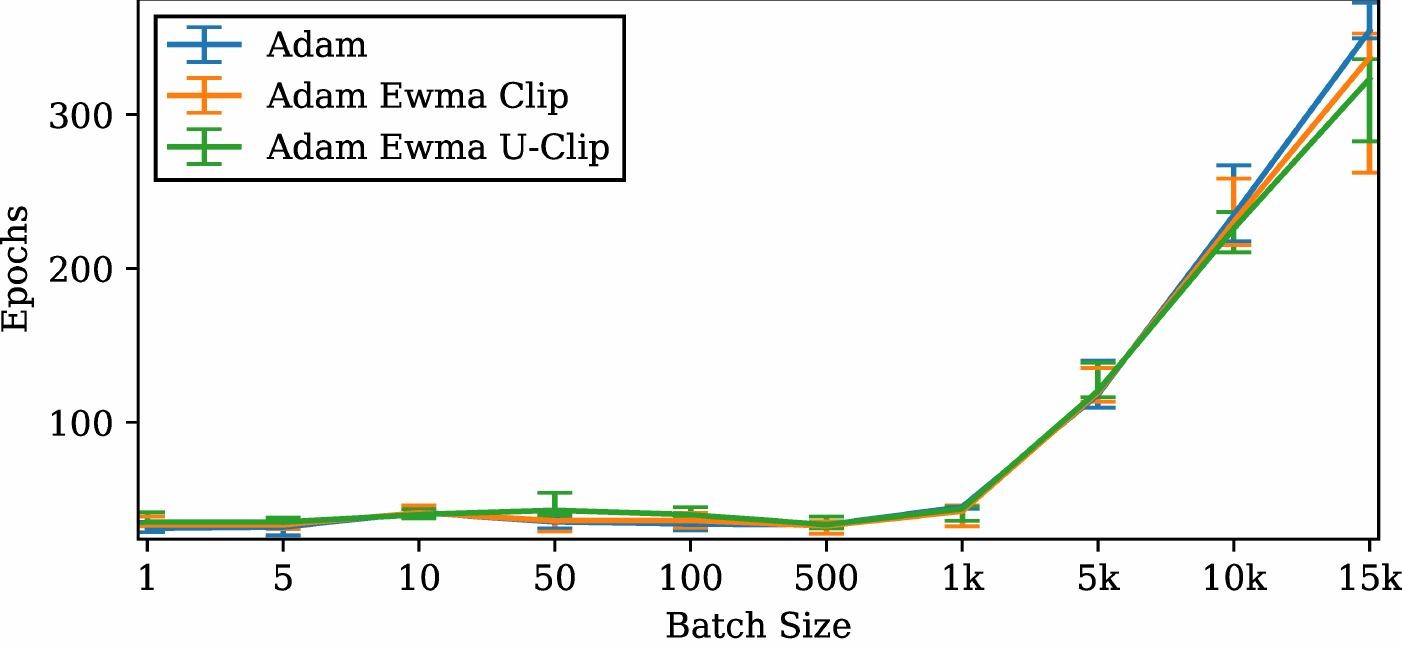}
    \end{subfigure}

    \caption{
      CIFAR10 results for adaptive clip regions with base optimizer Adam. 
      Adaptive \uclip{} gives no clear benefit in combination with Adam in this
      setting.
}
    \label{fig:cifar10-adaptive-adam}
\end{figure}
\end{document}